\documentclass{article} % For LaTeX2e
\usepackage{iclr2027_conference,times}

\usepackage{hyperref}
\usepackage{url}

\usepackage{microtype}
\usepackage{graphicx}
\usepackage{subcaption}
\usepackage{enumitem}
\usepackage{multirow} 

\usepackage{minted}
\usepackage{booktabs} % for professional tables
\usepackage[table,xcdraw]{xcolor} 
\usepackage{colortbl} 
\usepackage{algorithm}
\usepackage{algorithmic}
\usepackage{listings}
\usepackage{xcolor}
\usepackage{hyperref}
\usepackage{wrapfig}

\definecolor{codegray}{gray}{0.95}
\usepackage{amsmath}
\usepackage{amssymb}
\usepackage{mathtools}
\usepackage{amsthm}
\usepackage{bm} 
\usepackage[capitalize,noabbrev]{cleveref}
\usepackage{tikz}
\usetikzlibrary{tikzmark,fit,backgrounds}

\theoremstyle{plain}
\newtheorem{theorem}{Theorem}[section]

\newtheorem{lemma}[theorem]{Lemma}

\theoremstyle{definition}

\theoremstyle{remark}

\usepackage[textsize=tiny]{todonotes}

\usepackage{mathtools}
\usepackage{amssymb}
\usepackage{amsthm}
\usepackage{color}
\newcommand{\cut}[1]{{}}
\newcommand{\red}[1]{{\color{red}#1}} % comments

\newcommand{\bx}{{\bm{x}}}
\newcommand{\by}{{\bm{y}}}
\newcommand{\bz}{{\bm{z}}}
\newcommand{\bg}{{\bm{g}}}
\newcommand{\br}{{\bm{r}}}
\newcommand{\bw}{{\bm{w}}}
\newcommand{\bepsilon}{{\bm{\epsilon}}}
\newcommand{\bxi}{{\bm{\xi}}}
\newcommand{\balpha}{{\bm{\alpha}}}
\newcommand{\bet}{{\bm{\eta}}}

\newcommand{\vA}{{\mathbf{A}}}

\newcommand{\cA}{{\mathcal{A}}}

\newcommand{\cL}{{\mathcal{L}}}

\newcommand{\cR}{{\mathcal{R}}}

\newcommand{\cT}{{\mathcal{T}}}
\newcommand{\cU}{{\mathcal{U}}}

\makeatletter
\let\@@span\span
\def\sp@n{\@@span\omit\advance\@multicnt\m@ne}
\makeatother

\DeclareMathOperator*{\argmin}{arg\,min}

\newcommand{\bc}{\begin{center}}
\newcommand{\ec}{\end{center}}

\newcommand{\bdm}{\begin{displaymath}}
\newcommand{\edm}{\end{displaymath}}

\newcommand{\beq}{\begin{equation}}
\newcommand{\eeq}{\end{equation}}

\newcommand{\bfl}{\begin{flushleft}}
\newcommand{\efl}{\end{flushleft}}

\newcommand{\bt}{\begin{tabbing}}
\newcommand{\et}{\end{tabbing}}

\newcommand{\beqn}{\begin{align}}
\newcommand{\eeqn}{\end{align}}

\newcommand{\beqs}{\begin{align*}} % no equation numbers
\newcommand{\eeqs}{\end{align*}}  % no equation numbers

\title{Boosting Adversarial Robustness and Generalization with Dictionary Structure}

\author{%
  Zhichao Hou\textsuperscript{\rm 1},  Weizhi Gao\textsuperscript{\rm 1}, Hamid Krim\textsuperscript{\rm 1}, Runze Li\textsuperscript{\rm 2}, Xiaorui Liu\textsuperscript{\rm 1} \\
\textsuperscript{\rm 1}North Carolina State University, \textsuperscript{\rm 2}The Pennsylvania State University\\
  \texttt{\{zhou4,wgao23,ahk,xliu96\}@ncsu.edu, runzeli@psu.edu}
}
\iclrfinalcopy % Uncomment for camera-ready version, but NOT for submission.
\begin{document}

\maketitle
\lhead{}

\begin{abstract}

% Adversarial training methods have dominated robustness benchmarks, but they heavily reply on extensive training data, substantial computation resources, and large model backbones while suffering from strong overfitting and a noticeable performance plateau over the years. 

This work investigates a 
novel approach to boost adversarial robustness and generalization by incorporating structural prior into the design of deep learning models.
Specifically, our study surprisingly reveals that existing dictionary learning-inspired convolutional neural networks (CNNs) are robust against random noise but remain highly vulnerable to adversarial attacks.
To address this, we propose Elastic Dictionary Learning Networks (EDLNets), a novel ResNet architecture that significantly enhances adversarial robustness and generalization. 
Extensive  and reliable experiments demonstrate consistent improvements in adversarial robustness across multiple datasets, backbone architectures, and threat models. To the best of our knowledge, this is the first work to discover and validate that 
dictionary structure can reliably enhance deep learning robustness under strong adaptive attacks, unveiling a promising direction for future research.
% Our implementation is available at \url{https://anonymous.4open.science/r/ElasticDL-6343}. 
% and will be make publicly available.

% \xr{some reviewers are confused about ``structure prior'', we can be more explicit about this. ``Dictionary Structure'' is probably better than ``structure prior''?}
% To break this stagnation, we advocate an orthogonal and promising avenue by introducing structural priors into neural network design. Specifically, we revisit dictionary learning in deep learning and identify its limitations under light-tailed noise and adaptive attacks. 
% As a countermeasure, we present a novel elastic dictionary learning framework and solve it using an efficient RISTA algorithm, which can be readily integrated into model layer.  
% Our structural prior can effectively mitigate the robust overfitting problem in adversarial training, achieving remarkable robustness and improved generalization. Furthermore, our method
% is complementary to existing adversarial training approaches  and can be integrated with them to achieve state-of-the-art robustness, significantly surpassing top baselines in RobustBench in terms of both natural and robust performance.
% \xr{need to highlight how good our model is in the abstract}

\end{abstract}

\vspace{-0.1in}
\section{Introduction}
\vspace{-0.1in}
\label{sec:intro}

In adversarial machine learning, most works have demonstrated their effectiveness in safeguarding deep neural networks by refining the geometry and landscape of parameter space, including robust training~\citep{madry2017towards,zhang2019theoretically,gowal2021improving} and regularization~\citep{cisse2017parseval,zheng2016improving}.
Especially in the visual domain, adversarial training methods
powered by generative AI ~\citep{wang2023better,gowal2021improving}
have achieved remarkable success and dominated the robustness leaderboard~\citep{croce2020robustbench}. 
However, their success increasingly relies on extensive synthetic training data crafted by generative models and increasing networks capacity, posing critical challenges to breaking through the plateau in adversarial robustness.

In fact, adversarially trained neural networks improve robustness and generalization largely by memorizing adversarial perturbations during training~\citep{madry2017towards}. While increasing the backbone model size can enhance memorization capacity, adversarial training methods often suffer from  catastrophic \emph{robust overfitting} problem~\citep{rice2020overfitting}. Existing methods aim to mitigate robust overfitting from various perspectives, including regularization~\citep{andriushchenko2020understanding,qin2019adversarial,sriramanan2020guided}, data augmentation~\citep{devries2017improved,zhang2017mixup,carmon2019unlabeled,zhai2019adversarially}, and generative modeling
~\citep{wang2023better,gowal2021improving}.
However, all these methods 
contribute along a similar trajectory by refining the training strategy, making it difficult to achieve further groundbreaking advancements with the same network capacity.

In this work, we propose to explore an orthogonal direction for innovation
upon the finding that most state-of-the-art deep learning approaches largely ignore the \emph{dictionary structural prior} in signal representations.
Several works~\citep{papyan2017convolutional, cazenavette2021architectural, mahdizadehaghdam2019deep, li2022revisiting} have explored 
% building
robust architectures based on sparse coding or dictionary learning priors, assuming that a signal can be sparsely represented as a linear superposition of atoms from a convolutional dictionary. This approach enables natural image patterns to be captured by the sparse code while effectively denoising corruptions. However, their effectiveness has only been validated in scenarios involving weak random corruptions and universal perturbations. Its full potential remains unexplored in the pursuit of stronger adversarial robustness and 
% improved 
generalization.

To fill this research gap, we revisit dictionary learning in deep learning, which uncovers its inability in handling adaptive adversarial attacks, supported by both empirical evidence and theoretical reasoning.
To overcome these limitations, we propose a novel elastic dictionary learning (Elastic DL) framework that complements existing adversarial training methods to achieve superior robustness and generalization.
Our contributions are as follows:
\vspace{-0.1in}
\begin{itemize}[left=0pt,itemsep=0.1pt]
    % \item We revisit convolutional dictionary learning in deep learning, highlighting its failures under adaptive attacks, and we provide theoretical insights into these limitations.
    % \item  We first propose a robust dictionary learning approach via $\ell_1$-reconstruction and highlight its lower natural performance and the challenges in handling adaptive attacks. Furthermore,
\item
We revisit convolutional dictionary learning in deep neural networks and reveal the inherent sensitivity of its $\ell_2$-based reconstruction objective to adversarially induced outliers. To address this limitation, we introduce Elastic Dictionary Learning (Elastic DL) by incorporating $\ell_1$-based reconstruction, a novel framework that achieves a favorable trade-off between standard accuracy and adversarial robustness.
    \item We develop an efficient reweighted iterative  shrinkage thresholding algorithm (RISTA) to approximate the non-smooth Elastic DL objective with theoretical convergence guarantees. The algorithm can be seamlessly integrated into deep learning models as a replacement for conventional convolutional layers to enhance all convolutional architectures.
    \item Extensive experiments demonstrate that our proposed Elastic DL framework can significantly improve adversarial  robustness and generalization. 
    % Notably, our Elastic DL can achieve state-of-the-art performance, significantly outperforming the previous best defense 
    % on RobustBench~\citep{croce2020robustbench} leaderboard across various budgets under $\ell_\infty$-norm and $\ell_2$-norm attacks.
    When combined with high-performing defenses listed on the RobustBench~\citep{croce2020robustbench} leaderboard, Elastic DL improves upon their original performance across multiple perturbation budgets under both $\ell_\infty$- and $\ell_2$-norm attacks.
\end{itemize}

\vspace{-0.05in}
\section{Preliminary: Rethinking Convolution via Dictionary Learning}
\vspace{-0.1in}
\label{sec:rethink_dl}

Convolution is typically treated as a conventional feed-forward operation for capturing localized patch information in images. Here, we instead revisit convolution from a dictionary learning perspective, inverse-modeling the output embedding as the sparse code from which the input can be reconstructed via a learned convolutional dictionary.

\textbf{Notations.} 
Let the input signal be denoted as $\bxi \in \mathbb{R}^{H \times W}$ and the convolution kernel as $\balpha \in \mathbb{R}^{k \times k}$, where $k = 2k_0 + 1$. The \emph{convolution}  and the \emph{transposed convolution} of $\bxi$ and $\balpha$ are defined as:
% \begin{equation} 
$$(\balpha \star \bxi)[i, j] = \sum_{p=-k_0}^{k_0} \sum_{q=-k_0}^{k_0} \bxi[i+p, j+q] \cdot \balpha[p, q],\quad(\balpha * \bxi)[i, j] = \sum_{p=-k_0}^{k_0} \sum_{q=-k_0}^{k_0} \bxi[i-p, j-q] \cdot \balpha[p, q].$$
% \label{eq:correlation} 
% \end{equation}
% while the \emph{transposed convolution} of $\bxi$ and $\balpha$ is defined as:
% \begin{equation} 
% $(\balpha * \bxi)[i, j] = \sum_{p=-k_0}^{k_0} \sum_{q=-k_0}^{k_0} \bxi[i-p, j-q] \cdot \balpha[p, q]$. 
% \label{eq:convolution} 
% \end{equation}

% Different from the conventions in signal processing, the convolution layers in deep learning models actually perform multi-channel \emph{correlation} operations that map the $C$-channel input signal to the $D$-channel output signal.

Let the $C$-channel input signal be denoted as $\bx=\left\{\bxi_1,...,\bxi_C\right\}\in\mathbb{R}^{H\times W\times C}$, and $D$-channel
the output signal as $\bz=\left\{\bet_1,...,\bet_D\right\}\in\mathbb{R}^{H\times W\times D}$. The
convolution operator $\cA(\cdot)$ and its adjoint transposed convolution operator $\cA^*$ are associated with kernel $\vA$ as:
% \begin{equation}
$$
\cA(\bx)=\sum_{c=1}^{C}\left(\balpha_{1c}\star\bxi_c,..., \balpha_{Dc}\star\bxi_c\right)
% \in \mathbb{R}^{H\times W\times D}
,\quad \cA^*(\bz) = \sum_{d=1}^{D}\left(\balpha_{d1}*\bet_d,..., \balpha_{dC}*\bet_d\right) 
% \in \mathbb{R}^{H\times W\times C}
,
% \label{eq:convolution_layer}
% \end{equation}
$$
% The adjoint transposed convolution operator of $\cA$, denoted as $\cA^*$, is defined as:
% \begin{equation}
% \cA^*(\bz) = \sum_{d=1}^{D}\left(\balpha_{d1}*\bet_d,..., \balpha_{dC}*\bet_d\right) \in \mathbb{R}^{H \times W \times C},
% \label{eq:adjoint_convolution_layer}
% \end{equation}
where the associated kernel
% is:
$\mathbf{A} = \left\{\balpha_{dc}\right\}_{d\in [D], c\in [C]}\in \mathbb{R}^{D \times C \times k \times k}$.
% \begin{equation}
% \mathbf{A} = 
% \begin{pmatrix}
% \balpha_{11} & \balpha_{12} & \balpha_{13} & \cdots & \balpha_{1C} \\
% \balpha_{21} & \balpha_{22} & \balpha_{23} & \cdots & \balpha_{2C} \\
% \vdots & \vdots & \vdots & \ddots & \vdots \\
% \balpha_{D1} & \balpha_{D2} & \balpha_{D3} & \cdots & \balpha_{DC}
% \end{pmatrix}
% \in \mathbb{R}^{D \times C \times k \times k}.
% \label{eq:convolution_kernel}
% \end{equation}
Here, $H$, $W$, $C$, $D$, and $k$ represent the height, width, input dimension, output dimension, and kernel size, respectively.

% \vspace{-0.1in}

\textbf{Conventional convolutional learning.}
Convolutional architectures exploit the spatial structure of visual inputs through local receptive fields and shared feature detectors~\citep{fukushima1980neocognitron,lecun1989backpropagation}. Local connectivity captures spatial patterns, while weight sharing enables feature detection across different locations. The convolutional operation is expressed as
% \begin{equation}
    $\bz = \cA(\bx)$,
% \end{equation}
where $\cA$ is a learnable convolutional operator, $\bx$ is the input, and $\bz$ is the resulting feature representation. Although effective for feature extraction, this operation alone does not explicitly constrain how the extracted features reconstruct the input. Convolutional dictionary learning provides a complementary perspective by treating feature extraction as structured sparse inference.

\textbf{Dictionary learning in convolution.}
Prior studies~\citep{papyan2017convolutional,cazenavette2021architectural,mahdizadehaghdam2019deep,li2022revisiting} incorporate dictionary-learning priors into neural architectures to make their representations more interpretable. Under this formulation, an input signal $\bx \in \mathbb{R}^{H \times W \times C}$ is approximated by a sparse combination of atoms from a convolutional dictionary $\vA$, expressed as $\bx \approx \cA^*(\bz)$, where $\cA^*$ denotes the operator associated with $\cA$. Rather than obtaining features through a single convolution, the sparse representation is inferred by solving
\begin{equation}
    \min_{\bz}\ 
    \left\|\bx-\cA^*(\bz)\right\|_2^2
    +
    \lambda\left\|\bz\right\|_1,
    \label{eq:sparse_coding}
\end{equation}
where $\lambda > 0$ controls the trade-off between reconstruction fidelity and code sparsity. The reconstruction term encourages consistency with the learned dictionary, while the sparsity penalty favors representations supported by a small number of dictionary atoms.

\textbf{Robustness implication of sparse coding in vanilla dictionary learning.}
Prior studies have reported promising robustness of sparse codes against random corruptions and universal adversarial perturbations~\citep{cazenavette2021architectural,mahdizadehaghdam2019deep,li2022revisiting}, suggesting that this structural prior can suppress perturbations poorly represented by the dictionary while preserving informative features. This behavior arises from the interplay of the two reconstruction terms: the $\ell_2$-reconstruction term promotes recovering the original input $\bx$ from the dictionary $\cA$ and code $\bz$, while the $\ell_1$-sparsity term encourages $\bz$ to retain only the most informative atoms. Together, they suppress unstructured perturbations and noise that the dictionary cannot compactly represent.

\textbf{Optimization algorithms in solving vanilla dictionary learning.}
For a fixed dictionary $\cA$, Eq.~\eqref{eq:sparse_coding} is convex but non-smooth due to the $\ell_1$ penalty, and several iterative algorithms have been proposed to solve it. The Iterative Shrinkage-Thresholding Algorithm alternates between a gradient-descent step on the $\ell_2$-reconstruction term and a proximal step that enforces sparsity via the soft-thresholding operator~\citep{daubechies2004iterative}. FISTA builds on ISTA by extrapolating between consecutive iterates before each gradient step, using momentum from past updates to guide the search direction~\citep{beck2009fast}. Alternatively, the Alternating Direction Method of Multipliers (ADMM) introduces an auxiliary variable to split the problem into two decoupled subproblems, solving the quadratic data-fidelity term efficiently in the frequency domain via the Sherman-Morrison formula and the $\ell_1$ term through the same soft-thresholding step~\citep{papyan2017convolutional,boyd2011distributed}.

\vspace{-0.1in}
\section{Elastic Dictionary Learning Neural Networks}
\label{sec:method}
\vspace{-0.1in}

In this section, we first  propose a novel elastic dictionary learning (Elastic DL), which aims to achieve the best accuracy-robustness trade-off. We then develop a RISTA algorithm to efficiently solve elastic DL with rigorous theoretical guarantees, and unroll it as Elastic Dictionary Learning Networks (EDLNets). An overview of the proposed architecture is presented in Figure~\ref{fig:overview_edlnets}.

\begin{figure}[h!]
    \centering
    \vspace{-0.1in}
    \includegraphics[width=0.9\textwidth]{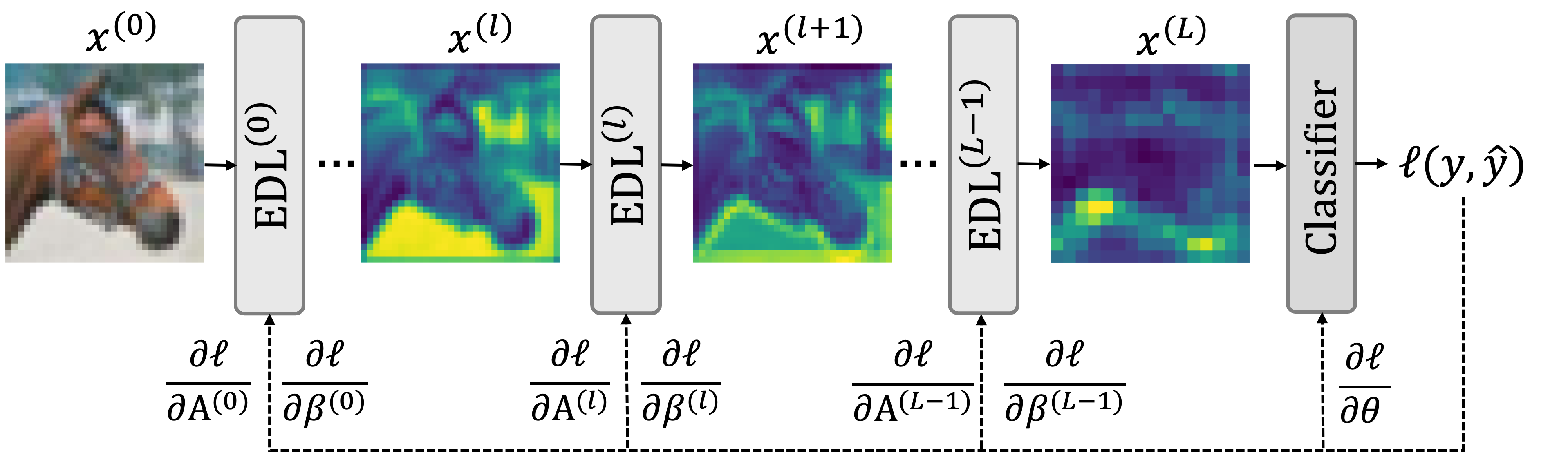} 
    \caption{Overview of Elastic Dictionary Learning Networks (EDLNets). EDLNets are constructed by replacing the convolutional layers in conventional backbones (e.g., ResNets) with EDL layers obtained by unrolling the proposed efficient RISTA algorithm. Each EDL layer incorporates a dictionary-structure prior, assuming that the input signal $\bx^{(l)}$ can be represented by a sparse code $\bx^{(l+1)}$ using a small number of atoms from the dictionary $\mathbf{A}^{(l)}$. During backward propagation, the task loss is propagated through all unrolled RISTA iterations to jointly optimize the layer-wise dictionaries $\{\mathbf{A}^{(l)}\}_{l=0}^{L-1}$ and balancing parameters $\{\beta^{(l)}\}_{l=0}^{L-1}$, as well as the classifier parameters $\theta$.
  }
    \label{fig:overview_edlnets}
    \vspace{-0.1in}
\end{figure}

\vspace{-0.1in}
\subsection{Elastic  Convolutional Dictionary Learning}
\vspace{-0.1in}

% \begin{table}[h!]
\begin{wraptable}{r}{0.5\textwidth}
\vspace{-0.15in}
\centering
\caption{Preliminary study on CIFAR10 comparing standard ResNet18 with a Vanilla DL variant under varying levels of random noise and adaptive PGD attack ($\epsilon=\frac{8}{255}$). Vanilla DL exhibits strong robustness against weak random noise but is easily broken by the strong PGD attack.}
\vspace{-0.15in}
\begin{center}
 \setlength{\tabcolsep}{1.8pt}
\begin{sc}
\resizebox{0.5\textwidth}{!}{
\begin{tabular}{l|ccccc|cccc}
\rowcolor[HTML]{C0C0C0}
\hline
\textbf{Model $\backslash$ Noise Level }  & \textbf{L-1} &\textbf{L-2}&\textbf{L-3}&\textbf{L-4}&\textbf{L-5}&\red{\textbf{PGD}} \\ \hline
Classic ConvNet (ResNet18)&81.44 & 57.23 & 48.32 & 32.49 & 16.98  &\red{\textbf{0.00}}\\
Vanilla DL (Fixed $\lambda=0.1$)&82.39 &68.90 &59.28 &40.8 &23.83 &\red{\textbf{0.01}}\\
Vanilla DL (Tuned $\lambda$)&82.39 &68.90 & 59.28 &43.71 & 33.43 &\red{\textbf{0.13}}\\

\bottomrule
\end{tabular}

}
\end{sc}
\end{center}
\vspace{-0.15in}
\label{tab:pre_sdnet18}
% \end{table}
\end{wraptable}

\textbf{Vanilla dictionary learning is not truly robust.} Despite the promising robustness of sparse coding in vanilla dictionary learning (Vanilla DL) reported in several works~\citep{cazenavette2021architectural,mahdizadehaghdam2019deep,li2022revisiting}, such benefits do not necessarily extend to stronger adaptive attacks. To validate this, we conduct a preliminary experiment on CIFAR10 using a Vanilla DL variant of ResNet18~\citep{li2022revisiting}, in which all convolutional layers are replaced with convolutional sparse coding layers based on Vanilla DL as formulated in Eq.~\eqref{eq:sparse_coding}. We evaluate this Vanilla DL network (with both a fixed $\lambda$ and a tuned $\lambda$) under random impulse noise and an adaptive PGD adversarial attack~\citep{madry2017towards} with budget $\frac{8}{255}$. As shown in Table~\ref{tab:pre_sdnet18}, the Vanilla DL network improves upon the standard ResNet18 baseline in robustness against random noise, with a more pronounced improvement achieved by tuning the sparsity weight $\lambda$. However, it still suffers a sharp performance drop under the adaptive PGD attack, with accuracy falling to near zero. Detailed results across various noise levels and $\lambda$ values are provided in Figure~\ref{fig:pre_sdnet18_various_lambda} in Appendix~\ref{sec:pre_study}.
 In fact, the $\ell_2$-reconstruction term of Vanilla DL in Eq.~\eqref{eq:sparse_coding} imposes a quadratic penalty $\|\cdot\|_2^2$ on the residual $\bx-\cA^*(\bz)$, making it highly sensitive to outliers introduced by high-level noise and adaptive attacks.
The experimental results reveal that existing Vanilla DL is robust against random noise but can easily be compromised by adaptive attacks. Thus, a substantial gap remains in achieving truly robust dictionary learning in deep learning.

\textbf{Elastic dictionary learning in convolution.}
Our design draws on a classical insight from robust statistics: the squared $\ell_2$ loss used in least-squares estimation amplifies large residuals quadratically and is therefore highly sensitive to outliers, whereas the $\ell_1$ loss underlying the Least Absolute Deviations (LAD) estimator grows only linearly with residual magnitude, making it far more robust to gross errors~\citep{huber1992robust,rousseeuw2003robust}. 
% Huber regression reconciles the two by combining both losses into a single penalty that behaves quadratically for small residuals and linearly for large ones, thereby retaining the efficiency of least squares while bounding the influence of outliers~\citep{huber1992robust}. 
Motivated by this principle, we introduce an analogous elastic reconstruction penalty into convolutional dictionary learning, so that the resulting sparse code remains robust to adversarially-induced outliers in the reconstruction residual while preserving reconstruction fidelity on clean inputs.
To achieve greater robustness while maintaining strong natural performance, we propose Elastic Dictionary Learning (Elastic DL), which combines squared $\ell_2$ and $\ell_1$ reconstruction penalties while retaining the sparse-coding prior:
% \vspace{-0.02in}
\begin{equation}
    \min_{\bz}\ \cL(\bz)
    =
    \frac{\beta}{2}\left\|\bx-\cA^*(\bz)\right\|_2^2
    +
    \frac{1-\beta}{2}\left\|\bx-\cA^*(\bz)\right\|_1
    +
    \lambda\left\|\bz\right\|_1,
    \label{eq:elastic_dictionary_learning}
    % \vspace{-0.05in}
\end{equation}
where $\beta \in [0,1]$ is a learnable parameter to control the contributions of the two reconstruction terms.

\textbf{Why the elastic dictionary learning improves robustness.}
Eq.~\eqref{eq:elastic_dictionary_learning} admits a probabilistic reading via maximum likelihood estimation: the $\ell_2$-reconstruction term implicitly assumes a light-tailed residual $\bx-\cA^*(\bz)$, which is reasonable for natural noise but lets a handful of adversarially large residuals dominate and distort $\bz$, whereas the $\ell_1$-reconstruction term corresponds to a heavy-tailed (Laplace) assumption that grows only linearly with residual magnitude, bounding the influence of such outliers. The two terms are thus complementary: $\ell_2$ preserves reconstruction efficiency on clean inputs, while $\ell_1$ suppresses adversarial outliers when present. The separate $\ell_1$-sparsity penalty on $\bz$ serves a different purpose, retaining only the most informative dictionary atoms and discarding noise-driven activations. Rather than fixing this trade-off a priori, we treat $\beta$ as a learnable, layer-wise parameter optimized jointly with the dictionary during adversarial training, letting each layer adapt toward the reconstruction assumption that best balances accuracy and robustness.

\textbf{Challenges in efficiently solving elastic dictionary learning.}
Despite its favorable robustness properties, the objective in Eq.~\eqref{eq:elastic_dictionary_learning} is considerably harder to optimize than its vanilla counterpart. In vanilla dictionary learning (Sec.~\ref{sec:rethink_dl}), the reconstruction term in Eq.~\eqref{eq:sparse_coding} is a smooth, convex quadratic, so the problem can be solved efficiently by taking a gradient step on this term and applying the soft-thresholding operator to enforce sparsity on $\bz$. In elastic dictionary learning, however, the $\ell_1$-reconstruction term introduces a second source of non-smoothness in addition to the $\ell_1$-sparsity term, so the reconstruction term itself no longer admits a simple gradient. This precludes a direct extension of the standard ISTA-style update and calls for a dedicated algorithm capable of efficiently approximating the solution to Eq.~\eqref{eq:elastic_dictionary_learning}. To the best of our knowledge, no prior work has proposed elastic convolutional dictionary learning and addressed how to unroll such a doubly non-smooth elastic dictionary learning objective into efficient, trainable neural network layers, leaving an important research gap that we address next.

\vspace{-0.05in}
\subsection{Reweighted Iterative Shrinkage Thresholding (RISTA) Algorithm}
\vspace{-0.1in}
To address the challenge posed above, we propose the Reweighted Iterative Shrinkage-Thresholding Algorithm (RISTA), which solves Eq.~\eqref{eq:elastic_dictionary_learning} by iteratively minimizing a convex \emph{localized upper bound} of the original non-smooth objective, yielding a rigorous descent on the original loss at every step.

\textbf{Localized upper bound.}
To address $\|\bx - \cA^*(\bz)\|_1$ term, we first propose a convex upper bound $\cU(\bz,\bz_*)$  as an alternative in the following Lemma~\ref{lemma:local_upper_bound}.
\begin{lemma}
\label{lemma:local_upper_bound}
Let $\cR(\bz):=\|\bx-\cA^*(\bz)\|_1$, and for any fixed point $\bz_*$,  $\cU(\bz,\bz_*)$ is defined as 
\begin{equation}
    \label{eq:upper_bound}
    \cU(\bz,\bz_*)=\|\bw^{1/2}\odot (\bx-\cA^*(\bz))\|_2^2+\frac{1}{2}\cR(\bz_*),
\end{equation}
where $\bw= \frac{1}{2|\bx-\cA^*(\bz_*)|}.$
Then, for any $\bz$, the following holds:
$$(1)\; \cU(\bz,\bz_*) \geq \cR(\bz), \quad (2)\; \cU(\bz_*,\bz_*) = \cR(\bz_*) .$$
\end{lemma}
\vspace{-0.1in}
\begin{proof}
    Please refer to Appendix~\ref{sec:proof_local_upper_bound}.
\end{proof}
\vspace{-0.1in}
The statement (1) indicates that $\cU(\bz,\bz_*)$ serves as an upper bound for $\cR(\bz)$, while statement (2) demonstrates that  $\cU(\bz,\bz_*)$ equals $\cR(\bz)$ at point $\bz_*$.
With fixed $\bz_*$, the alternative objective $\cU(\bz,\bz_*)$ in Eq.~\eqref{eq:upper_bound} is quadratic and can be efficiently optimized.
Therefore, instead of minimizing the non-smooth $\cR(z)$ directly, we can alternatively optimize the quadratic upper bound $\cU(\bz,\bz_t)$ with gradient descent algorithm at iteration $t$.

\textbf{Alternative objective.}
By Lemma~\ref{lemma:local_upper_bound}, an alternative objective for Eq.~\eqref{eq:elastic_dictionary_learning} at each step $t$ is:
\begin{equation}
\begin{aligned}
\bz_{t+1}&=\argmin_\bz \hat{\cL}(\bz,\bz_t) = \frac{\beta}{2}\|\bx - \cA^*(\bz)\|_2^2 + \frac{1-\beta}{2}  \|\bw_t^{1/2} \odot \left(\bx - \cA^*(\bz)\right)\|_2^2 + \lambda\|\bz\|_1 ,
\label{eq:alternative_objective_l1}
\end{aligned}
\end{equation}
where 
% $
% \bw_t_i = \frac{\tilde{\bw}_t_i}{\frac{1}{N}\sum_{i=1}^N \tilde{\bw}_t_i}, \quad \tilde{\bw}_t = \frac{1}{|\bx - \cA^*(\bz_t)|} \in \mathbb{R}^N, \quad N = H \times W \times C.
% $
$
\bw_t =  \frac{1}{2|\bx - \cA^*(\bz_t)|} \in \mathbb{R}^{H \times W \times C}.
$
Specifically, when $\bw_t = \mathbf{1}$, the problem reduces to the formulation in Eq.~\eqref{eq:sparse_coding}.
In practice, to avoid the singularity when the reconstruction residual approaches zero, we use an $\epsilon$-stabilized form in our implementation:
$\bw_t=\frac{1}{2(|\bx-\cA^*(\bz)|+\epsilon)}$
.

% \subsection{Algorithm Development }

\textbf{RISTA algorithmic iteration.}
Then, we can 
optimize the $\ell_1$-regularized problem in Eq.~\eqref{eq:alternative_objective_l1} instead of original Eq.~\eqref{eq:elastic_dictionary_learning} by our reweighted iterative  shrinkage thresholding algorithm (RISTA):
\begin{equation}
    % \begin{aligned}
    \bm{g}_t = \cA \left(\left(\beta \mathbf{1} + (1-\beta) \bw_t\right)\odot\left(\bx-\cA^*(\bz_t)\right)\right),\quad
        \bz_{t+1} =  \mathcal{T}_{\lambda \eta} \left( \bz_t + \eta \cdot \bm{g}_t  \right),
        \label{eq:algo_iteration_l1}
    % \end{aligned}
\end{equation}
where $\mathcal{T}_{\lambda \eta}(\bz)=\text{sign}(\bz) \left(|\bz|-\lambda\eta \right)_+$ represents the soft thresholding operator and $\eta$ is the step size. The detailed derivation of Eq.~\eqref{eq:algo_iteration_l1} is provided in Appendix~\ref{sec:proof_algo_iteration}.
%%%
As a consequence of Lemma~\ref{lemma:local_upper_bound}, we can conclude the iteration $\{\bz_t\}_{t=0}^{T}$ obtained by Eq.~\eqref{eq:algo_iteration_l1} fulfill the loss descent of original $\cL(\bz)$:
$$\cL(\bz_{t+1})\leq\hat{\cL}(\bz_{t+1},\bz_t)\leq\hat{\cL}(\bz_t, \bz_t)=\cL(\bz_t).$$
Here, the first inequality follows from the upper-bound property in Lemma~\ref{lemma:local_upper_bound} (1), the second holds because $\bz_{t+1}$ is obtained by minimizing $\hat{\cL}(\cdot,\bz_t)$ over $\bz$ and is therefore no larger than its value at any other point, including $\bz_t$, and the final equality follows from the tightness property in Lemma~\ref{lemma:local_upper_bound} (2).
This implies 
convergence of  Eq.~\eqref{eq:elastic_dictionary_learning}   can be achieved by optimizing the localized upper bound Eq.~\eqref{eq:alternative_objective_l1}. 

\vspace{-0.1in}
\section{Experiment}
\vspace{-0.1in}
In this section, we comprehensively evaluate the effectiveness of our proposed 
% elastic dictionary training 
EDLNets under various experimental settings. Additionally, we provide several ablation studies to demonstrate the working mechanism of our approach.

\vspace{-0.1in}
\subsection{Experimental Setting}
\vspace{-0.1in}
\label{sec:exp_set}
\textbf{Datasets.} We conduct the experiments on several datasets including CIFAR10~\citep{krizhevsky2009learning}, CIFAR100~\citep{krizhevsky2009learning} and Tiny-ImageNet~\citep{le2015tiny}.

\textbf{Backbone architectures.} We select ResNets as the backbones, including ResNet10, ResNet18, ResNet34, and ResNet50~\citep{he2016deep}. Each of the convolutional layers in ResNets are replaced with our Elastic DL layer, resulting in the corresponding EDLNets.
% Elastic DL neural networks
We use ResNet18 as the default backbone if not being specified.

\textbf{Evaluation methods.} 
We evaluate the performance of the models against various attacks, including FGSM~\citep{goodfellow2014explaining}, PGD~\citep{madry2017towards}, C\&W~\citep{carlini2017towards}, AutoAttack~\citep{croce2020reliable}, and SparseFool~\citep{modas2019sparsefool}, covering budget measurements across $\ell_\infty$-norm, $\ell_2$-norm, and $\ell_1$-norm.
For the PGD attack, we consider both $\ell_\infty$-norm and $\ell_2$-norm, denoted as PGD-$\ell_\infty$ and PGD-$\ell_2$, respectively. SparseFool uses the $\ell_1$-norm. Unless otherwise specified, $\ell_\infty$ is used as the default measurement.
% \textbf{Robustness reliability test.} 
% To prevent a false sense of security caused by gradient obfuscation, 
We perform multiple robustness reliability tests, including \emph{certifiable robustness} (Figure~\ref{fig:certified_robustness}), \emph{transferability analysis} (Figure~\ref{fig:transfer_analysis}), \emph{zero-order gradient analysis} (Appendix~\ref{sec:zero_order_gradient}), and attack with BPDA  (Appendix~\ref{sec:bpda}).

% AutoAttack is an ensemble attack consisting of three adaptive white-box attacks and one black-box attack, which is considered as a reliable evaluation method to avoid the false sense of security. 

\textbf{Baselines.} For robust overfitting mitigation, we include the baselines including regularization ($\ell_1$, $\ell_2$ regularizations and their combination), Cutout~\citep{devries2017improved}, Mixup~\citep{zhang2017mixup}, and early stopping~\citep{rice2020overfitting}. For adversarial training methods, we compare the baselines including PGD-AT~\citep{madry2017towards}, TRADES~\citep{zhang2019theoretically}, MART~\citep{wang2019improving}, SAT~\citep{huang2020self},  AWP~\citep{wu2020adversarial}, Consistency~\citep{tack2022consistency}, DYNAT~\citep{liu2024dynamic}, PORT~\citep{sehwag2021robust}, and HAT~\citep{rade2022reducing}.

\textbf{Hyperparameter setting.} 
We train the baselines for 200 epochs with batch size 128, weight decay 2e-5, momentum 0.9, and an initial learning rate of 0.1 that is divided by 10 at the 100-th and 150-th epoch. The $\epsilon$ to avoid singularity is set as 1e-3.
For our Elastic DL, we pretrain the Vanilla DL model for 150 epochs and then fine-tune the Elastic DL model for 50 epochs. We compare and discuss different training strategies in Appendix~\ref{sec:training_strategy}.

\vspace{-0.1in}
\subsection{Adversarial Robustness \& Generalization}
% \vspace{-0.1in}

\begin{table}[h!]
\centering
\vspace{-0.1in}
\caption{ Natural and robust performance of PGD-based adversarial training with different methods to mitigate the overfitting on CIFAR10. BEST represents the highest test accuracy achieved during training, while FINAL is the average accuracy over the last five epochs. DIFF, the difference between BEST and FINAL, measures the ability to mitigate overfitting. 
% Each of the regularization methods listed is trained using the optimally chosen hyperparameter. 
The best performance is highlighted in \textbf{bold}, while the second-best is \underline{underlined}.
}

\label{tab:robust_overfitting}
\vspace{-0.1in}
\begin{center}
\begin{sc}
\resizebox{0.65\textwidth}{!}{
\rowcolors{2}{gray!20}{white}
\begin{tabular}{c|ccc|ccc}
\hline
\rowcolor[HTML]{C0C0C0}
&\multicolumn{3}{c|}{\textbf{Natural Acc.} }&\multicolumn{3}{c}{\textbf{Robust Acc.} }\\\hline
\textbf{Method}&\textbf{Final}&\textbf{Best}&\textbf{Diff}&\textbf{Final}&\textbf{Best}&\textbf{Diff}\\\hline
Vanilla&78.98 & 79.90 & 0.92 & 44.90 & 48.01 & 3.11 \\
$\ell_1$ Reg.&64.84 & 65.71 & 0.87 & 40.94 & 41.97 & 1.03 \\
$\ell_2$ Reg.&78.88 & 79.39 & 0.51 & 42.73 & 48.26 & 5.53 \\
$\ell_2$ + $\ell_1$  Reg.&66.86 & 67.62 & 0.76 & 42.53 & 43.33 & 0.80\\
Cutout&75.11 & 75.58 & 0.47 & 47.12 & 48.23 & 1.11 \\
Mixup&69.64 & 72.05 & 2.41 & 46.10 & 48.53 & 2.43\\
% Semi-supervised&\\
Early Stopping&75.51 & 75.51 & \textbf{0.00} & \underline{47.69} & 47.95 & \textbf{0.26} \\
Vanilla DL& \underline{82.59} & \textbf{83.27} & 0.68 & 44.03 & \underline{50.53} & 6.50\\
Elastic DL \red{(Ours)}& \textbf{83.01} & \textbf{83.27} & \underline{0.26} & \textbf{54.94} & \textbf{55.66} & \underline{0.72}\\
\hline
\end{tabular}

}

\end{sc}
\end{center}
\vspace{-0.2in}
\end{table}

First, we validate the effectiveness of our approach in mitigating overfitting. Next, we comprehensively evaluate the robustness of adversarial training methods. Finally, we demonstrate that our approach surpasses the state-of-the-art methods on the leaderboard  via structural priors.

\textbf{Robust overfitting mitigation.} To validate the effectiveness of incorporating structural priors, we compare our method with existing popular baselines on CIFAR10 in mitigating the \emph{robust overfitting} problem in  Table~\ref{tab:robust_overfitting} and Figure~\ref{fig:curve_resnet10_all_adv_acc_test}.
% Moreover, we also track the training curves during the pretraining and fine-tuning phases in Figure~\ref{fig:adv_train_curve} to gain insight into the working mechanism of our method. 
% In the pretraining phase, the Vanilla DL model is trained for 200 epochs using a step-decay learning rate schedule. Subsequently, the model is switched to Elastic DL with all $\beta=0.5$ as the initialization, followed by fine-tuning for 50 epochs.
We leave  the training curves
of all the methods 
in Appendix~\ref{sec:training_curves_each_method} and Appendix~\ref{sec:curves_comparison_all_method} due to the space limit.   
From the results, we can make the following observations: 
\vspace{-0.1in}
\begin{itemize}[left=0.0em]
\vspace{-0.1in}
    \item From Table~\ref{tab:robust_overfitting}, we observe that our Elastic DL 
    not only achieves a significant advantage in both absolute FINAL and BEST performance but also maintains a relatively small gap (DIFF) between them, indicating that incorporating the structural prior effectively guides adversarial training to achieve better robustness and generalization.
    
    % \item In terms of FINAL and BEST performance, we observe that our Elastic DL method significantly outperforms other baselines in both natural and robust accuracy, indicating that incorporating the structural prior effectively guides adversarial training to achieve better robustness and generalization.
    % \item In terms of DIFF performance, early stopping~\citep{rice2020overfitting} effectively reduces the gap between the best and final performance. However, it falls short in achieving the absolute best performance, thereby limiting the upper bound of final natural and robust accuracy. In contrast, our Elastic DL method not only achieves a significant advantage in both absolute final and best performance but also maintains a relatively small gap between them.

    \item From Figure~\ref{fig:curve_resnet10_all_adv_acc_test}, we observe that during the 100th to 200th epochs, the Vanilla DL model exhibits a severe \emph{robust overfitting} phenomenon. By incorporating our Elastic DL structural prior at the 150th epoch, the test robustness improves substantially, highlighting the promising potential of the Elastic DL structural prior in overcoming the bottleneck of adversarial robustness and generalization.
\end{itemize}

\begin{figure}[h!]
    \centering

    % ==================== Left: Figure ====================
    \begin{minipage}[t]{0.49\textwidth}
        \vspace{-0.1in}
        \centering

        \includegraphics[width=0.99\linewidth]
        {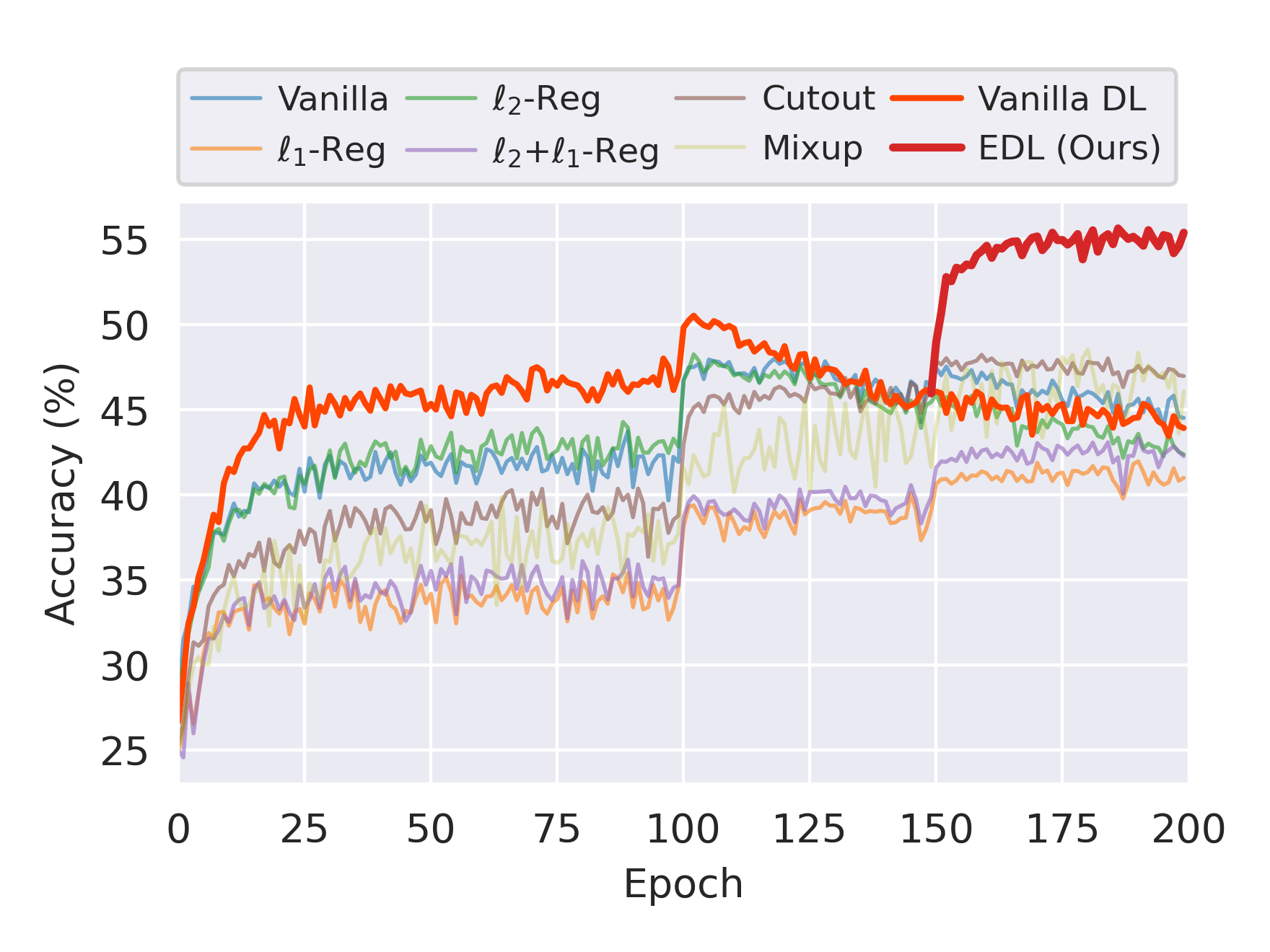}

        \vspace{-0.15in}
        \captionof{figure}{
            Test robust accuracy during the adversarial training. we pretrain
            the Vanilla DL model for 150 epochs and fine-tune the
            Elastic DL model starting from 150-th epoch. Our Elastic DL method
            can achieve the best robustness. adversarial robustness.
        }
        \label{fig:curve_resnet10_all_adv_acc_test}
    \end{minipage}
    \hfill
    % ==================== Right: Table ====================
    \begin{minipage}[t]{0.49\textwidth}
        \vspace{0pt}
        \centering

        \captionof{table}{
            Adversarial robustness on CIFAR10 with ResNet18 as backbone.
            The best performance is highlighted in \textbf{bold}.
        }
        \label{tab:cifar10-main}

        \begin{sc}
        \resizebox{\linewidth}{!}{%
            \rowcolors{2}{gray!20}{white}
            \begin{tabular}{l|c|cccc}
                \hline
                \rowcolor[HTML]{C0C0C0}
                \textbf{Method}
                & \textbf{Clean}
                & \textbf{PGD}
                & \textbf{FGSM}
                & \textbf{C\&W}
                & \textbf{AA} \\
                \hline

                MART & 79.03 & 48.90 & 60.86 & 45.92 & 43.88 \\
                SAT & 83.28 & 43.57 & 50.13 & 47.47 & 39.72 \\
                AWP & 81.20 & 51.60 & 55.30 & 48.00 & 46.90 \\
                Consistency & 84.37 & 45.19 & 53.84 & 43.75 & 40.88 \\
                DYNAT & 82.34 & 52.25 & 65.96 & 52.19 & 45.10 \\
                \hline

                PGD-AT & 80.90 & 44.35 & 58.41 & 46.72 & 42.14 \\
                + Vanilla DL & 83.28 & 45.64 & 53.88 & 41.22 & 43.70 \\
                + Elastic \red{(Ours)}
                & 83.57 & 53.22 & 69.35 & 60.80 & 52.90 \\
                \hline

                TRADES-2.0 & 82.80 & 48.32 & 51.67 & 40.65 & 36.40 \\
                + Vanilla DL & 79.05 & 40.64 & 47.12 & 41.49 & 34.90 \\
                + Elastic \red{(Ours)}
                & 79.85 & 49.32 & 58.68 & 49.47 & 47.20 \\
                \hline

                TRADES-0.2 & 85.74 & 32.63 & 44.26 & 26.70 & 19.00 \\
                + Vanilla DL & 82.55 & 25.37 & 44.48 & 30.30 & 15.30 \\
                + Elastic \red{(Ours)}
                & 84.75 & 33.61 & 57.86 & 40.68 & 28.10 \\
                \hline

                PORT & 84.59 & 58.62 & 62.64 & 58.12 & 55.14 \\
                + Vanilla DL & 82.35 & 56.40 & 60.68 & 56.77 & 54.00 \\
                + Elastic \red{(Ours)}
                & 82.76 & 59.00 & 68.54 & 60.92 & 56.30 \\
                \hline

                HAT & 85.95 & 56.29 & 61.17 & 49.52 & 53.16 \\
                + Vanilla DL & 86.42 & 57.79 & 62.67 & 51.61 & 54.30 \\
                + Elastic \red{(Ours)}
                & \textbf{86.84}
                & \textbf{62.48}
                & \textbf{71.46}
                & \textbf{59.90}
                & \textbf{59.07} \\
                \hline
            \end{tabular}%
        }
        \end{sc}
    \end{minipage}

    \vspace{-0.15in}
\end{figure}

\textbf{Adversarial training robustness.} To validate the effectiveness of our Elastic DL, we select several existing popular adversarial defenses and report the experimental
results of backbone ResNet18 under various attacks in Table~\ref{tab:cifar10-main}.  From the results we can make the following observations:
\vspace{-0.1in}
\begin{itemize}[left=0em]
    \item Our HAT + Elastic DL significantly outperforms other methods across various attacks, achieving state-of-the-art performance among all baselines.
    \item Our Elastic DL is a robust architecture that is orthogonal to existing adversarial training methods and can be combined with them to further improve robustness.
    % \item Our proposed Elastic DL consistently improve over Vanilla DL when combined with various adversarial training methods, demonstrating its effectiveness in accommodating both light-tailed and heavy-tailed noise.
\end{itemize}

\textbf{Comparison with high-performing defenses.}
We examine whether Elastic DL improves the performance of HAT~\citep{rade2022reducing} and PORT~\citep{sehwag2021robust}, two competitive defenses included in RobustBench~\citep{croce2020robustbench}. Using ResNet-18 on CIFAR-10, we evaluate their original and Elastic DL variants under $\ell_\infty$ and $\ell_2$ attacks. As shown in Tables~\ref{tab:benchmark_resnet18_cifar10_linf} and~\ref{tab:benchmark_resnet18_cifar10_l2}, adding Elastic DL improves robust accuracy in most evaluated settings, while the effect on natural accuracy varies by method.

% \vspace{-0.1in}
\begin{minipage}[h!]{0.48\textwidth}

% \begin{table}[h!]
% \vspace{-0.2in}
\centering
\captionof{table}{
% Benchmarking the state-of-the-art performance under $\ell_\infty$-norm attack of ResNet18 on CIFAR10.
% State-of-the-art 
Performance of ResNet18 on CIFAR10 under $\ell_\infty$-norm attack.
}

\begin{sc}
\vspace{-0.15in}
\begin{center}
% \begin{sc}
\resizebox{0.99\textwidth}{!}{
\setlength{\tabcolsep}{2.0pt}
\rowcolors{2}{gray!20}{white}
\begin{tabular}{l|c|ccc|ccc}
\bottomrule
\rowcolor[HTML]{C0C0C0}
\multicolumn{8}{c}{ \textbf{Performance under $\ell_\infty$-norm Attack} }\\
\hline
&\textbf{Clean}&\multicolumn{3}{c|}{ \textbf{PGD}-$\ell_\infty$}&\multicolumn{3}{c}{ \textbf{AutoAttack}-$\ell_\infty$}\\
\hline
\textbf{ Budget}&\textbf{0}&$\mathbf{\frac{8}{255}}$&$\mathbf{\frac{16}{255}}$&$\mathbf{\frac{32}{255}}$&$\mathbf{\frac{8}{255}}$&$\mathbf{\frac{16}{255}}$&$\mathbf{\frac{32}{255}}$\\
\hline
PORT&84.59&58.62 & 27.49 & 5.79& 55.14 & 17.80 & 0.30\\
+ Vanilla DL&82.35&56.40 & 27.30 & 6.38&54.00 & 20.40 & 0.60\\
+ Elastic (\red{Ours})&82.76&59.00 & 36.53 & 12.17&56.30 & 24.60 & 1.70\\
\hline
HAT&85.95&56.29&25.82&6.09&53.16& 17.20 & 0.60\\
+ Vanilla DL&86.42&57.79&26.08 & 6.07&54.30&17.56 & 0.63\\
+ Elastic (\red{Ours})&\textbf{86.84}&\textbf{62.48}&\textbf{37.66} & \textbf{13.69}&\textbf{59.10}& \textbf{25.93}& \textbf{2.10} \\
\hline

% \bottomrule
% \hline
\end{tabular}
}
% \vspace{-0.2in}
% \end{sc}
\end{center}
\end{sc}
\label{tab:benchmark_resnet18_cifar10_linf}
% \end{table}
\end{minipage}
\hfill
\begin{minipage}[h!]{0.48\textwidth}
% \begin{table}[h!]
% \vspace{-0.2in}
\centering
\captionof{table}{
% Benchmarking the state-of-the-art performance under $\ell_2$-norm attack of ResNet18 on CIFAR10.
% State-of-the-art 
Performance of ResNet18 on CIFAR10 under $\ell_2$-norm attack.
}
\vspace{-0.15in}
\begin{center}
\begin{sc}
\resizebox{0.99\textwidth}{!}{
\setlength{\tabcolsep}{2.0pt}
\rowcolors{2}{gray!20}{white}
\begin{tabular}{l|c|ccc|ccc}
\bottomrule
\rowcolor[HTML]{C0C0C0}
\multicolumn{8}{c}{ \textbf{Performance  under $\ell_2$-norm attack} }\\
\hline
&\textbf{Clean}&\multicolumn{3}{c|}{ \textbf{PGD}-$\ell_2$}&\multicolumn{3}{c}{ \textbf{AutoAttack}-$\ell_2$}\\
\hline
\textbf{ Budget}&\textbf{0}&\textbf{0.5}&\textbf{1.0}&\textbf{2.0}&\textbf{0.5}&\textbf{1.0}&\textbf{2.0}\\

\hline
PORT&88.82 &74.89 & 54.47 & 27.69 &\textbf{73.80}&48.10&5.90\\
+ Vanilla DL&87.34 &73.52 & 53.75 & 27.50&71.80 & 49.10 & 6.70\\
+ Elastic (\red{Ours})&87.81 &\textbf{75.56} & \textbf{56.76} & \textbf{31.44}&72.20 & \textbf{52.40} & \textbf{11.10} \\
\hline
HAT&89.92&74.68 & 47.67 & 21.38&72.90 & 40.80 & 2.20\\
+ Vanilla DL& 88.84&67.99 & 40.87 & 17.97&66.80 & 27.80 & 0.60\\
+ Elastic (\red{Ours})& \textbf{89.95}&74.62 & 51.41 & 27.05&73.20 & 44.50 & 3.20 \\
\hline

% \bottomrule
% \hline
\end{tabular}
}
\end{sc}
\end{center}
\label{tab:benchmark_resnet18_cifar10_l2}
% \vspace{-0.2in}
% \end{table}
\end{minipage}

\subsection{Ablation Study}

\begin{figure}[h!]
\vspace{-0.1in}
    \centering
    % Subfigure (a)
    \begin{subfigure}[b]{0.32\textwidth} % Set the width of the subfigure
        \centering
        \includegraphics[width=\textwidth]{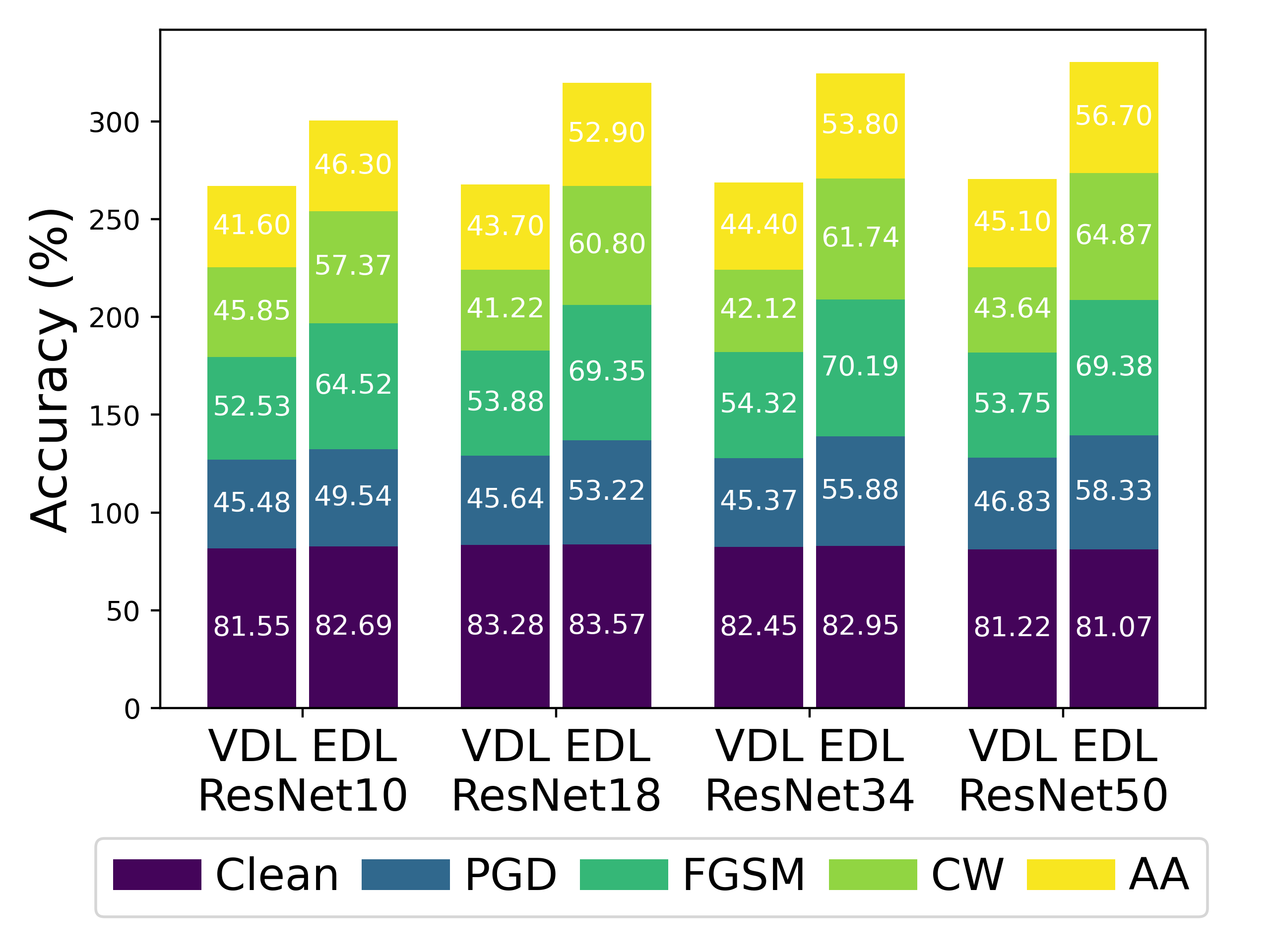} % Replace with your image path
        \caption{CIFAR10} % Subfigure caption
        \label{fig:subfig_a}
    \end{subfigure}
    \hfill
    % Subfigure (b)
    \begin{subfigure}[b]{0.32\textwidth}
        \centering
        \includegraphics[width=\textwidth]{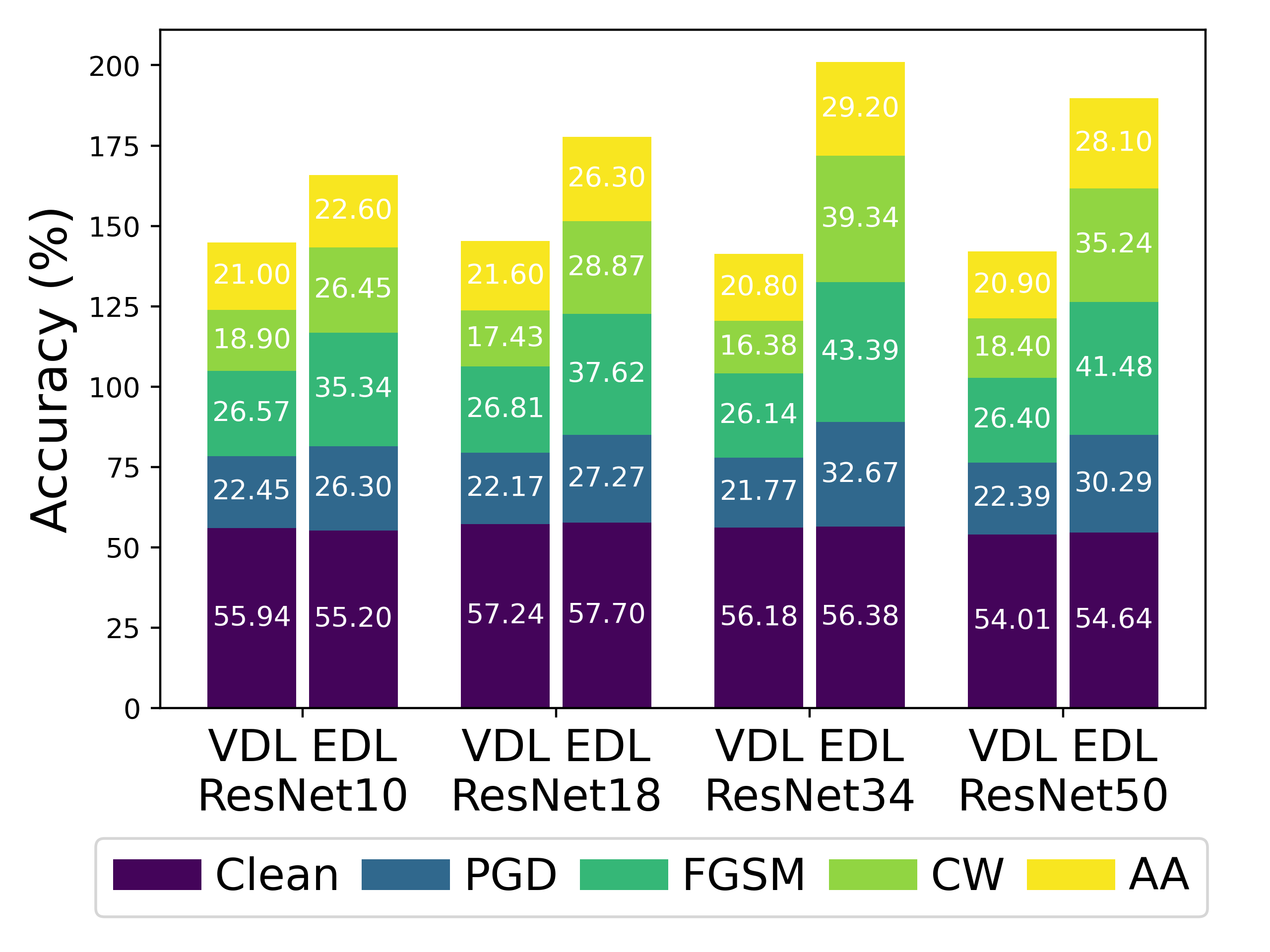} % Replace with your image path
        \caption{CIFAR100}
        \label{fig:subfig_b}
    \end{subfigure}
    \hfill
    % Subfigure (b)
    \begin{subfigure}[b]{0.32\textwidth}
        \centering
        \includegraphics[width=\textwidth]{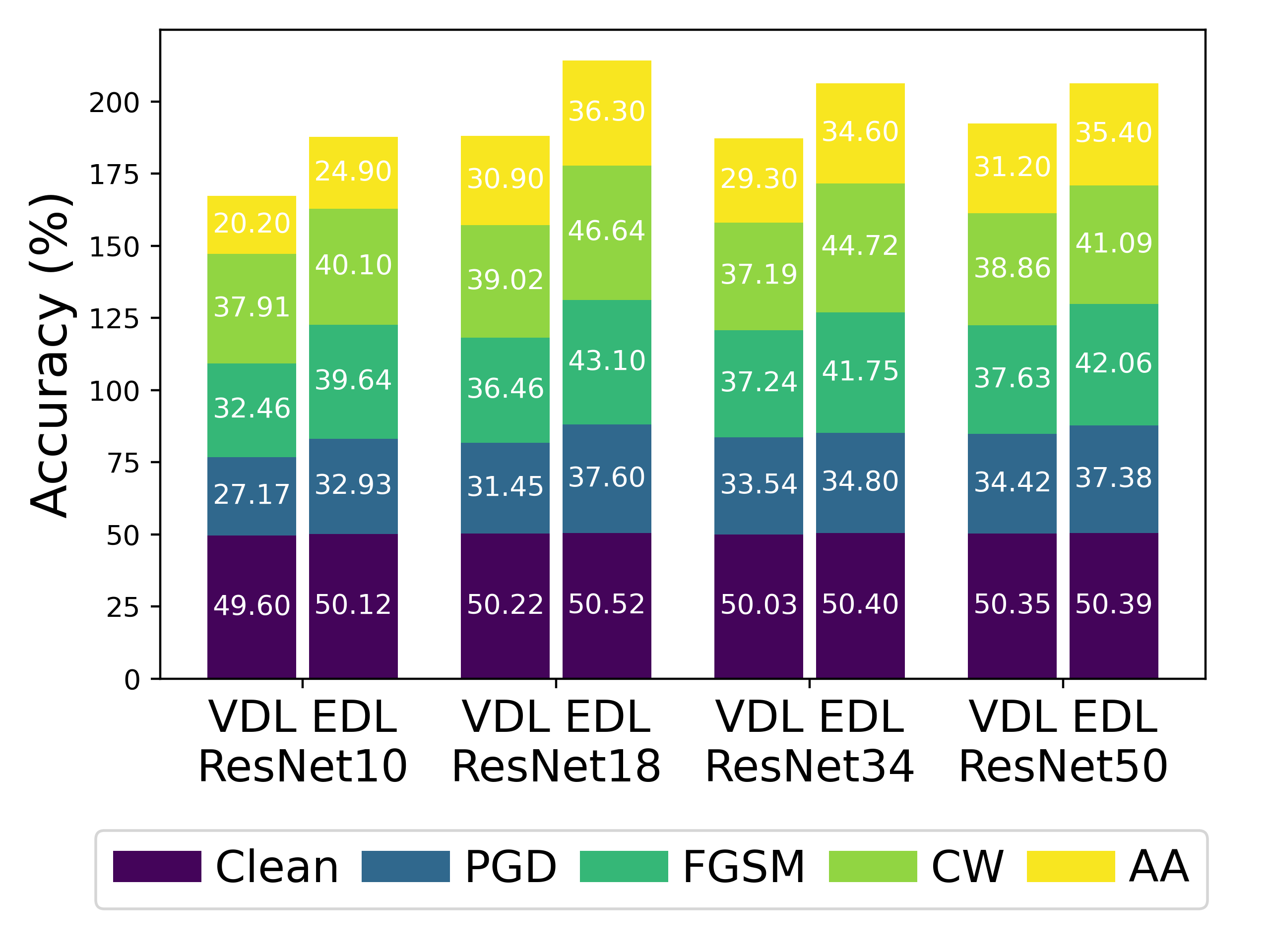} % Replace with your image path
        \caption{Tiny-ImageNet}
        \label{fig:subfig_d}
    \end{subfigure}
    % \hfill
    % % Subfigure (c)
    % \begin{subfigure}[b]{0.33\textwidth}
    %     \centering
    %     \includegraphics[width=\textwidth]{figures/Adversarial_Robustness_Diff_Adv_Train.png} % Replace with your image path
    %     \caption{Different Adversarial Training (CIFAR10)}
    %     \label{fig:subfig_c}
    % \end{subfigure}
    %     \hfill
    % % Subfigure (b)
    % \begin{subfigure}[b]{0.66\textwidth}
    %     \centering
    %     \includegraphics[width=\textwidth]{figures/diff_norm_budget.png} % Replace with your image path
    %     \caption{Different attack measurements (CIFAR10)}
    %     \label{fig:subfig_d}
    % \end{subfigure}
    % \vspace{-0.3in}
    \vspace{-0.1in}
    \caption{Adversarial robustness under various settings. Our Elastic DL outperforms Vanilla DL across various  datasets (CIFAR10 / CIFAR100 / Tiny-ImageNet), backbones (ResNet10 / ResNet18 / ResNet34 / ResNet50) and attacks (PGD / FGSM / CW / AA).
    % \xr{For all figures, we should make the font size larger (close to the text size in main paper) for readability. It is hard to read if they are such small}
    } % Main figure caption
    \label{fig:diff_datasets}
    \vspace{-0.2in}
\end{figure}

\textbf{Universality across datasets and backbones.} To validate the consistent effectiveness of our proposed methods, we conduct comprehensive abation studies on the different backbones (ResNet10, ResNet18, ResNet34, ResNet50), datasets (CIFAR10, CIFAR100, Tiny-ImageNet).
As demonstrated in the Figure~\ref{fig:diff_datasets}, Table~\ref{tab:diff_backbone_cifar10},~\ref{tab:diff_backbone_cifar100} and ~\ref{tab:diff_backbone_imagenet} in Appendix~\ref{sec:universality}, our proposed Elastic DL exhibit excellent clean performance and robustness under various attacks.

\begin{figure}[h!]
\vspace{-0.18in}
\centering
    \begin{minipage}[b]{0.32\textwidth}
\includegraphics[width=1.05\linewidth]{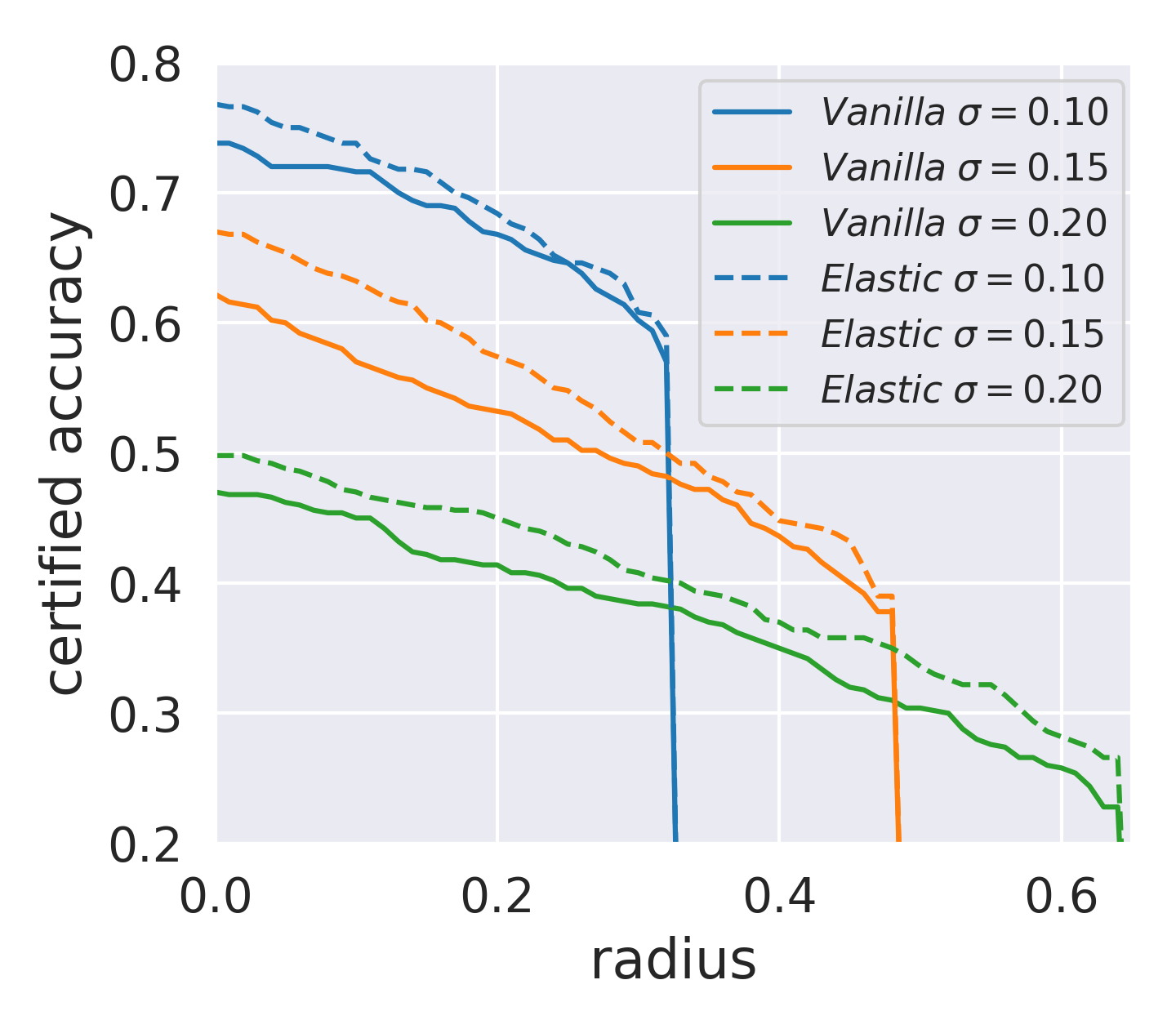}
\vspace{-0.25in}
\caption{Certifiable robustness. Elastic DL delivers better certified
robustness than vanilla DL.
}
\label{fig:certified_robustness}
    \end{minipage}
    \hfill
    \begin{minipage}[b]{0.65\textwidth}
    \centering    
    \includegraphics[width=1.0\linewidth]{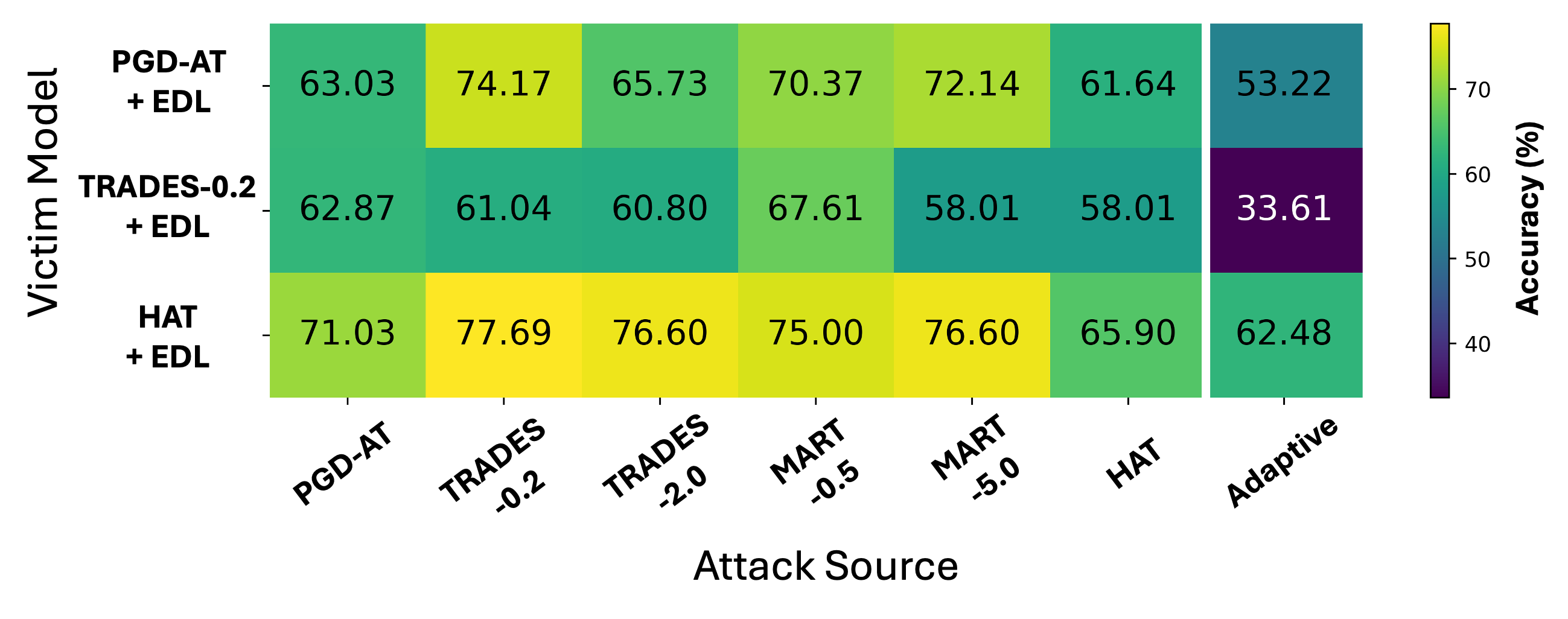}
    \caption{Transferability analysis. We evaluate the transfer attacks from multiple baselines, using the adaptive attack as a comparison, where the adaptive attack demonstrates the strongest performance. }
    \label{fig:transfer_analysis}
    \end{minipage}
    \vspace{-0.1in}
\end{figure}

% \textbf{Orthogonality to adversarial training.} Our proposed Elastic DL framework introduce structural prior into neural networks, which is orthogonal with existing adversarial training techniques. As shown in Table~\ref{tab:cifar10-main} and  Figure~\ref{fig:diff_adv_train} in Appendix~\ref{sec:orthogonality}, our Elastic DL can be combined with
% different adversarial training (PGD-AT, TRADES-2.0/0.2, PORT, HAT) to further improve the performance.

% \begin{figure}[h!]
%     \centering
%     \includegraphics[width=0.4\textwidth]{figures/Adversarial_Robustness_Diff_Adv_Train.png} % Replace with your image path
%     \caption{Different Adversarial Training (CIFAR10)}
%     \label{fig:diff_adv_train}
% \end{figure}

% \begin{figure}[h!]
%     \centering
%     \includegraphics[width=0.48\textwidth]{figures/diff_norm_budget.png} % Replace with your image path
%     \caption{Different attack measurements (CIFAR10)}
%     \label{fig:diff_measure}
% \end{figure}

\textbf{Certifiable robustness.} 
We also provide the results of certifiable robustness via randomized smoothing, which is a certified defense that can theoretically guarantee certified accuracy regardless of the evaluated attacks. The results in Figure~\ref{fig:certified_robustness} demonstrate that our Elastic DL delivers better certified robustness than the vanilla DL.

\textbf{Transferability analysis.} 
To validate the effectiveness of our method and strength of evaluated adaptive attack, we  evaluate transfer attacks from multiple baselines, along with the adaptive attack for comparison. It can be observed from  Figure~\ref{fig:transfer_analysis} that the adaptive attack yields the strongest attack, thereby validating the effectiveness of our experimental evaluation.

\begin{figure}[h!]
    \centering
    \vspace{-0.1in}
    \includegraphics[width=0.99\linewidth]{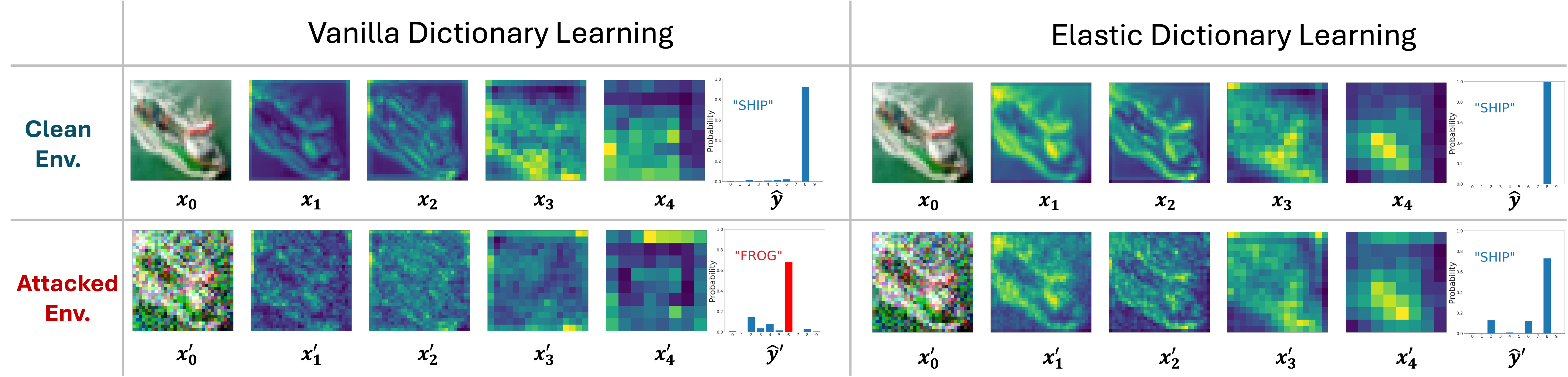}
    \vspace{0.0em}
    \caption{Hidden embedding visualization under clean and attacked scenarios. The difference between clean and attacked embeddings in Elastic DL is smaller compared to Vanilla DL, with this effect becoming more significant in deeper layers. Consequently, while an adversarial attack alters the Vanilla DL output from "\emph{SHIP}" to "\emph{FROG}", Elastic DL successfully preserves the correct prediction. }
    \label{fig:hidden_visual}
    \vspace{-0.2in}
\end{figure}

% \begin{figure*}[h!]
%     \begin{minipage}[b]{0.77\textwidth}
%     \includegraphics[width=1.0\linewidth]{figures/hidden_visual.pdf}
%     \vspace{0.0em}
%     \caption{Hidden embedding visualization under clean and attacked scenarios. The difference between clean and attacked embeddings in Elastic DL is smaller compared to Vanilla DL, with this effect becoming more significant in deeper layers. Consequently, while an adversarial attack alters the Vanilla DL output from "\emph{SHIP}" to "\emph{FROG}", Elastic DL successfully preserves the correct prediction. }
%     \label{fig:hidden_visual}
%     \end{minipage}
%     \vspace{-0.1in}
%     \hfill
%     \begin{minipage}[b]{0.22\textwidth}
%     \centering
%     \includegraphics[width=0.99\linewidth]{figures/diff_hidden_relative_all.png}
%     % \vspace{-2.5em}
%     \caption{Embedding difference. 
%     Our Elastic DL shows smaller embedding difference than Vanilla DL.
%     }
%     \label{fig:embed_diff}
%     \end{minipage}
%     % \vspace{0.1in}
% \end{figure*}

% \begin{wrapfigure}{r}{0.35\textwidth}
%     \centering
%     \vspace{-0.2in}
%     \includegraphics[width=0.85\linewidth]{figures/diff_hidden_relative_all.png}
%     % \vspace{-2.5em}
%     \vspace{-0.15in}
%     \caption{Embedding difference. 
%     % Our Elastic DL shows smaller embedding difference than Vanilla DL.
%     }
%     \vspace{-0.15in}
%     \label{fig:embed_diff}
% \end{wrapfigure}

\textbf{Hidden embedding visualization.}
We also conduct visualization analyses on the hidden embedding to obtain better insight into the effectiveness of our proposed Elastic DL. We begin by quantifying the relative difference between clean embeddings ($\bx_i$) and attacked embeddings ($\bx'_i$) across all layers, as in Figure~\ref{fig:embed_diff} in Appendix~\ref{sec:hidden_embedding_all}. Additionally, we visualize one instance in Figure~\ref{fig:hidden_visual}, with more examples provided in Appendix~\ref{sec:hidden_embedding_all}.
The results in Figure~\ref{fig:embed_diff} show that Elastic DL has smaller embedding difference across  layers, indicating that our proposed Elastic DL architecture indeed mitigates the impact of the adversarial perturbation. Moreover, 
as demonstrated in the example 
in Figure~\ref{fig:hidden_visual}, the presence of adversarial perturbations can disrupt the hidden embedding patterns, leading to incorrect predictions in 
% the case of
Vanilla DL. 
In contrast, our Elastic DL appears to lessen the effects of such perturbations and maintain predicting groundtruth label. From the figures, we can also clearly tell that the difference between clean attacked embeddings of Vanilla DL is much more significant than  in  Elastic DL.

\textbf{Different attack measurements.} In addition to $\ell_\infty$-norm attack (PGD-$\ell_\infty$), we also validate the consistent effectiveness of our Elastic DL with $\ell_2$-norm (PGD-$\ell_2$) and $\ell_1$-norm (SparseFool) attacks in the Figure~\ref{fig:diff_measure} 
% and  Table~\ref{tab:diff_budget_norm} 
in Appendix~\ref{sec:diff_measurement}.

\begin{figure}[h!]
\vspace{-0.1in}
\centering
    \begin{minipage}[b]{0.32\textwidth}
\centering
\includegraphics[width=1.05\linewidth]{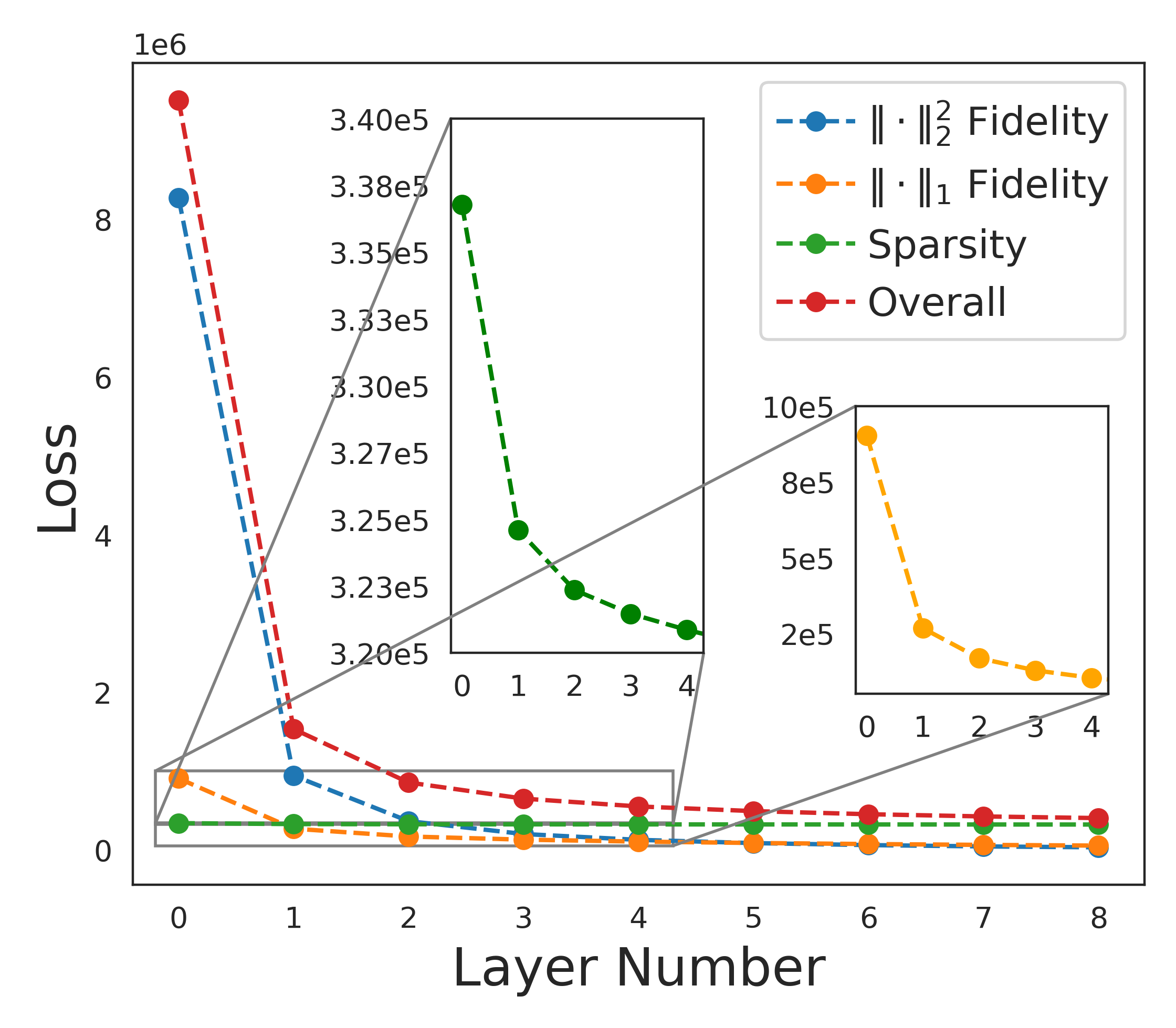}
% \vspace{-0.3in}
\caption{Algorithm convergence. 
RISTA algorithm 
% effectively reduces the dictionary learning objective and 
achieves fast convergence within just three steps.
}
\label{fig:algo_convergence}
    \end{minipage}
    \hfill
    \begin{minipage}[b]{0.32\textwidth}
\centering
\includegraphics[width=1.05\linewidth]{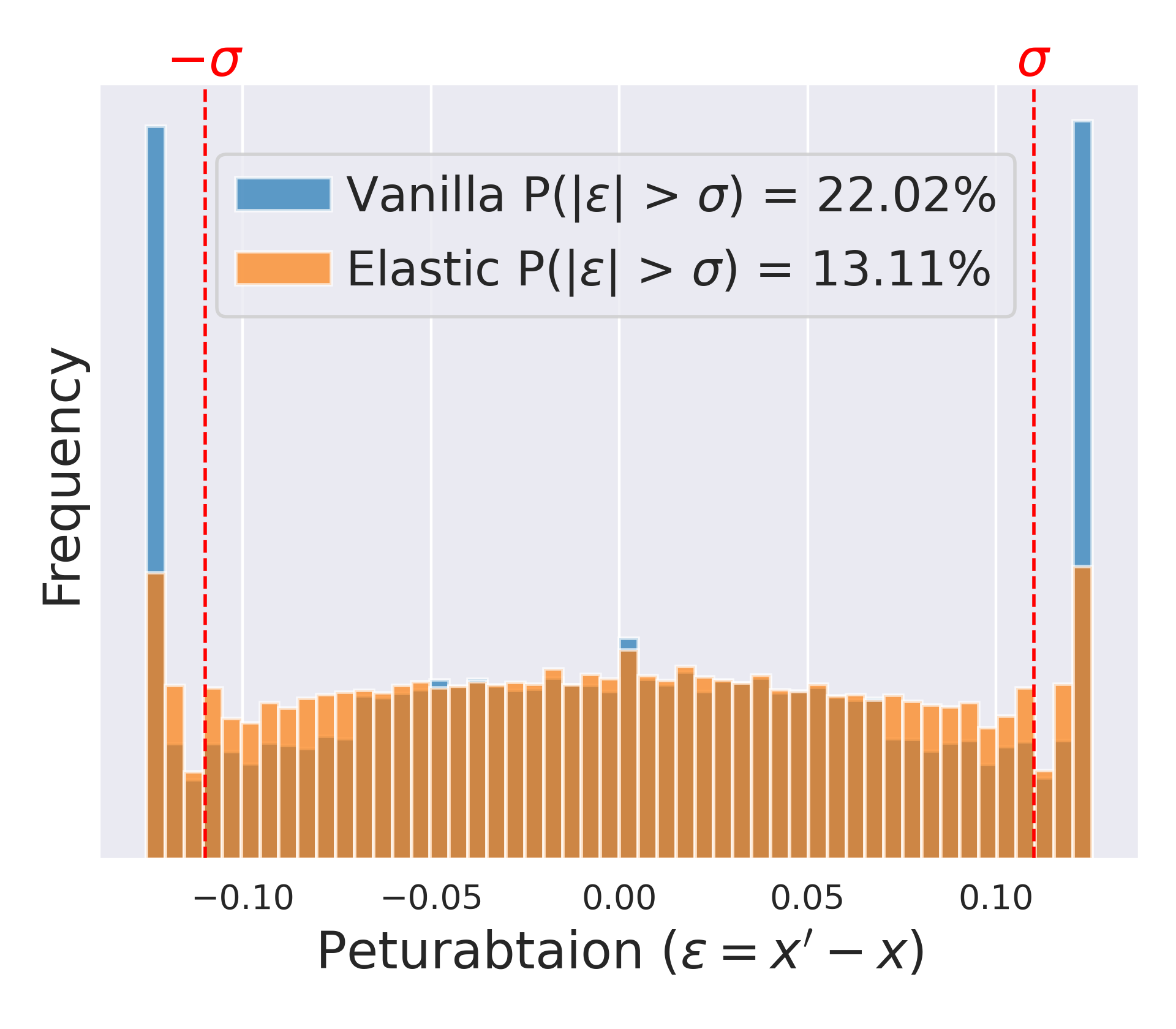}
% \vspace{-0.35in}
\caption{Attack behaviors. 
The attacker tends to attack Vanilla DL model by  introducing outlying values. 
% with heavy-tailed noise, which be effectively mitigated by the Elastic DL model. 
}
\label{fig:perturbation}
    \end{minipage}
    \hfill
    \begin{minipage}[b]{0.32\textwidth}
    \centering
    \includegraphics[width=1.05\linewidth]{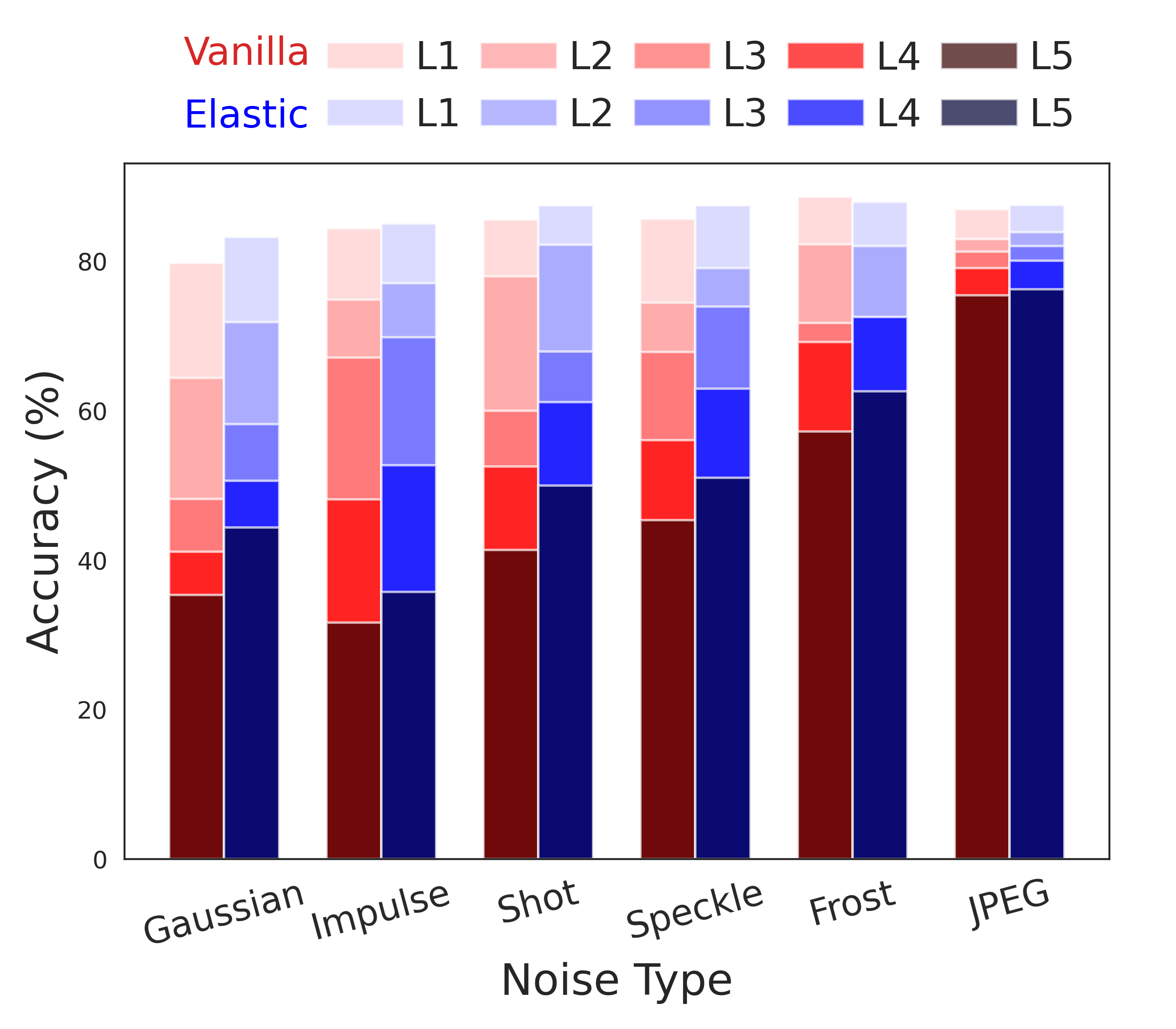}
    % \vspace{-0.3in}
    \caption{Out-of-distribution robustness.
    % Beyond in-distribution robustness, 
    Our Elastic DL also demonstrates excellent out-of-distribution robustness.
    }
    \label{fig:ood_robustness}
    \end{minipage}
\vspace{-0.1in}
\end{figure}

\textbf{Convergence.} 
To validate the effectiveness of our RISTA iterations, we plot the loss descent curves of overall objective Eq.\eqref{eq:elastic_dictionary_learning} along with the individual terms ($\|\bx - \cA^*(\bz)\|_2^2, \|\bx - \cA^*(\bz)\|_1$ and $ \|\bz\|_1$) in Figure~\ref{fig:algo_convergence}, which shows 
% It can be observed
that RISTA converges rapidly within first three steps.

% \newpage
\textbf{Attack behaviors.}
To investigate the attack behaviors, we apply the PGD attack to both models and visualize the resulting perturbations in Figure~\ref{fig:perturbation}. It can be observed that, in the Vanilla DL model, the adversarial attack introduces substantial outlying noise, which can be largely mitigated by our Elastic DL method.
% violating the light-tailed distribution assumption. This observation also explains the inferior performance of Vanilla DL compared to our Elastic DL.

% \begin{figure*}[t!]
%     \centering
%     \includegraphics[width=1.0\linewidth]{figures/hidden_visual.pdf}
%     \caption{Hidden embedding visualization of Vanilla DL and Elastic DL under clean and attacked scenarios. The difference between clean and attacked embeddings in Elastic DL is smaller compared to Vanilla DL, with this effect becoming more significant in deeper layers. Consequently, while an adversarial attack alters the Vanilla DL output from "\emph{SHIP}" to "\emph{FROG}", Elastic DL successfully preserves the correct prediction. }
%     \label{fig:hidden_visual}
% \end{figure*}

% \begin{figure}[h!]
%     \centering
%     \includegraphics[width=1.0\linewidth]{figures/idx15_hidden_visualization_resnet.png}
%     \includegraphics[width=1.0\linewidth]{figures/idx15_hidden_visualization_adv_resnet.png}
%     % \includegraphics[width=1.0\linewidth]{figures/idx15_hidden_visualization_L2.png}
%     % \includegraphics[width=1.0\linewidth]{figures/idx15_hidden_visualization_adv_L2.png}
%     \includegraphics[width=1.0\linewidth]{figures/idx15_hidden_visualization_L1.png}
%     \includegraphics[width=1.0\linewidth]{figures/idx15_hidden_visualization_adv_L1.png}
%     \caption{Hidden Embedding Visualization}
%     \label{fig:hidden_embedding_visualization}
% \end{figure}

\textbf{Out-of-distribution Robustness.} Beyond in-distribution robustness, we further validate the advantage of our proposed Elastic DL structure by evaluating the out-of-distribution performance of Vanilla DL and Elastic DL. The results in Figure~\ref{fig:ood_robustness} demonstrate the superiority of our Elastic DL over the Vanilla  DL under various types of out-of-distribution noise.

\begin{wraptable}{r}{0.56\textwidth}
\centering
\vspace{-0.19in}
\caption{Computational cost comparison.}
\label{tab:l2_vs_l12_cost}
\vspace{-0.18in}
\begin{center}
\begin{sc}
\resizebox{0.56\textwidth}{!}{
\setlength{\tabcolsep}{1.8pt}
\begin{tabular}{l|cccc}
\hline
\rowcolor[HTML]{C0C0C0}
\textbf{Method} & \textbf{Train (h)} & \textbf{Params (M)} & \textbf{FLOPs (G)} & \textbf{Size (MB)} \\\hline
Vanilla DL & 13.8 & 11.174 & 4.776 & 42.71 \\
Elastic DL & 15.2 & 11.174 & 4.776 & 42.71 \\
\hline
\end{tabular}
}
\vspace{-0.25in}
\end{sc}
\end{center}
\end{wraptable}
\textbf{Computational cost analysis.}
Table~\ref{tab:l2_vs_l12_cost} compares the computational cost of Vanilla DL and Elastic DL on CIFAR10 with ResNet18. Both models share identical parameters (11.174M), FLOPs (4.776G), and checkpoint size (42.71MB), since Elastic DL adds only the learnable weight $\beta$ and reuses the same dictionary operators as Vanilla DL. Under PGD-7 adversarial training, Vanilla DL trains for 200 epochs (13.8 GPU hours), while Elastic DL's pretraining--finetuning strategy (150 epochs Vanilla DL + 50 epochs Elastic DL) raises this to 15.2 GPU hours, only about 10\% more. Thus, Elastic DL's robustness gains come at a modest computational overhead.

\vspace{-0.1in}
\section{Related Works}
\vspace{-0.1in}

\textbf{Robust overfitting. } Overfitting in adversarially trained deep networks has been shown to significantly harm test robustness~\cite{rice2020overfitting}. To address the issue of severe robust overfitting, several efforts have been made from various perspectives. For instance, Dropout~\cite{srivastava2014dropout} is a widely used regularization method that randomly disables units and connections during training to mitigate overfitting. 
Regularization techniques~\cite{andriushchenko2020understanding,qin2019adversarial,sriramanan2020guided} have also proven effective in preventing overfitting by penalizing the complexity of model parameters.
Data augmentation is another common approach for reducing overfitting in deep network training~\cite{schmidt2018adversarially}, with methods including Cutout~\cite{devries2017improved}, Mixup~\cite{zhang2017mixup}, semi-supervised learning techniques~\cite{carmon2019unlabeled, zhai2019adversarially}, and generative modeling~\cite{wang2023better,gowal2021improving} being particularly notable. Additionally, early stopping~\cite{rice2020overfitting} has demonstrated great effectiveness in achieving optimal robust performance during adversarial training. However, existing methods have yet to fully realize the potential of structural priors for improving adversarial robustness and generalization.

\textbf{Dictionary learning prior in deep learning.} Dictionary learning has been well-studied and widely applied in signal and image processing~\cite{olshausen1996emergence, wright2008robust, wright2020dense, zhao2011background, yang2011robust, lu2013online, chen2013robust, jiang2015robust, yang2011robust}, based on the assumption that an input signal can be represented by a few atoms from a dictionary. Building on this foundation, \citet{papyan2017convolutional, cazenavette2021architectural, mahdizadehaghdam2019deep, li2022revisiting} successfully incorporated dictionary learning into deep learning to interpret or replace the "black-box" nature of neural networks. 
While these methods have demonstrated promising generalization and robustness against random noise and universal attacks~\cite{li2022revisiting, mahdizadehaghdam2019deep}, their practical benefits for improving robustness under adaptive attacks are yet to be thoroughly investigated.

\vspace{-0.15in}
\section{Conclusion}
\vspace{-0.15in}
This paper proposes an orthogonal direction to break through the current plateau of adversarial robustness. We begin by revisiting dictionary learning in deep learning and reveal its limitations in dealing with outliers 
% under the single noise distribution assumption 
and its vulnerability to adaptive attacks. 
As a countermeasure, we propose a novel elastic dictionary learning approach along with an efficient RISTA algorithm unrolled as an 
% implicit
convolutional neural layers. Our comprehensive experiments demonstrate that our method achieves remarkable robustness, surpassing state-of-the-art baselines available on the robustness leaderboard.
To the best of our knowledge, this is the first work to discover and validate that structural prior can reliably enhance adversarial robustness and generalization, unveiling a promising direction for future research. 

\newpage

\bibliography{iclr2027_conference}
\bibliographystyle{iclr2027_conference}

\newpage
\appendix
\onecolumn

\section{Overview of Elastic Dictionary Learning}
\label{sec:overview_edl}

\textbf{Overview of Elastic DL neural networks.}
Having derived RISTA as an efficient solver for the elastic dictionary learning objective in Eq.~\eqref{eq:elastic_dictionary_learning}, we now turn to how this algorithm is used to construct a trainable neural architecture. We present the complete forward pass of an Elastic Dictionary Learning Network (EDLNet) in Algorithm~\ref{alg:elastic-dlnn}. Specifically, we replace each conventional convolutional layer with an Elastic Dictionary Learning (EDL) layer, whose underlying optimization problem is approximately solved by unrolling $T$ iterations of the proposed RISTA algorithm. The resulting architecture can be trained end-to-end, jointly optimizing the dictionary operators $\{\cA^{(l)}\}_{l=0}^{L-1}$, layer-wise balancing weights $\{\beta^{(l)}\}_{l=0}^{L-1}$, and classifier parameters $\theta$.

\begin{algorithm}[H]
    \caption{Forward Pass of an Elastic DL Network}
    \label{alg:elastic-dlnn}
    \begin{algorithmic}[1]
        \STATE \textbf{Input:} Image $\bx$
        \STATE \textbf{Initialize:} $\bx^{(0)} \leftarrow \bx$

        \FOR{$l = 0$ \TO $L-1$}
            \STATE
            $\bz_{0}^{(l)}
            \leftarrow
            \cA^{(l)}\!\left(\bx^{(l)}\right)$
            \FOR{$t = 0$ \TO $T-1$}
                \STATE $\br_{t}^{(l)} \leftarrow \bx^{(l)}-\cA^{(l)*}\!\left(\bz_{t}^{(l)}\right), \bw_{t}^{(l)} \leftarrow \frac{1}{2\left(\left| \br_{t}^{(l)} \right| + \epsilon\right)}$
                \STATE $\bm{g}_t^{(l)}  \leftarrow  \cA^{(l)}\left((\beta^{(l)} \mathbf{1} + (1-\beta^{(l)}) \bw_t^{(l)})\odot \br_{t}^{(l)}\right)$
                \STATE $\displaystyle \bz_{t+1}^{(l)} \leftarrow \mathcal{T}_{\lambda \eta}\!\left( \bz_{t}^{(l)} + \eta\cdot
                \bg_t^{(l)}
                \right)$
            \ENDFOR

            \STATE
            $\bx^{(l+1)}
            \leftarrow
            \bz_{T}^{(l)}$
        \ENDFOR

        \STATE
        \textbf{Output:}
        $\hat{y}
        \leftarrow
        \operatorname{MLP}_{\theta}\!\left(\bx^{(L)}\right)$
    \end{algorithmic}
\end{algorithm}

% \textbf{Exploded view of Elastic DL layer.} We also provide an exploded view of each Elastic DL layer in Figure~\ref{fig:edl_layer}. The input signal $\bz^{(l)}$ can be represented by a linear superposition of several atoms $\{\balpha_{dc}\}$ from a convolutional dictionary $\vA^{(l)}$. Each EDL layer is unrolled using the proposed RISTA algorithm, which approximates the solution to the elastic dictionary learning objective.

% \begin{figure}[h!]
%     \centering
%     \includegraphics[width=\textwidth]{figures/edl_layer_v2.png}
%     \caption{Exploded view of Elastic DL (EDL) layer.}
%     \label{fig:edl_layer}
% \end{figure}

\newpage
\section{Theoretical Proof}

\subsection{Proof of Lemma~\ref{lemma:local_upper_bound}}
\label{sec:proof_local_upper_bound}

\begin{proof}

Since $\sqrt{a}\leq\frac{a}{2\sqrt{b}}+\frac{\sqrt{b}}{2}$ and equality holds when $a=b$, substituting $a=(\bx[i,j,c]-\cA^*(\bz)[i,j,c])^2$ and $b=(\bx[i,j,c]-\cA^*(\bz_*)[i,j,c])^2$ gives
\begin{align}
    &|\bx[i,j,c]-\cA^*(\bz)[i,j,c]|\\
    \leq& \frac{1}{2}\cdot \frac{1}{|\bx[i,j,c]-\cA^*(\bz_*)[i,j,c]|}\cdot (\bx[i,j,c]-\cA^*(\bz)[i,j,c])^2+\frac{1}{2}|\bx[i,j,c]-\cA^*(\bz_*)[i,j,c]|\\
    =&\bw[i,j,c]\cdot (\bx[i,j,c]-\cA^*(\bz)[i,j,c])^2+\frac{1}{2}|\bx[i,j,c]-\cA^*(\bz_*)[i,j,c]|
\end{align}

Summing both sides over $(i,j,c)$, we obtain
\begin{align*}
\cR(\bz)&=\|\bx-\cA^*(\bz)\|_1 = \sum_{i,j,c}|\bx[i,j,c]-\cA^*(\bz)[i,j,c]|\\
   &\leq
    \sum_{i,j,c}\bw[i,j,c]\cdot (\bx[i,j,c]-\cA^*(\bz)[i,j,c])^2+\frac{1}{2} \sum_{i,j,c}|\bx[i,j,c]-\cA^*(\bz_*)[i,j,c]|\\
  &=\|\bw^{1/2}\odot (\bx-\cA^*(\bz))\|_2^2+\frac{1}{2}\cR(\bz_*)\\
  &=\cU(\bz,\bz_*)
\end{align*}

and equality holds at $a=b$ (i.e., $\bz=\bz_*$):
\begin{equation}
\label{eq:local_equality}
    \cU(\bz_*,\bz_*) = \cR(\bz_*).
\end{equation}

\end{proof}

\subsection{Proof of Algorithm Iteration in Eq.~\eqref{eq:algo_iteration_l1}}
\label{sec:proof_algo_iteration}

Here, we derive the algorithm for Elastic dictionary learning (Elastic DL).
\begin{proof}
For the convex objective
$$f(\bz)=\frac{\beta}{2}\|\bx - \cA^*(\bz)\|_2^2 +\frac{1-\beta}{2} \|(\bw^{(t)})^{1/2} \odot \left(\bx - \cA^*(\bz)\right)\|_2^2,$$
we can reach the optimum via first-order gradient descent:
$$ \bz_{t+1}=\bz_t-t\nabla f(\bz_t),$$
or equivalently,
$$\bz_{t+1}=\argmin_\bz\{f(\bz_t)+\langle\bz-\bz_t,\nabla f(\bz_t)\rangle+\frac{1}{2t}\|\bz-\bz_t\|^2\}.$$

For the corresponding $\ell_1$-regularized problem
$$\min_\bz f(\bz) + \lambda \|\bz\|_1,$$
we have
\begin{align*}
   \bz_{t+1}&=\argmin_\bz\{ f(\bz_t)+\langle\bz-\bz_t,\nabla f(\bz_t)\rangle+\frac{1}{2t}\|\bz-\bz_t\|^2+\lambda\|\bz\|_1\} \\
   &= \argmin_\bz\{\frac{1}{2t}\|\bz-(\bz_t-t\nabla f(\bz_t))\|^2+\lambda\|\bz\|_1\}\\
   &= \argmin_\bz\{g(\bz):=\frac{1}{2t}\|\bz-\by\|^2+\lambda\|\bz\|_1\} \quad (\by=\bz_t-t\nabla f(\bz_t))
\end{align*}
The optimality condition is:
\begin{align*}
    &0\in\partial_\bz g(\bz^*) =  \frac{1}{t}(\bz^*-\by)+\lambda \text{sign}(\bz^*)\\
    \Leftrightarrow \quad  & \by \in \bz^* + \lambda t \text{sign}(\bz^*) \\
    \Leftrightarrow \quad  & \by \in \left(\text{Id} + \lambda t \text{sign}(\cdot)\right)(\bz^*) \\
    \Leftrightarrow \quad  & \bz^* = \cT_{\lambda t}(\by) := \left(\text{Id} + \lambda t \text{sign}(\cdot)\right)^{-1}(\by) = \text{sign}(\by) \left(|\by|-\lambda t\right)_+.
\end{align*}
Since
\begin{align*}
\nabla f(\bz_t)&=-\beta\cA(\bx - \cA^*(\bz_t)) - (1-\beta)\cA\left( \bw^{(t)} \odot \left(\bx - \cA^*(\bz_t)\right)\right)\\
&=-\cA \left(\left(\beta \mathbf{1} + (1-\beta) \bw^{(t)}\right)\odot\left(\bx - \cA^*(\bz_t)\right)\right),
\end{align*}
we have
\begin{align*}
\bz_{t+1}&=\bz^*=\mathcal{T}_{\lambda t}\left(\by \right)=\mathcal{T}_{\lambda t} \left( \bz_t - t \cdot \nabla f(\bz_t)  \right)\\
&=\mathcal{T}_{\lambda t} \left( \bz_t + t \cdot \cA \left(\left(\beta \mathbf{1} + (1-\beta) \bw^{(t)}\right)\odot\left(\bx - \cA^*(\bz_t)\right)\right)   \right).
\end{align*}

\end{proof}

\newpage
\section{Related Works}
\label{sec:related_works_app}
\subsection{Adversarial Attacks}
Adversarial attacks are typically classified into two main categories: \textit{white-box} and \textit{black-box} attacks. In white-box attacks, the attacker has full knowledge of the target neural network, including its architecture, parameters, and gradients. Common examples of white-box attacks include gradient-based methods such as FGSM~\citep{goodfellow2014explaining}, DeepFool~\citep{moosavi2016deepfool}, PGD~\citep{madry2017towards},  C\&W attack~\citep{carlini2017towards} and Spiking attack~\citep{hou2026fine}.
In contrast, black-box attacks operate under limited information, where the attacker can only interact with the model through its input-output behavior without direct access to internal details. Examples of black-box methods include surrogate model-based approaches~\citep{papernot2017practical}, zeroth-order optimization techniques~\citep{chen2017zoo}, and query-based methods~\citep{andriushchenko2020square, alzantot2019genattack}.

Below we list details of the attacks used in the main paper:
\begin{itemize}[left=0.0em]
    \item Fast Gradient Sign Method (FGSM)~\citep{goodfellow2014explaining}: FGSM is one of the earliest and most widely used adversarial attack methods. It generates adversarial examples using the gradient of the loss function with respect to the input to craft small but purposeful perturbations that lead the model to make incorrect predictions.
    \item Projected Gradient Descent (PGD)~\citep{madry2017towards}: PGD is an iterative and more robust extension of FGSM. It repeatedly applies small perturbations within a defined range (an $\epsilon$-ball) to maximize the model's loss, and is widely considered a strong adversary for evaluating model robustness.
    \item Carlini \& Wagner Attack (C\&W)~\citep{carlini2017towards}: This attack crafts adversarial examples by optimizing a custom loss function designed to minimize perturbation magnitude while ensuring the resulting sample is misclassified.
    \item AutoAttack~\citep{croce2020reliable}: AutoAttack is an ensemble of adversarial attack methods that automatically evaluates model robustness. It combines several attacks to provide a strong, reliable benchmark without manual tuning.
    \item SparseFool~\citep{modas2019sparsefool}: SparseFool is a sparse adversarial attack that perturbs only a few pixels of the input image, highlighting how minimal changes can significantly alter model predictions.
\end{itemize}

\subsection{Adversarial Defenses}

Significant efforts have been devoted to enhancing model robustness through a variety of strategies, including detection techniques~\citep{metzen2017detecting,feinman2017detecting,grosse2017statistical,sehwag2021robust,rade2022reducing,addepalli2022scaling}, purification-based approaches~\citep{ho2022disco,nie2022diffusion,shi2021online,yoon2021adversarial}, robust training methods~\citep{madry2017towards,zhang2019theoretically,gowal2021improving,li2023wat},  regularization-based techniques~\citep{cisse2017parseval,zheng2016improving}, and robustness-informed architecture~\citep{hou2024protransformer,hou2024robust,hou2025robustness}. Among these, adversarial training-based methods~\citep{sehwag2021robust,rade2022reducing,addepalli2022scaling} have proven highly effective against adaptive adversarial attacks, consistently leading the robustness leaderboard (RobustBench)~\citep{croce2020robustbench}.
Despite their success, most existing methods rely heavily on extensive synthetic training data generated by advanced generative models, larger network architectures, and empirically driven training strategies. These dependencies pose substantial challenges to advancing beyond the current plateau in adversarial robustness.
In this work, we introduce an elastic dictionary framework that incorporates structural priors into model design. This approach is fully orthogonal to existing methods and offers a complementary pathway to further enhance robustness when integrated with current techniques.

Below we list details of the adversarial training baselines used in the main paper:
\begin{itemize}[left=0.0em]
    \item PGD-AT~\citep{madry2017towards}: Projected Gradient Descent Adversarial Training (PGD-AT) is a fundamental adversarial training approach that enhances model robustness by iteratively generating adversarial examples using PGD and training the model on them.

    \item TRADES~\citep{zhang2019theoretically}: TRADES (Tradeoff-inspired Adversarial Defense via Surrogate Loss Minimization) balances robustness and accuracy by introducing a regularization term that penalizes the discrepancy between natural and adversarial predictions.

    \item MART~\citep{wang2019improving}: Misclassification-Aware Adversarial Training (MART) improves robustness by assigning higher weights to misclassified examples, emphasizing correctly classified samples' robustness.

    \item SAT~\citep{huang2020self}: Self-Adaptive Training (SAT) refines adversarial training by adjusting the training process based on the model's confidence, mitigating the effects of incorrect labels and improving generalization.

    \item AWP~\citep{wu2020adversarial}: Adversarial Weight Perturbation (AWP) enhances robustness by perturbing model parameters within a constrained space to improve worst-case performance against adversarial attacks.

    \item Consistency~\citep{tack2022consistency}: Consistency training leverages perturbation-invariant representations to enhance robustness by enforcing consistent predictions across different transformations of the input.

    \item DYNAT~\citep{liu2024dynamic}: Dynamic Adversarial Training (DYNAT) adapts training strategies dynamically based on model performance, balancing robustness and generalization efficiency.

    \item PORT~\citep{sehwag2021robust}: Proxy Distribution-based Robust Training (PORT) leverages data from proxy distributions, such as those generated by advanced generative models, to enhance adversarial robustness. By formally analyzing robustness transfer and optimizing training, PORT demonstrates significant improvements in robustness under various threat models.

    \item HAT~\citep{rade2022reducing}: Helper-based Adversarial Training (HAT) mitigates the accuracy-robustness trade-off by incorporating additional incorrectly labeled examples during training, reducing excessive margin changes along certain adversarial directions and achieving a better trade-off than existing methods.
\end{itemize}

\newpage
\section{Training Strategy}
\label{sec:training_strategy}

To identify a more effective training strategy, we evaluate two common training paradigms:
\begin{itemize}[left=0.0em]
    \item \textbf{Training from scratch.} We initialize $\beta=0.5$ in Elastic DL and train the model from scratch with a decaying learning-rate schedule.

    \item \textbf{Pretraining--Finetuning.} We first train Vanilla DL for 150 epochs until performance stabilizes. We then replace the forward module with Elastic DL, initialize $\beta=0.5$, and finetune the model for an additional 50 epochs.
\end{itemize}

Figure~\ref{fig:training_strategy} compares the performance of these two strategies. We observe that training Elastic DL from scratch leads to inferior performance. At the early stage of training, both the backbone parameters and the layer-wise balancing parameters $\beta$ are optimized jointly under a relatively large learning rate, which makes optimization unstable and significantly affects final model quality. In particular, this early-stage instability degrades clean accuracy relative to Vanilla DL; since clean accuracy often underpins strong robust accuracy, this drop also hurts final robustness.

Therefore, in our main experiments, we adopt the \textbf{Pretraining--Finetuning} strategy.

\begin{figure}[h!]
    \centering
    \includegraphics[width=0.5\linewidth]{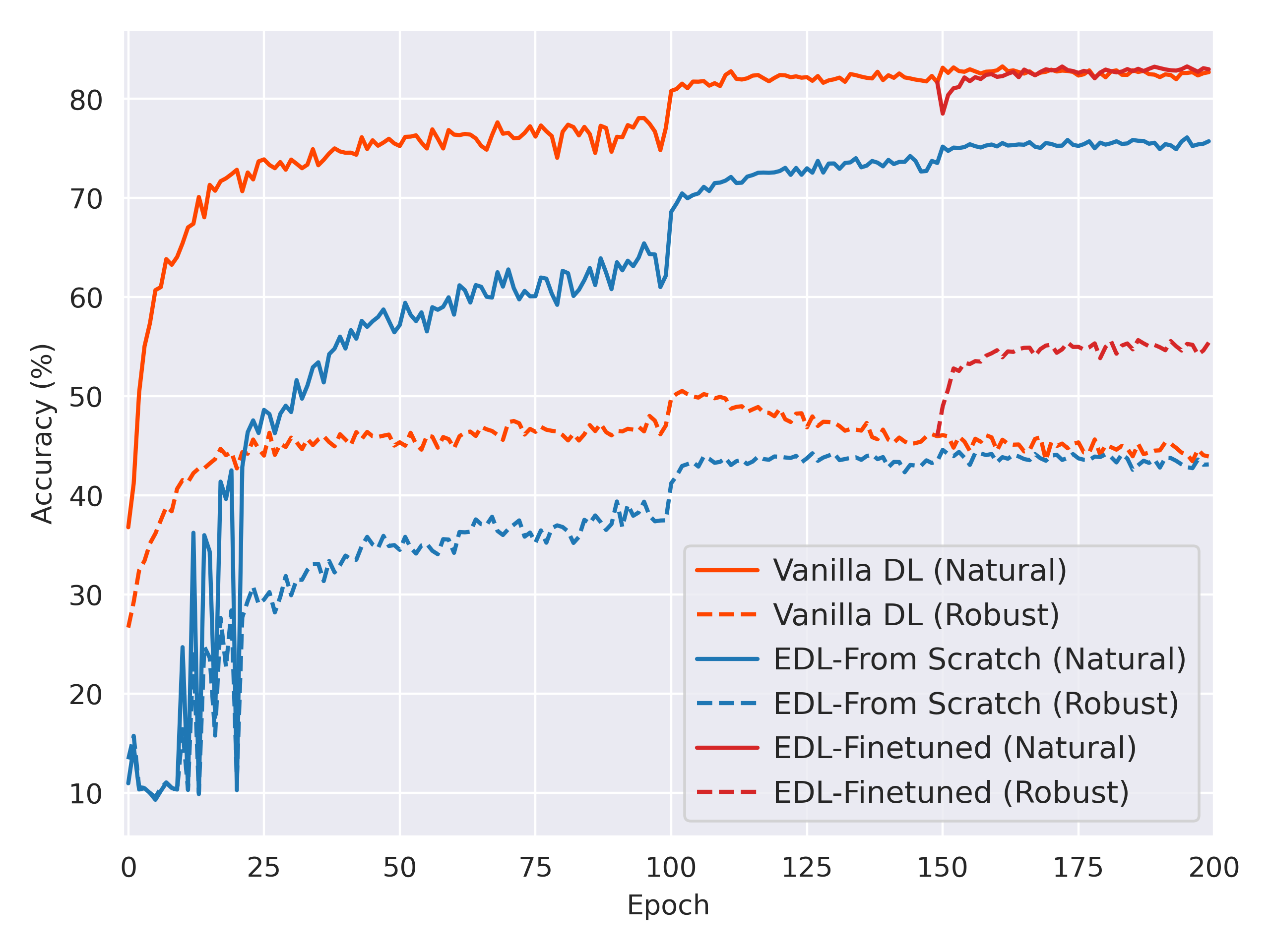}
    \caption{Comparison of training-from-scratch and pretraining--finetuning strategies for Elastic DL. Pretraining--finetuning achieves more stable optimization and better final performance.}
    \label{fig:training_strategy}
\end{figure}

% \newpage
\section{Additional Experiments}

\subsection{Preliminary Studies}
\label{sec:pre_study}

\textbf{Preliminary study on Vanilla DL robustness.}
We evaluate Vanilla DL (with both fixed and tuned $\lambda$) under random impulse noise and adaptive PGD adversarial attacks~\citep{madry2017towards}. As illustrated in Figure~\ref{fig:pre_sdnet18_various_lambda}, Vanilla DL's accuracy declines as the noise level increases and its distribution becomes more heavy-tailed. While tuning the sparsity weight $\lambda$ improves resilience to random noise, models with any $\lambda$ suffer a sharp performance drop under the adaptive PGD attack, with accuracy nearing zero.

\begin{figure}[h!]
    \centering
    \includegraphics[width=0.4\textwidth]{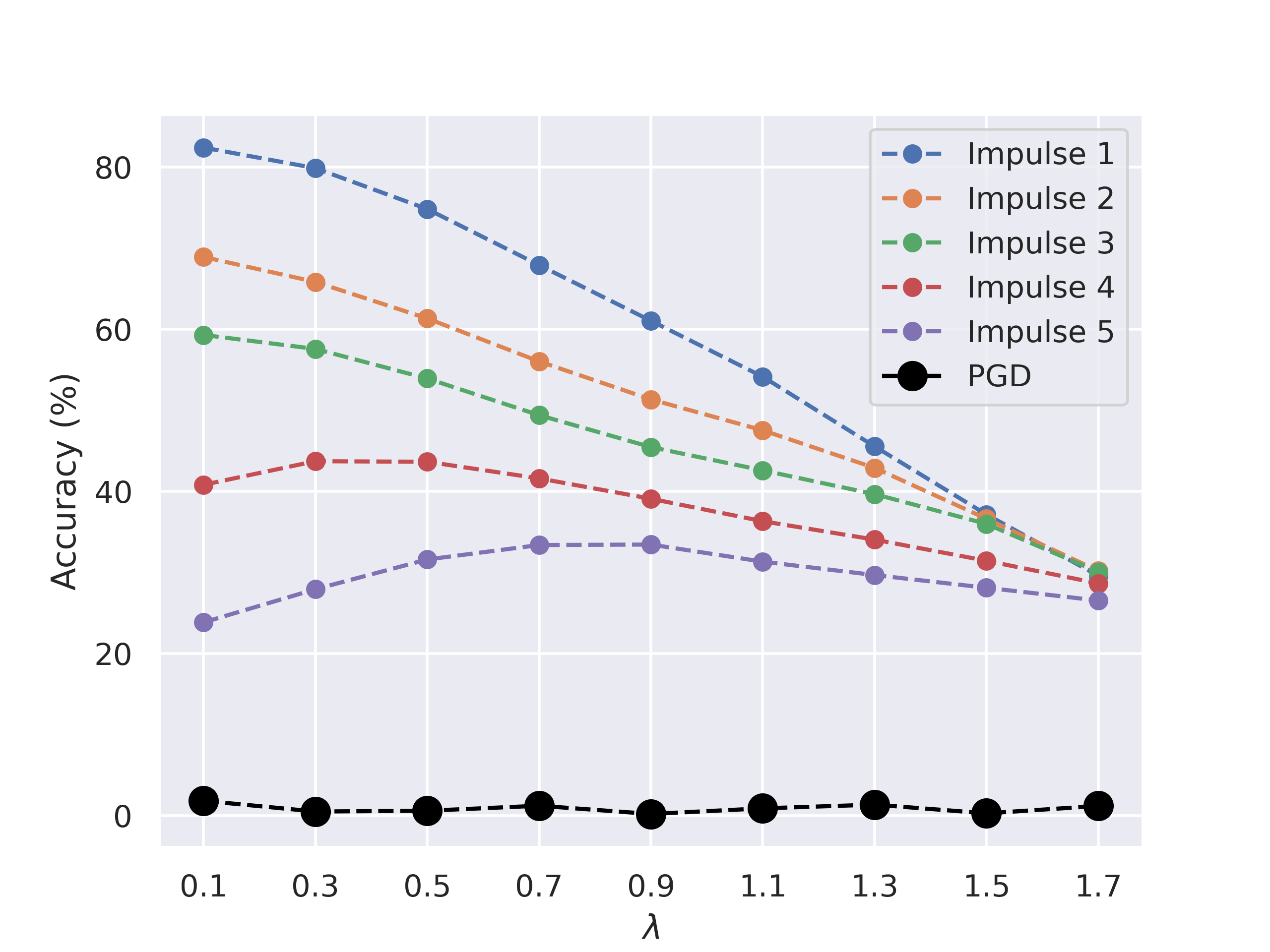}
    \caption{Performance of the Vanilla DL variant of ResNet18 under random impulse noise at different levels.}
    \label{fig:pre_sdnet18_various_lambda}
\end{figure}

\newpage
\subsection{Adversarial Training Curves}

\subsubsection{Training Curves of Each Method}
\label{sec:training_curves_each_method}

\textbf{Training curve of our Elastic DL.} From Figure~\ref{fig:adv_train_curve}, we observe that during epochs 100--150, the Vanilla DL model exhibits a severe \emph{robust overfitting} phenomenon: training performance improves while test robust accuracy drops significantly. After incorporating our Elastic DL structural prior at epoch 150, both training and testing robustness improve substantially. Although natural performance dips slightly during the initial switching period, it recovers quickly within a few epochs. This highlights the promising potential of the Elastic DL structural prior for breaking through the bottleneck of adversarial robustness and generalization.

\begin{figure}[h!]
    \centering
    \includegraphics[width=0.5\linewidth]{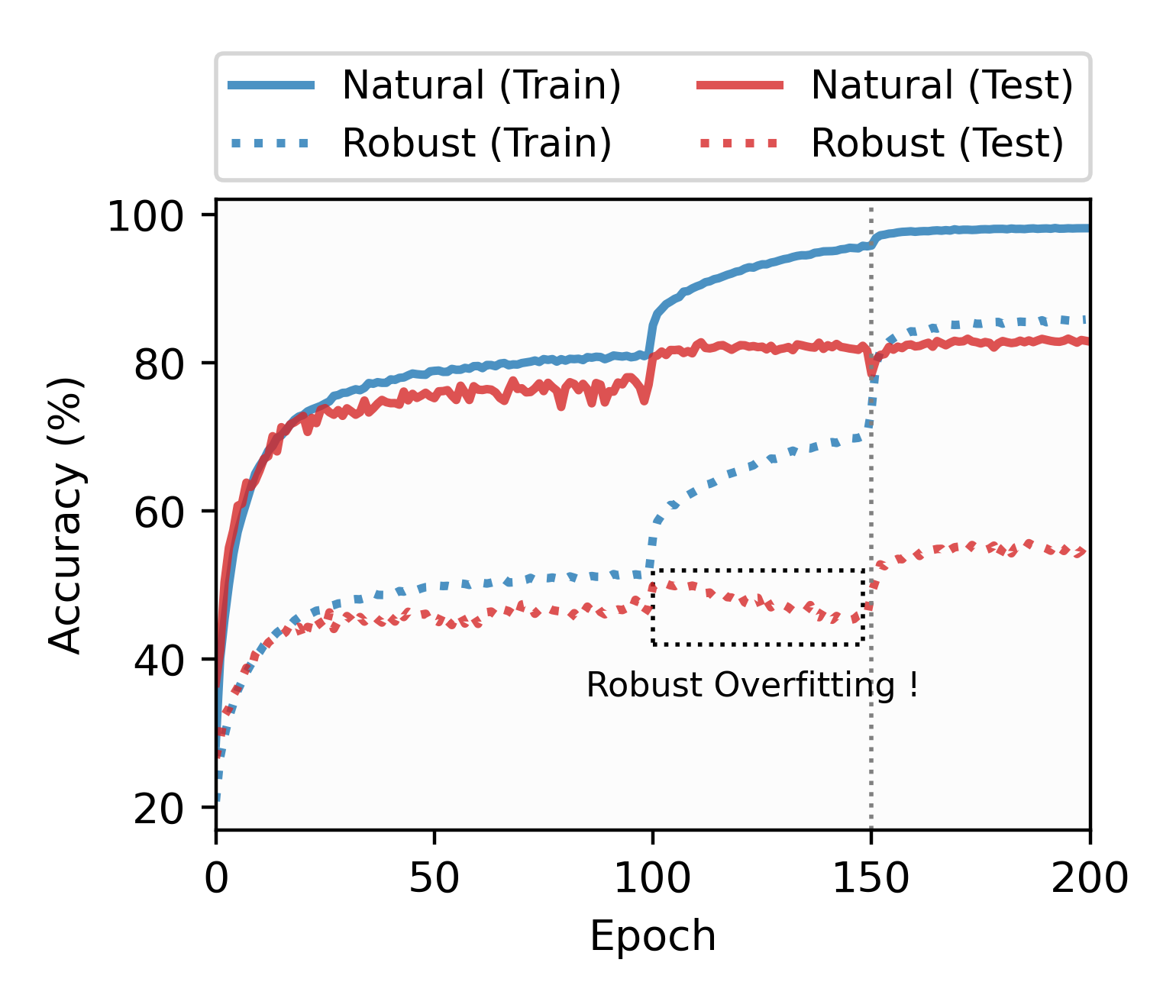}
    \caption{Adversarial training curve of our Elastic DL. During epochs 100--150, the model experiences a catastrophic \emph{robust overfitting} problem. Introducing the Elastic DL structural prior at epoch 150 and fine-tuning effectively mitigates overfitting and yields significantly improved robustness and generalization.}
    \label{fig:adv_train_curve}
\end{figure}

\newpage
\textbf{Training curves of baseline methods.}
We track the training curves of the baselines, including regularization ($\ell_1$, $\ell_2$, and their combination), Cutout~\citep{devries2017improved}, and Mixup~\citep{zhang2017mixup}, in Figure~\ref{fig:training_curves_each_method}.

\begin{figure}[h!]
    \centering
    \begin{subfigure}[b]{0.36\textwidth}
        \centering
        \includegraphics[width=\textwidth]{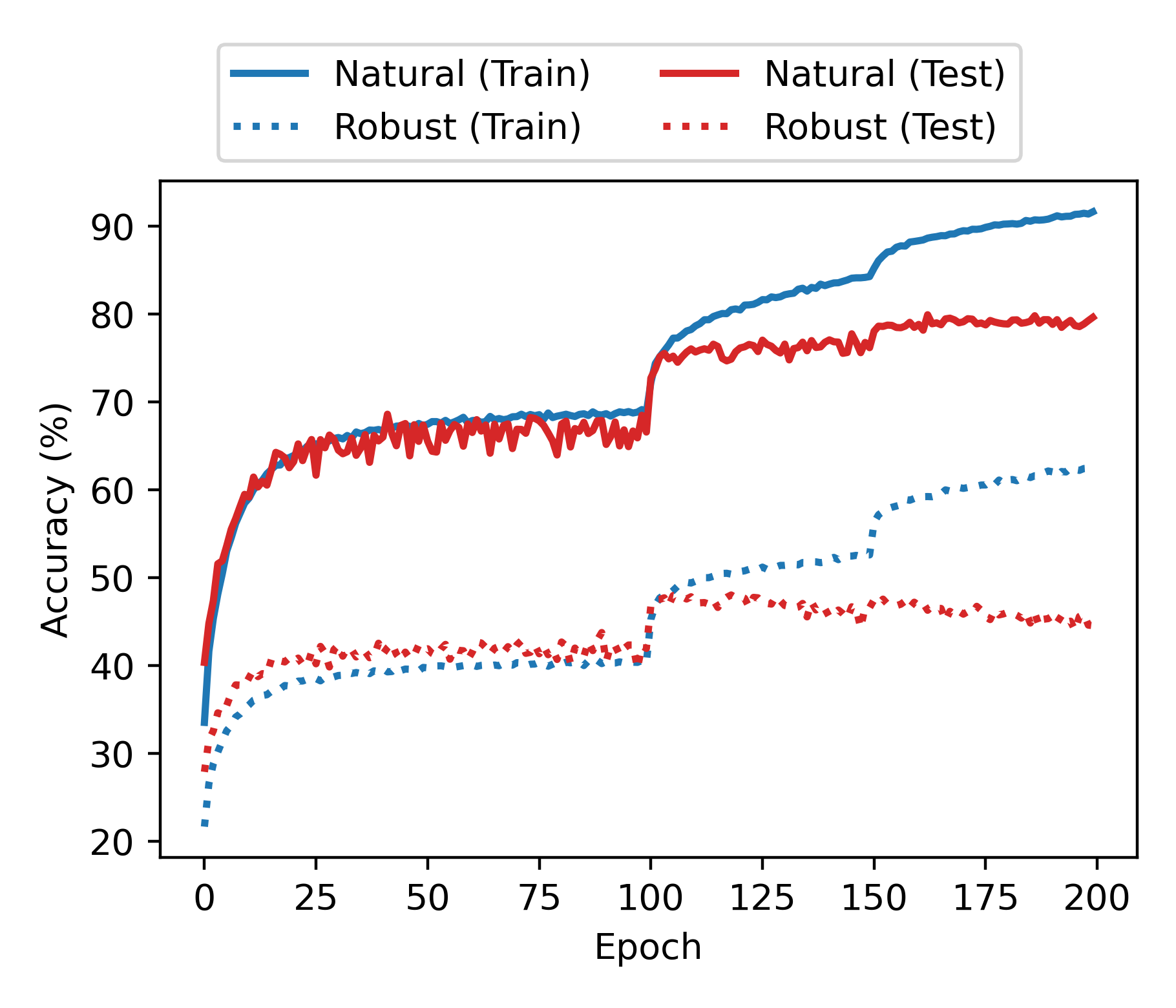}
        \caption{Vanilla}
    \end{subfigure}
    \hfill
    \begin{subfigure}[b]{0.36\textwidth}
        \centering
        \includegraphics[width=\textwidth]{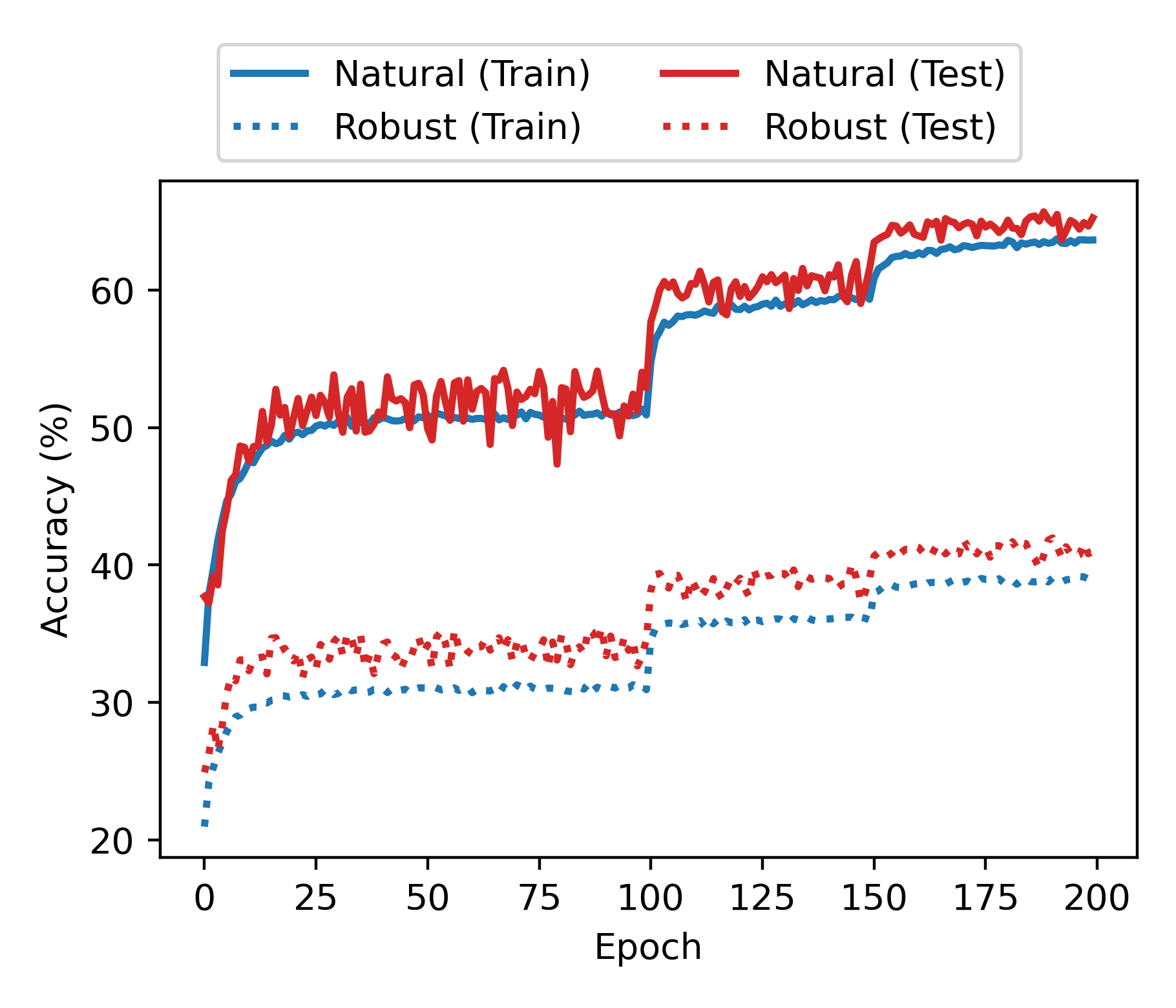}
        \caption{$\ell_1$-Regularization}
    \end{subfigure}
    \hfill
    \begin{subfigure}[b]{0.36\textwidth}
        \centering
        \includegraphics[width=\textwidth]{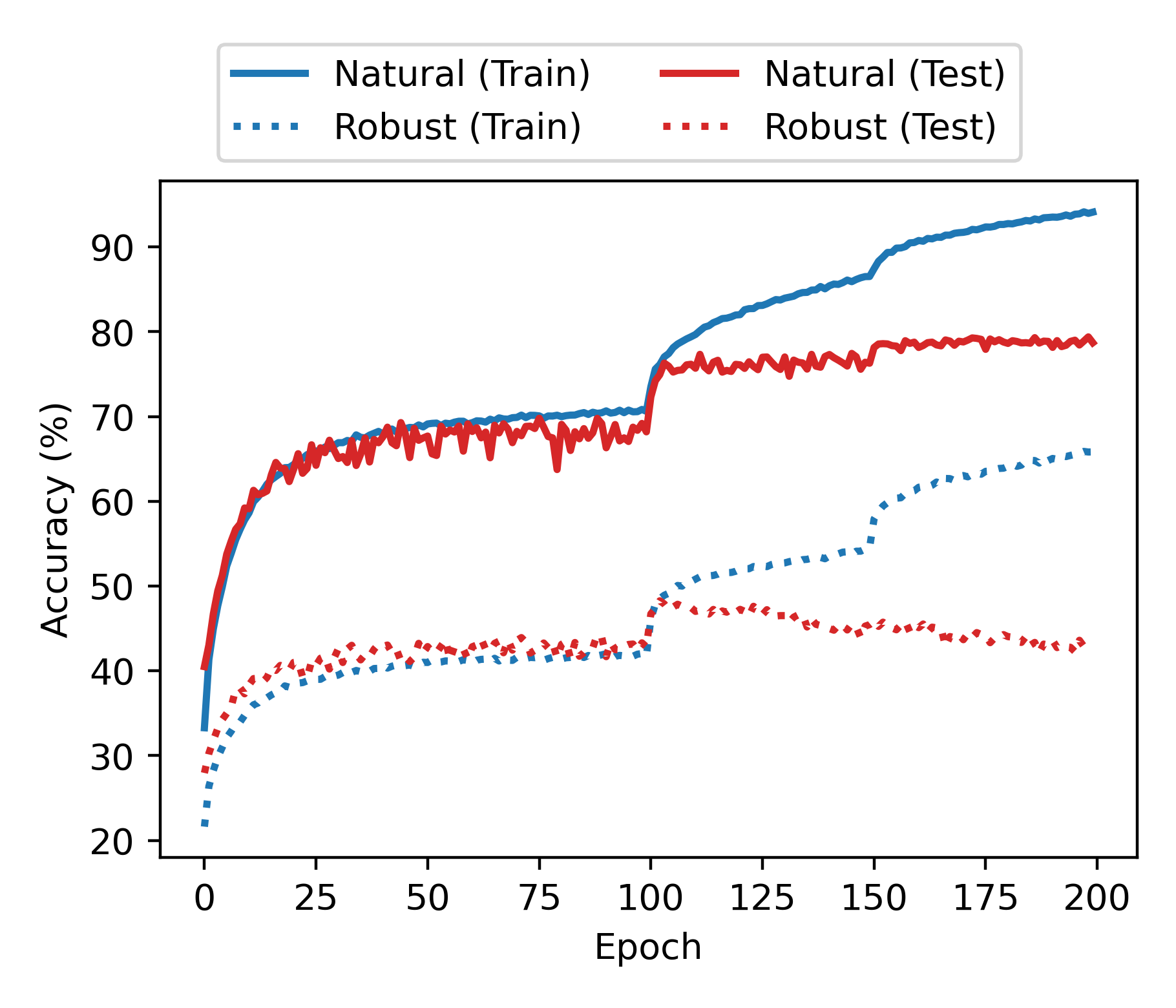}
        \caption{$\ell_2$-Regularization}
    \end{subfigure}
    \hfill
    \begin{subfigure}[b]{0.36\textwidth}
        \centering
        \includegraphics[width=\textwidth]{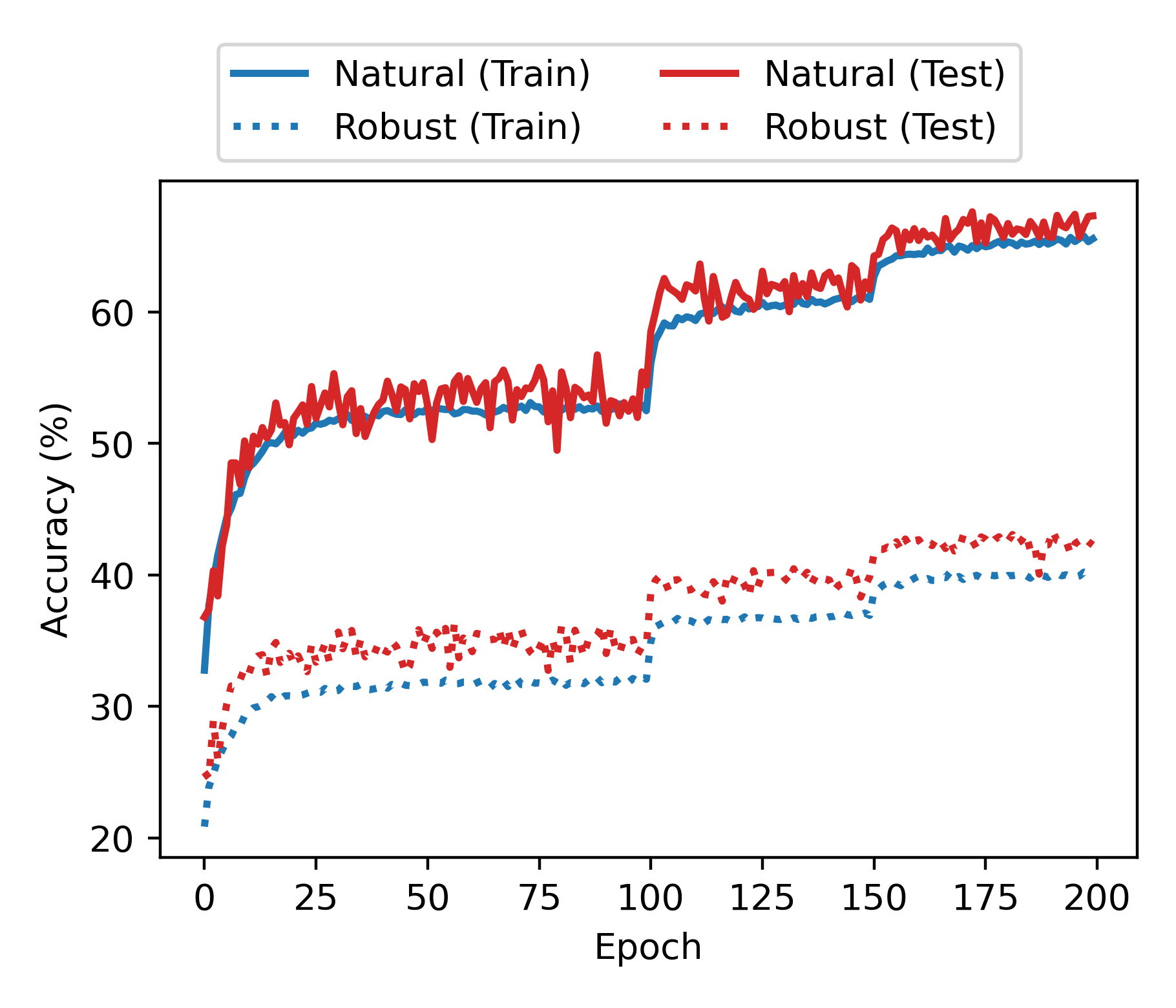}
        \caption{$\ell_2$ and $\ell_1$-Regularization}
    \end{subfigure}
    \hfill
    \begin{subfigure}[b]{0.36\textwidth}
        \centering
        \includegraphics[width=\textwidth]{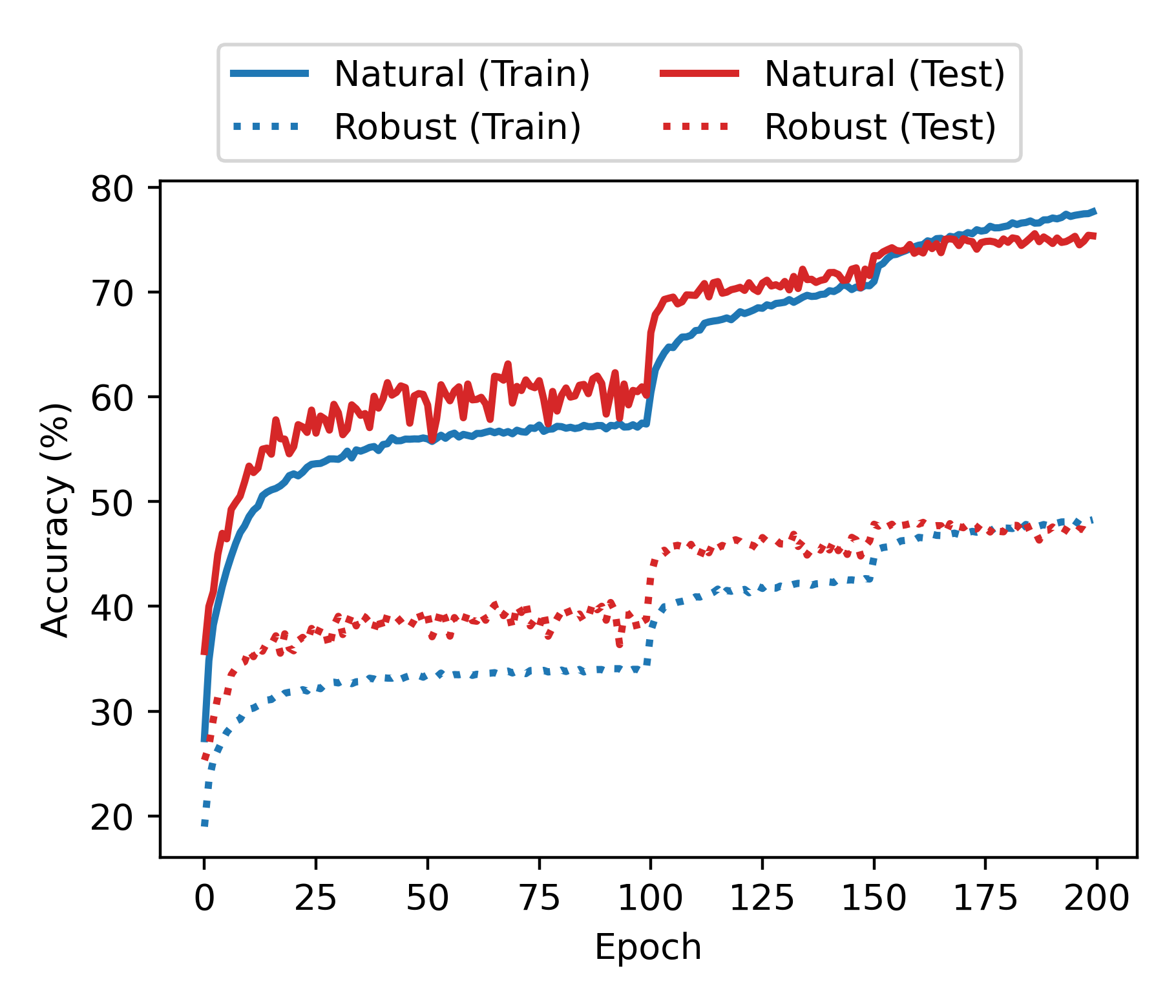}
        \caption{Cutout}
    \end{subfigure}
    \hfill
    \begin{subfigure}[b]{0.36\textwidth}
        \centering
        \includegraphics[width=\textwidth]{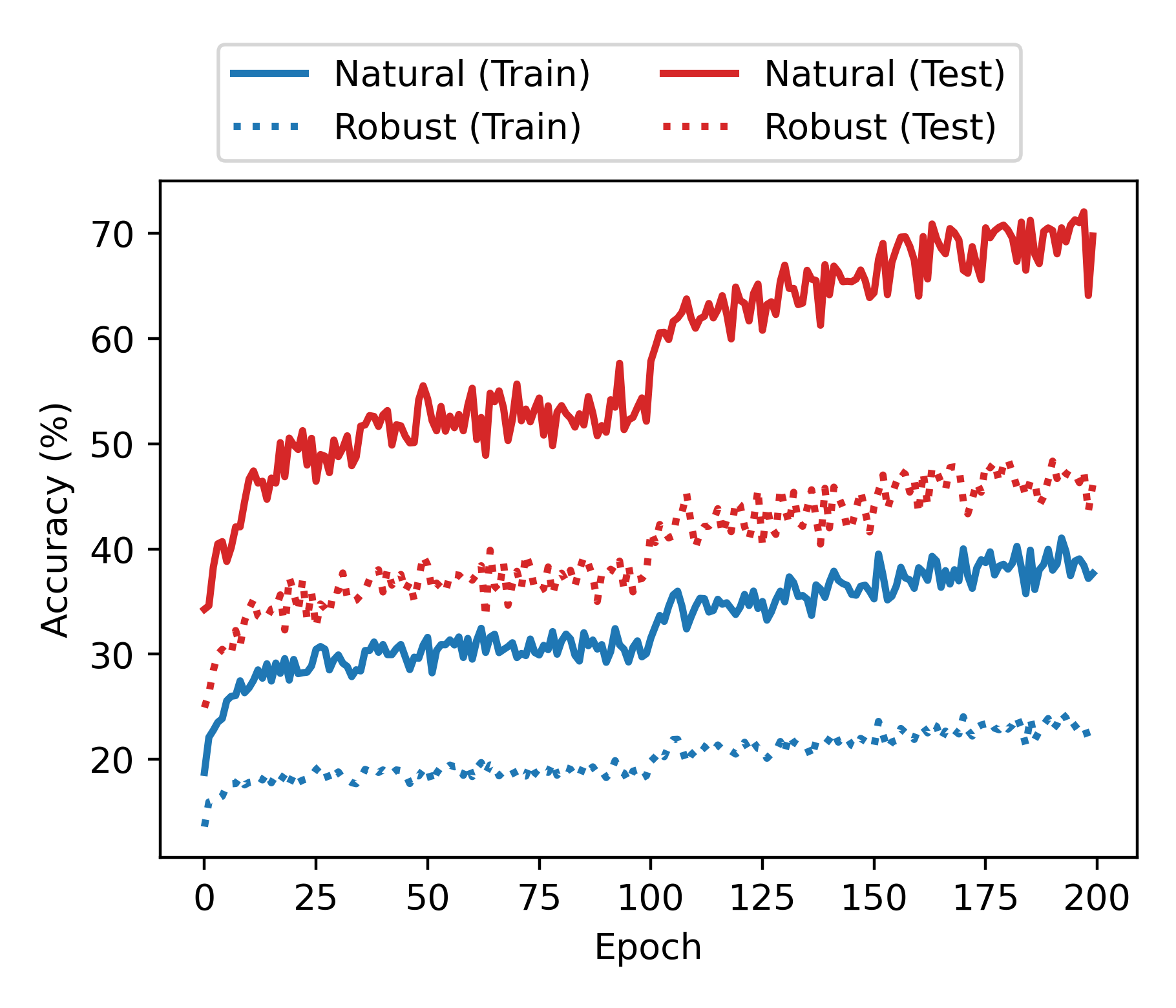}
        \caption{Mixup}
    \end{subfigure}
    \caption{Training curves of baselines.}
    \label{fig:training_curves_each_method}
\end{figure}

\newpage
\subsubsection{Comparison of All Methods}
\label{sec:curves_comparison_all_method}
To compare all methods, we plot the natural and robust performance on the training and testing sets in Figure~\ref{fig:comparison_training_curves_all_methods}. The figures show the consistent advantage of our Elastic DL over other methods.

\begin{figure}[h!]
    \centering
    \begin{subfigure}[b]{0.48\textwidth}
        \centering
        \includegraphics[width=\textwidth]{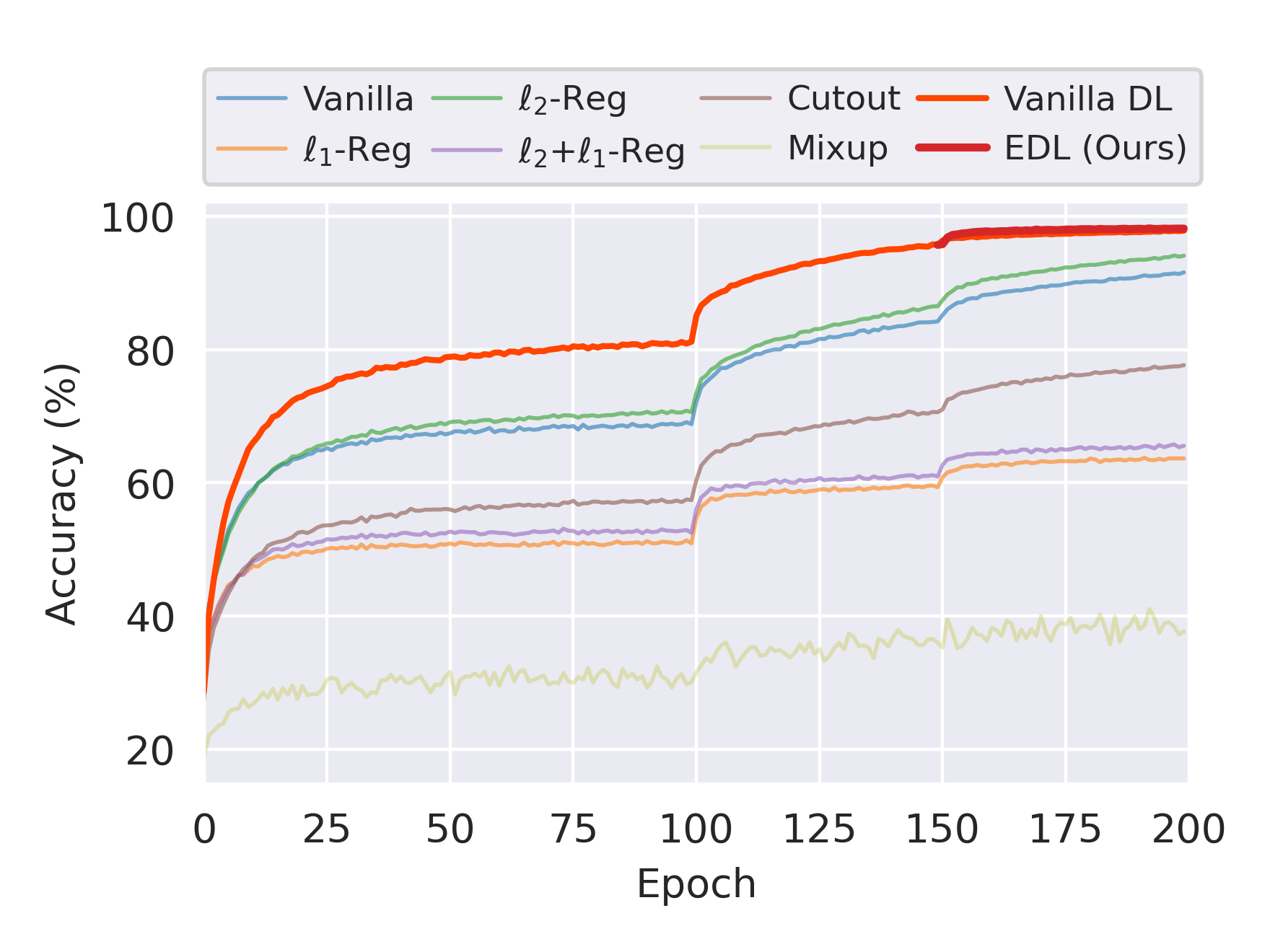}
        \caption{Natural (Train)}
    \end{subfigure}
    \hfill
    \begin{subfigure}[b]{0.48\textwidth}
        \centering
        \includegraphics[width=\textwidth]{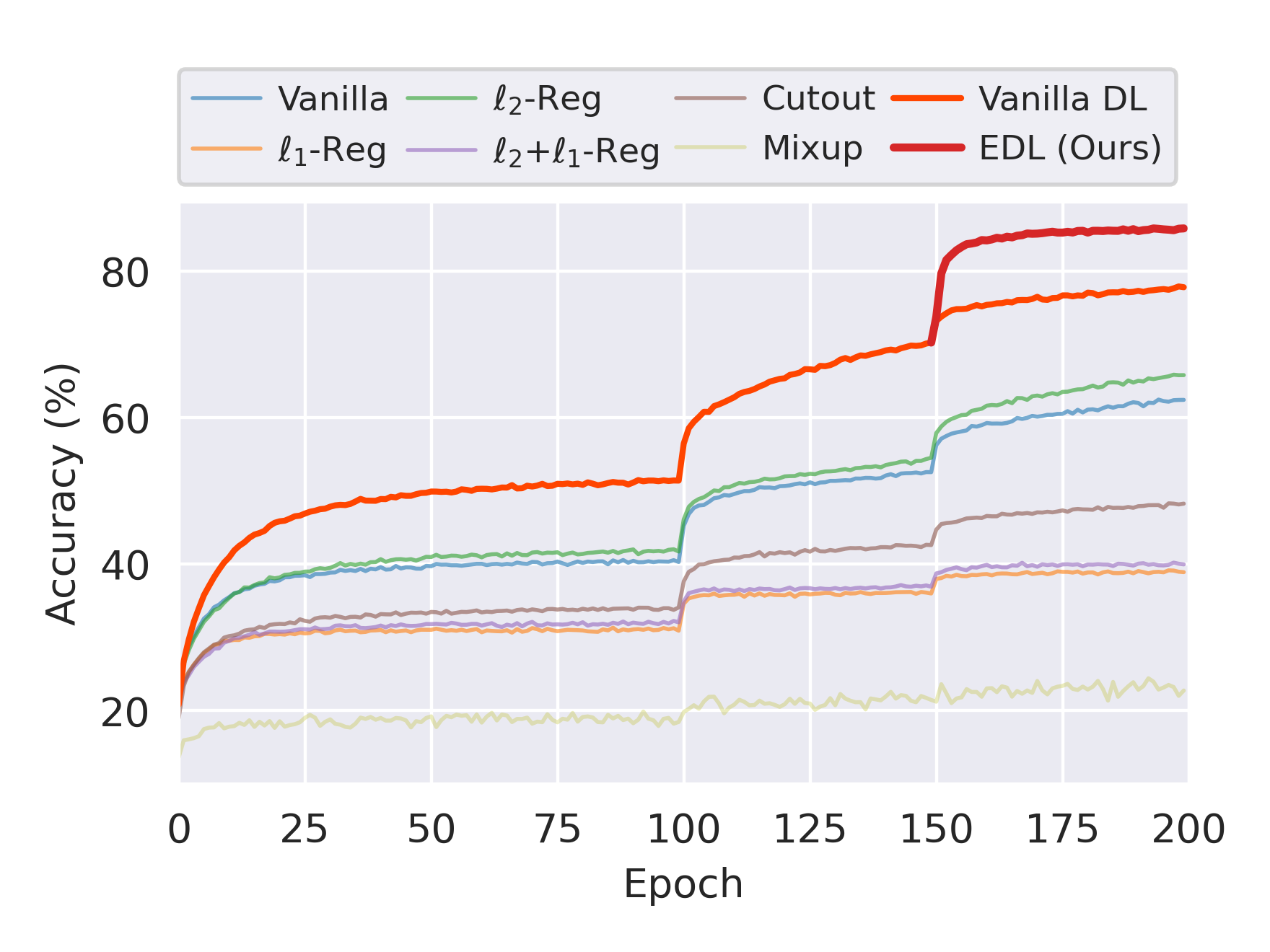}
        \caption{Robust (Train)}
    \end{subfigure}
    \hfill
    \begin{subfigure}[b]{0.48\textwidth}
        \centering
        \includegraphics[width=\textwidth]{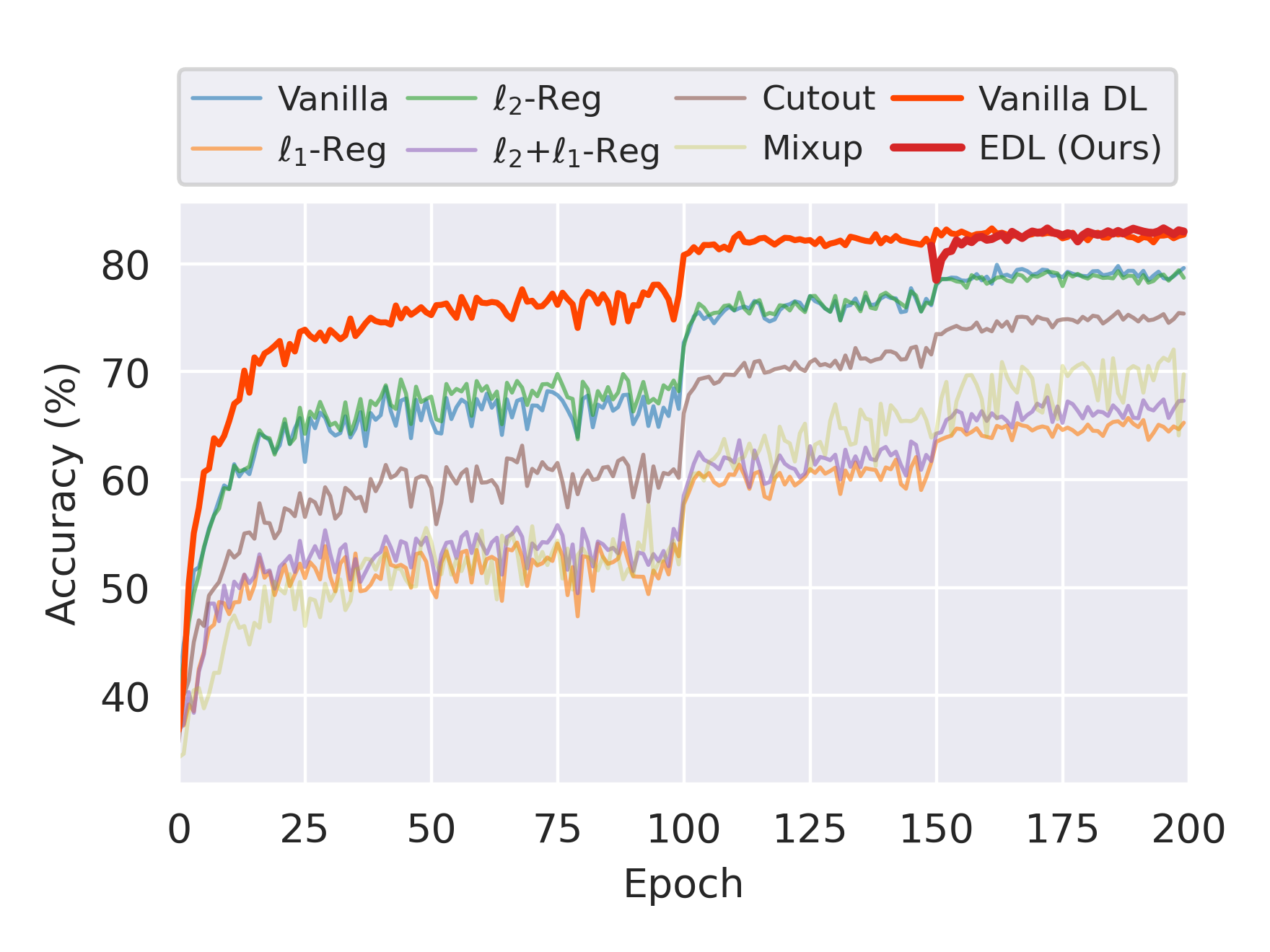}
        \caption{Natural (Test)}
    \end{subfigure}
    \hfill
    \begin{subfigure}[b]{0.48\textwidth}
        \centering
        \includegraphics[width=\textwidth]{figures/training_curves/curve_resnet10_all_adv_acc_test.png}
        \caption{Robust (Test)}
    \end{subfigure}
    \caption{Comparison of training curves across all methods.}
    \label{fig:comparison_training_curves_all_methods}
\end{figure}

\newpage
\subsection{Ablation Studies}
\label{sec:ablation_app}

\subsubsection{Universality}
\label{sec:universality}

\textbf{Universality across backbones, datasets, and attacks.}
We conduct ablation studies across different backbones, datasets, and attacks in Table~\ref{tab:diff_backbone_cifar10}, Table~\ref{tab:diff_backbone_cifar100}, and Table~\ref{tab:diff_backbone_imagenet}. Our proposed method shows consistent effectiveness across all settings.

\begin{table}[h!]
\centering
\caption{Adversarial robustness on CIFAR10 with different backbones.}
\label{tab:diff_backbone_cifar10}
\vspace{0.1in}
\begin{center}
\begin{sc}
\resizebox{0.7\textwidth}{!}{
\begin{tabular}{l|ccccc}
\hline
\rowcolor[HTML]{C0C0C0}
\textbf{Method} & \textbf{Natural} & \textbf{PGD} & \textbf{FGSM} & \textbf{C\&W} & \textbf{AA} \\ \hline
Vanilla DL + ResNet10  & 81.55 & 45.48 & 52.53 & 45.85 & 41.60 \\
Elastic DL + ResNet10  & 82.69 & 49.54 & 64.52 & 57.37 & 46.30 \\ \hline
Vanilla DL + ResNet18  & 83.28 & 45.64 & 53.88 & 41.22 & 43.70 \\
Elastic DL + ResNet18  & 83.57 & 53.22 & 69.35 & 60.80 & 52.90 \\ \hline
Vanilla DL + ResNet34  & 82.45 & 45.37 & 54.32 & 42.12 & 44.40 \\
Elastic DL + ResNet34  & 82.95 & 55.88 & 70.19 & 61.74 & 53.80 \\ \hline
Vanilla DL + ResNet50  & 81.22 & 46.83 & 53.75 & 43.64 & 45.10 \\
Elastic DL + ResNet50  & 81.07 & 58.33 & 69.38 & 64.87 & 56.70 \\ \hline
\end{tabular}
}
\end{sc}
\end{center}
\end{table}

\begin{table}[h!]
\centering
\caption{Adversarial robustness on CIFAR100 with different backbones.}
\label{tab:diff_backbone_cifar100}
\vspace{0.1in}
\begin{center}
\begin{sc}
\resizebox{0.7\textwidth}{!}{
\begin{tabular}{l|ccccc}
\hline
\rowcolor[HTML]{C0C0C0}
\textbf{Method} & \textbf{Natural} & \textbf{PGD} & \textbf{FGSM} & \textbf{C\&W} & \textbf{AA} \\ \hline
Vanilla DL + ResNet10  & 55.94 & 22.45 & 26.57 & 18.90 & 21.00 \\
Elastic DL + ResNet10  & 55.20 & 26.30 & 35.34 & 26.45 & 22.60 \\ \hline
Vanilla DL + ResNet18  & 57.24 & 22.17 & 26.81 & 17.43 & 21.60 \\
Elastic DL + ResNet18  & 57.70 & 27.27 & 37.62 & 28.87 & 26.30 \\ \hline
Vanilla DL + ResNet34  & 56.18 & 21.77 & 26.14 & 16.38 & 20.80 \\
Elastic DL + ResNet34  & 56.38 & 32.67 & 43.39 & 39.34 & 29.20 \\ \hline
Vanilla DL + ResNet50  & 54.01 & 22.39 & 26.40 & 18.40 & 20.90 \\
Elastic DL + ResNet50  & 54.64 & 30.29 & 41.48 & 35.24 & 28.10 \\ \hline
\end{tabular}
}
\end{sc}
\end{center}
\end{table}

\begin{table}[h!]
\centering
\caption{Adversarial robustness on Tiny-ImageNet with different backbones.}
\label{tab:diff_backbone_imagenet}
\vspace{0.1in}
\begin{center}
\begin{sc}
\resizebox{0.7\textwidth}{!}{
\begin{tabular}{l|ccccc}
\hline
\rowcolor[HTML]{C0C0C0}
\textbf{Method} & \textbf{Natural} & \textbf{PGD} & \textbf{FGSM} & \textbf{C\&W} & \textbf{AA} \\ \hline
Vanilla DL + ResNet10  & 49.60 & 27.17 & 32.46 & 37.91 & 20.20 \\
Elastic DL + ResNet10  & 50.12 & 32.93 & 39.64 & 40.10 & 24.90 \\ \hline
Vanilla DL + ResNet18  & 50.22 & 31.45 & 36.46 & 39.02 & 30.90 \\
Elastic DL + ResNet18  & 50.52 & 37.60 & 43.10 & 46.64 & 36.30 \\ \hline
Vanilla DL + ResNet34  & 50.03 & 33.54 & 37.24 & 37.19 & 29.30 \\
Elastic DL + ResNet34  & 50.40 & 34.80 & 41.75 & 44.72 & 34.60 \\ \hline
Vanilla DL + ResNet50  & 50.35 & 34.42 & 37.63 & 38.86 & 31.20 \\
Elastic DL + ResNet50  & 50.39 & 37.38 & 42.06 & 41.09 & 35.40 \\ \hline
\end{tabular}
}
\end{sc}
\end{center}
\end{table}

\newpage
\subsubsection{Orthogonality to Adversarial Training}
\label{sec:orthogonality}
Our proposed Elastic DL framework incorporates structural priors into neural networks, complementing existing adversarial training techniques. As shown in Table~\ref{tab:cifar10-main} and Figure~\ref{fig:diff_adv_train}, Elastic DL can be integrated with various adversarial training methods (PGD-AT, TRADES-2.0/0.2, HAT) to consistently enhance performance.

\begin{figure}[h!]
    \centering
    \includegraphics[width=0.45\textwidth]{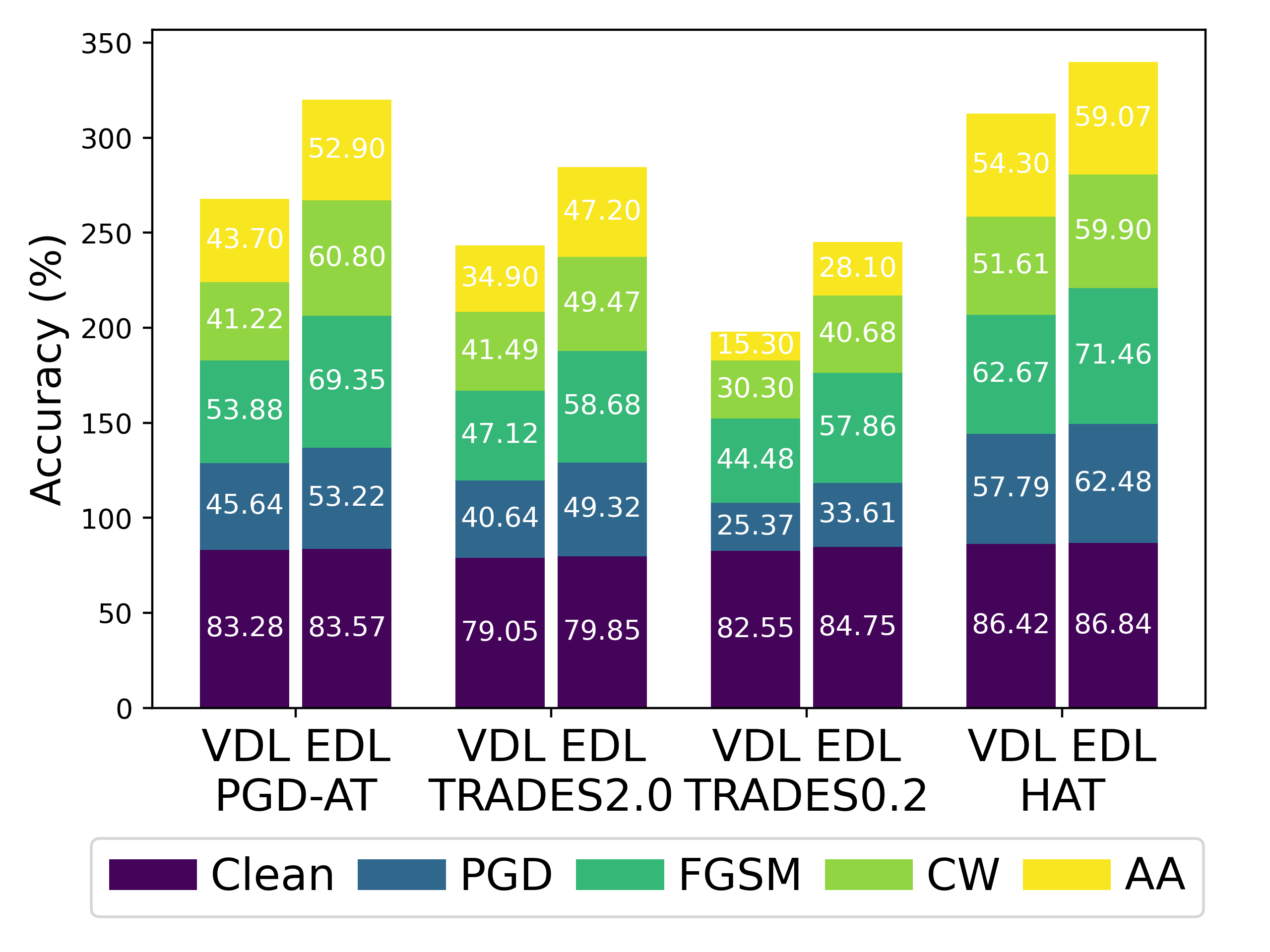}
    \caption{Different adversarial training methods. Our Elastic DL is orthogonal to existing adversarial training methods and can be combined with them to further improve performance.}
    \label{fig:diff_adv_train}
\end{figure}

\subsubsection{Different Budget Measurements}
\label{sec:diff_measurement}
In addition to the $\ell_\infty$-norm attack (PGD-$\ell_\infty$), we also validate the consistent effectiveness of our Elastic DL against $\ell_2$-norm (PGD-$\ell_2$) and $\ell_1$-norm (SparseFool) attacks in Figure~\ref{fig:diff_measure}.

\begin{figure}[h!]
    \centering
    \includegraphics[width=0.9\textwidth]{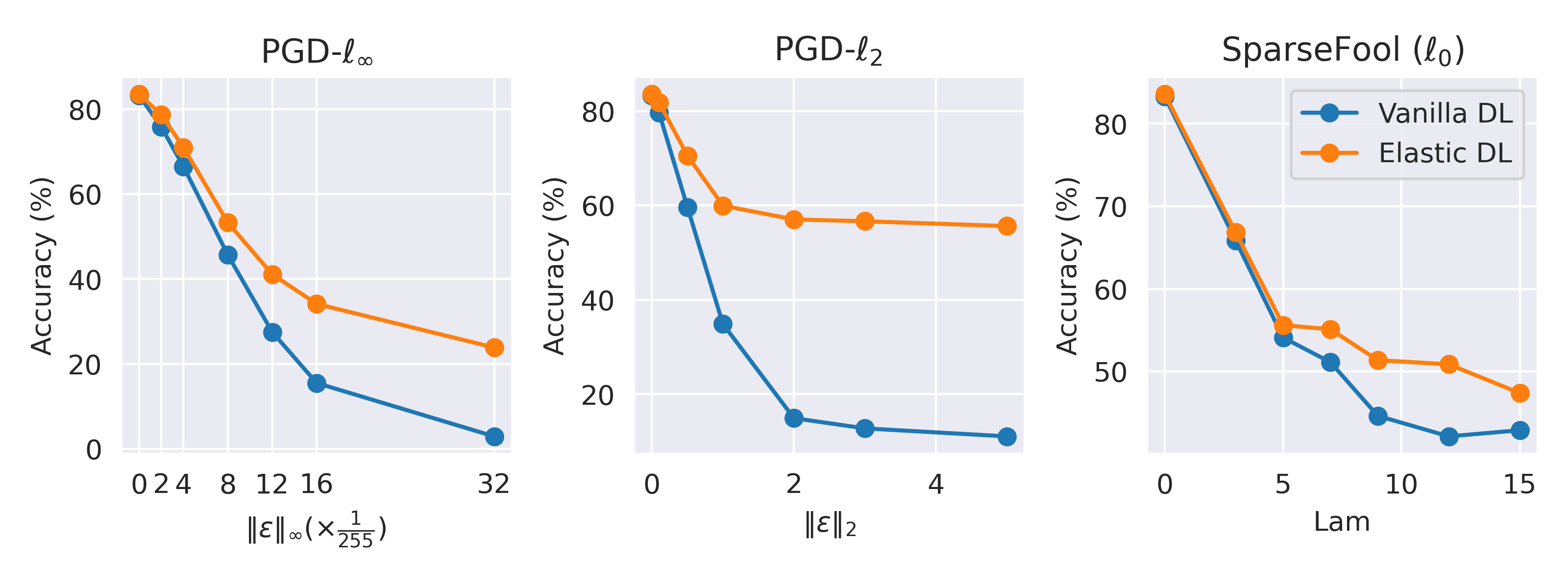}
    \caption{Different attack measurements. Our Elastic DL consistently outperforms Vanilla DL across attacks (PGD-$\ell_\infty$, PGD-$\ell_2$, SparseFool) evaluated under various metrics ($\ell_\infty,\ell_2,\ell_0$ norms).}
    \label{fig:diff_measure}
\end{figure}

\subsubsection{Zero-Order Gradient Analysis}
\label{sec:zero_order_gradient}

To further validate that our method does not introduce obfuscated gradients, we estimate the gradient with a zero-order method ($\frac{\partial f}{\partial x} \approx \frac{f(x+\epsilon)-f(x)}{\epsilon}$) and compare it against the gradient computed by autograd. The results in Table~\ref{tab:zero_order} show that the relative difference between the autograd and zero-order gradients ($\frac{| \text{Grad}_{\text{zero}} - \text{Grad}_{\text{auto}} |}{| \text{Grad}_{\text{auto}} |}$) is negligible. Moreover, the error does not accumulate or grow with the number of model layers, confirming that our method does not introduce gradient-related issues.

\begin{table}[h!]
    \centering
    \begin{sc}
    \resizebox{0.99\textwidth}{!}{
    \setlength{\tabcolsep}{1.8pt}
    \begin{tabular}{c|ccccccccccccccccccc}
    \hline
    \rowcolor[HTML]{C0C0C0}
    \textbf{Num. of Layers} & \textbf{3} & \textbf{5} & \textbf{8} & \textbf{10} & \textbf{15} & \textbf{20} & \textbf{25} & \textbf{30} & \textbf{40} & \textbf{50} \\ \hline
    Relative Error & 0.00066 & 0.00095 & 0.00152 & 0.00092 & 0.00118 & 0.00098 & 0.00112 & 0.00082 & 0.00060 & 0.00093 \\ \hline
    \end{tabular}
    }
    \end{sc}
    \vspace{0.1in}
    \caption{Zero-order gradient analysis: relative difference between autograd and zero-order gradients across network depths.}
    \label{tab:zero_order}
\end{table}

\subsubsection{Attack with Backward Pass Differentiable Approximation (BPDA)}
\label{sec:bpda}
To address concerns about gradient masking from the soft-thresholding operation, we explicitly evaluate our method under an adaptive attack based on Backward Pass Differentiable Approximation (BPDA). Concretely, during the forward pass we retain the exact proximal operator $y = \cT_{\lambda t}(x)$, while in the backward pass we replace the gradient of the soft-thresholding operation with an identity mapping, $\frac{\partial y}{\partial x} \approx I$. Under this stronger adaptive attack setting, robustness remains stable, indicating that the observed robustness is not solely due to gradient obfuscation.

\begin{table}[h!]
    \centering
    \begin{sc}
    \begin{tabular}{l|ccccc}
    \hline
    \rowcolor[HTML]{C0C0C0}
    \textbf{Budget} & \textbf{0} & \textbf{8/255} & \textbf{16/255} & \textbf{32/255} & \textbf{64/255} \\ \hline
    Elastic DL          & 83.57 & 53.29 & 34.13 & 23.84 & 12.45 \\
    Elastic DL + BPDA   & 83.57 & 53.20 & 34.08 & 23.51 & 12.44 \\ \hline
    \end{tabular}
    \end{sc}
    \caption{Adversarial robustness under standard adaptive attack versus BPDA attack, showing comparable performance and ruling out gradient obfuscation.}
    \label{tab:bpda_attack}
\end{table}

\newpage
\subsubsection{Hidden Embedding Visualization}
\label{sec:hidden_embedding_all}

\begin{wrapfigure}{r}{0.48\textwidth}
    \centering
    \vspace{-0.2in}
    \includegraphics[
        width=0.7\linewidth,
        height=1.55in,
        keepaspectratio
    ]{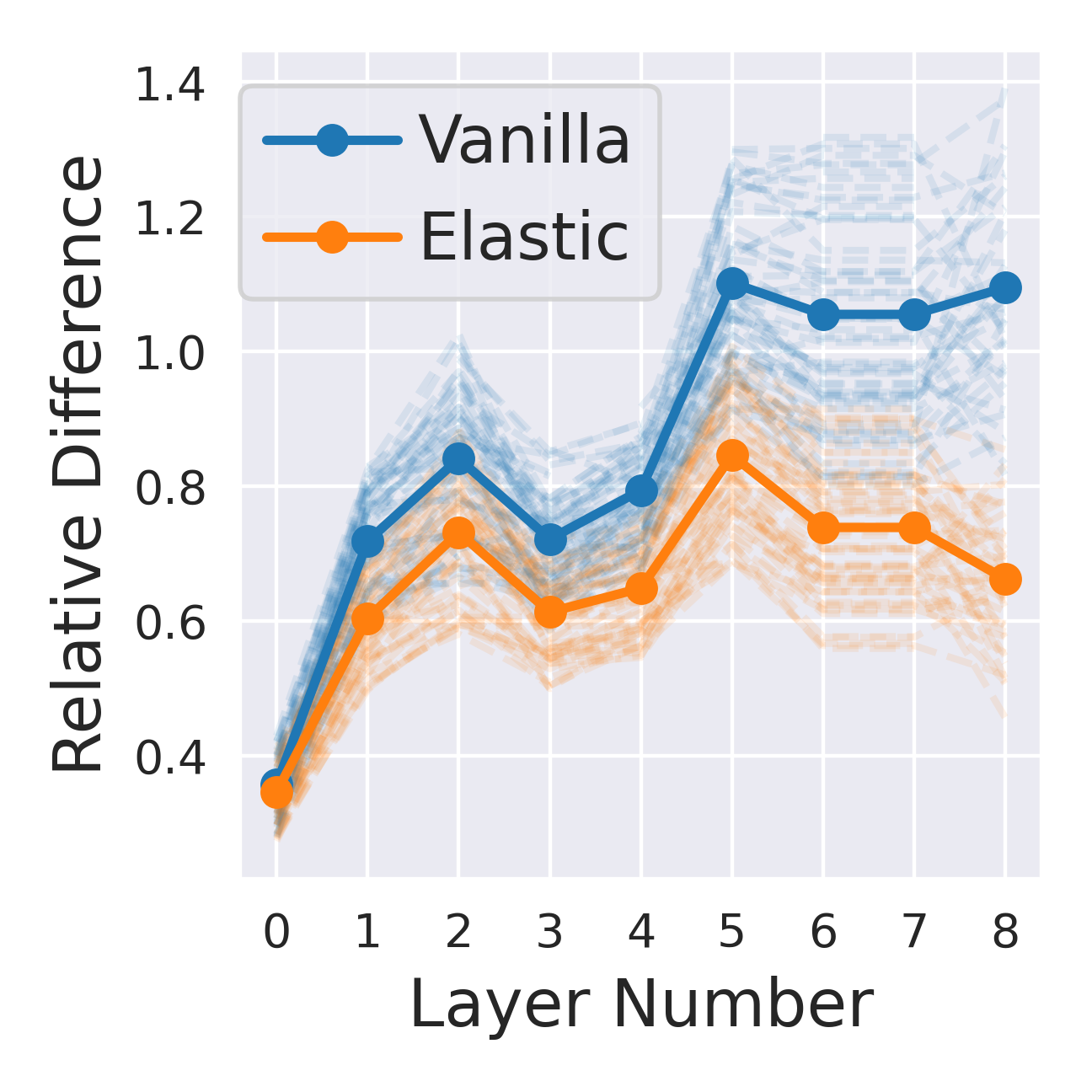}
    \vspace{-0.15in}
    \captionof{figure}{Embedding difference across layers.}
    \vspace{-0.2in}
    \label{fig:embed_diff}
\end{wrapfigure}

We conduct visualization analyses on the hidden embeddings to gain further insight into the effectiveness of our proposed Elastic DL. We begin by quantifying the relative difference between clean embeddings ($\bx_i$) and their adversarially attacked counterparts ($\bx'_i$) across all layers of the network in Figure~\ref{fig:embed_diff}. As shown in Figure~\ref{fig:hidden_embed_visualization_all_part1} and Figure~\ref{fig:hidden_embed_visualization_all_part2}, adversarial perturbations induce a substantially larger relative difference in the hidden embeddings of Vanilla DL than in those of Elastic DL, with this gap becoming increasingly pronounced in deeper layers, indicating that Elastic DL is markedly more effective at containing the propagation of adversarial perturbations through the network rather than merely suppressing them at the input. Beyond this layer-wise quantification, we further visualize representative embedding instances to illustrate the effect qualitatively: under adversarial perturbation, the disrupted hidden embedding patterns cause Vanilla DL to shift its prediction away from the groundtruth label, whereas Elastic DL preserves embedding patterns close to those of the clean input and thus maintains the correct prediction.

\newpage
Here are instances of CAT, SHIP, FROG, AUTOMOBILE, and TRUCK:
\begin{figure}[h!]
    \centering
    \includegraphics[width=0.99\linewidth]{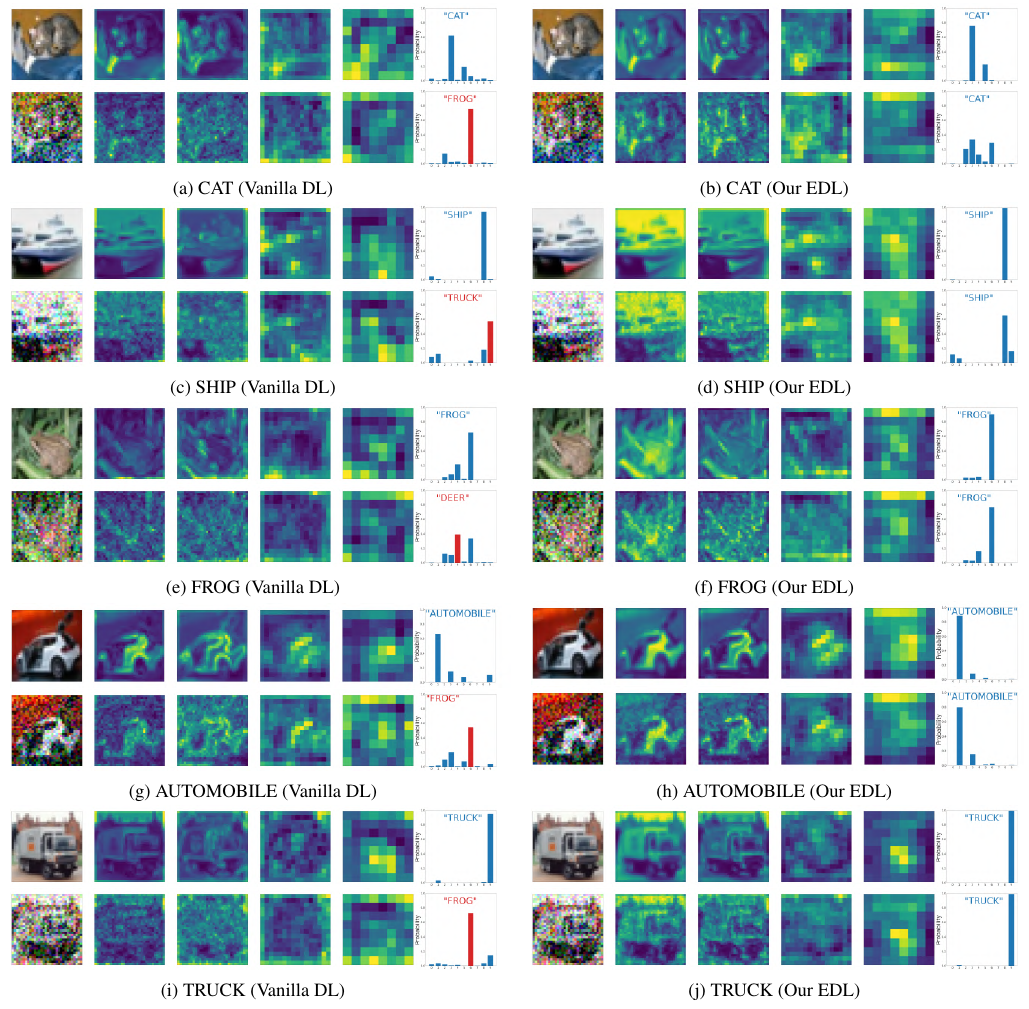}
    \caption{Hidden embedding visualization (Part 1).}
    \label{fig:hidden_embed_visualization_all_part1}
\end{figure}

\newpage
Here are instances of BIRD, HORSE, AIRPLANE, DEER, and DOG:

\begin{figure}[h!]
    \centering
    \includegraphics[width=0.99\linewidth]{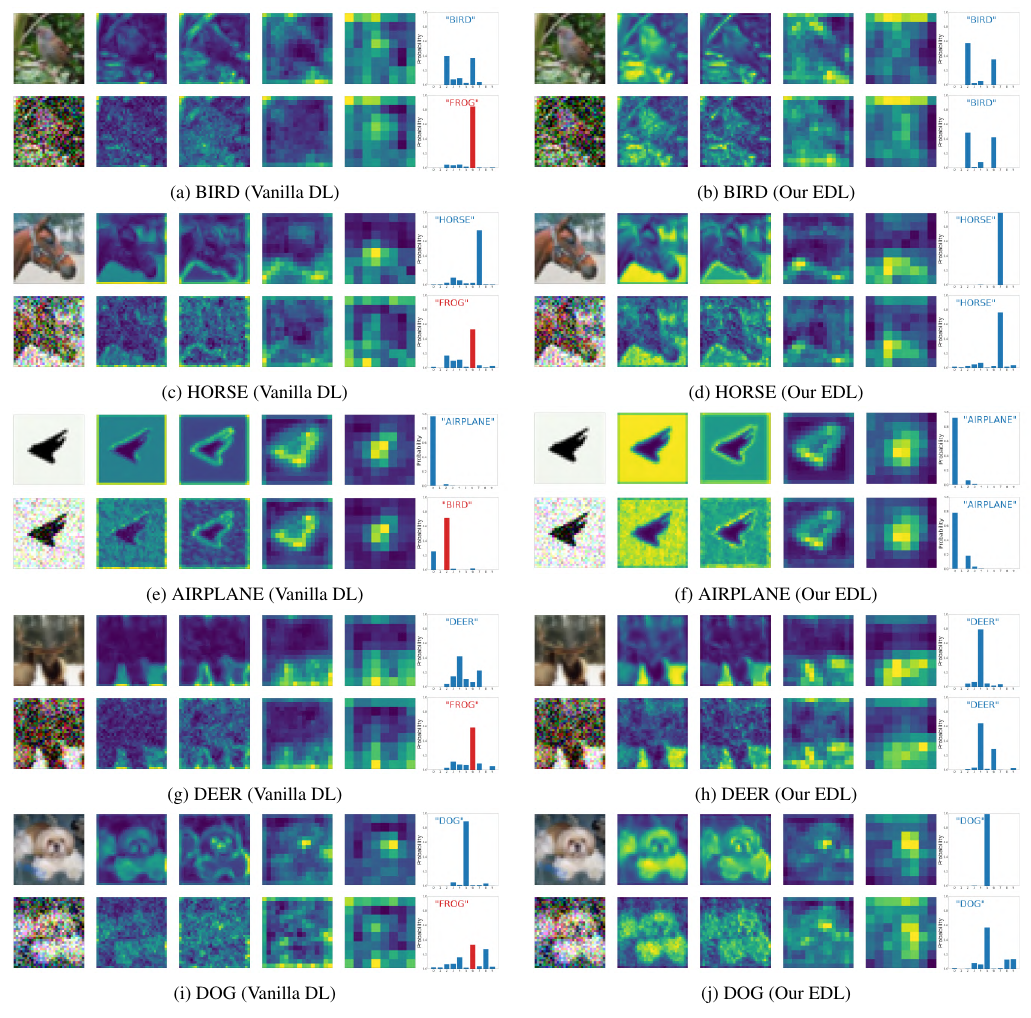}
    \caption{Hidden embedding visualization (Part 2).}
    \label{fig:hidden_embed_visualization_all_part2}
\end{figure}

\newpage
\subsubsection{Reconstruction Process}
\label{sec:reconstruction_process}

\textbf{Image \& noise reconstruction.} In conventional feedforward neural networks, adding a perturbation $\bepsilon$ to the input can lead the model to make incorrect predictions. As illustrated in Figure~\ref{fig:reconstruction_process}, our approach instead reconstructs both the clean image $\bx$ and the perturbation $\bepsilon$ through a dictionary learning process. To evaluate the effectiveness of our method, we quantify the reconstruction error between the recovered noise $\hat{\bepsilon}$ under our Elastic DL framework and noise generated by various methods (random noise, transfer noise from ResNet/Vanilla DL, and adaptive noise from Elastic DL). As shown in Table~\ref{tab:reconstruction}, the recovered noise from our approach exhibits the smallest difference relative to the adaptive noise in Elastic DL, demonstrating that our proposed framework more effectively reconstructs and mitigates the impact of adversarial noise on predictions.

\begin{table}[h!]
\centering
\caption{Reconstruction error between the recovered noise $\hat{\bepsilon}$ and various input noises, including random noise ($\bepsilon_{\text{random}}$), transfer noise from ResNet ($\bepsilon_{\text{resnet}}$) and Vanilla DL ($\bepsilon_{\text{vanilla}}$), and adaptive noise from our Elastic DL ($\bepsilon_{\text{elastic}}$). Our Elastic DL achieves the smallest reconstruction error, indicating that our approach can adaptively recover and neutralize the input perturbation, thereby mitigating its impact.}
\label{tab:reconstruction}
\vspace{0.1in}
\begin{center}
\begin{sc}
\resizebox{0.5\textwidth}{!}{
\begin{tabular}{c|ccc}
\hline
\rowcolor[HTML]{C0C0C0}
\textbf{Error} & $\|\cdot\|_1$ & $\|\cdot\|_2$ & $\|\cdot\|_\infty$ \\ \hline
$\bepsilon_{\text{random}}-\hat{\bepsilon}$ & 1294.75 $\pm$ 406.78 & 26.09 $\pm$ 7.04 & 0.901 $\pm$ 0.10 \\ \hline
$\bepsilon_{\text{resnet}}-\hat{\bepsilon}$ & 131.51 $\pm$ 10.53 & 2.93 $\pm$ 0.22 & 0.163 $\pm$ 0.01 \\ \hline
$\bepsilon_{\text{vanilla}}-\hat{\bepsilon}$ & 129.07 $\pm$ 13.22 & 2.85 $\pm$ 0.26 & 0.157 $\pm$ 0.01 \\ \hline
$\bepsilon_{\text{elastic}}-\hat{\bepsilon}$ & \textbf{122.62 $\pm$ 9.92} & \textbf{2.69 $\pm$ 0.22} & \textbf{0.149 $\pm$ 0.01} \\ \hline
\end{tabular}
}
\end{sc}
\end{center}
\end{table}

\begin{figure}[h!]
    \centering
    \includegraphics[width=1.0\linewidth]{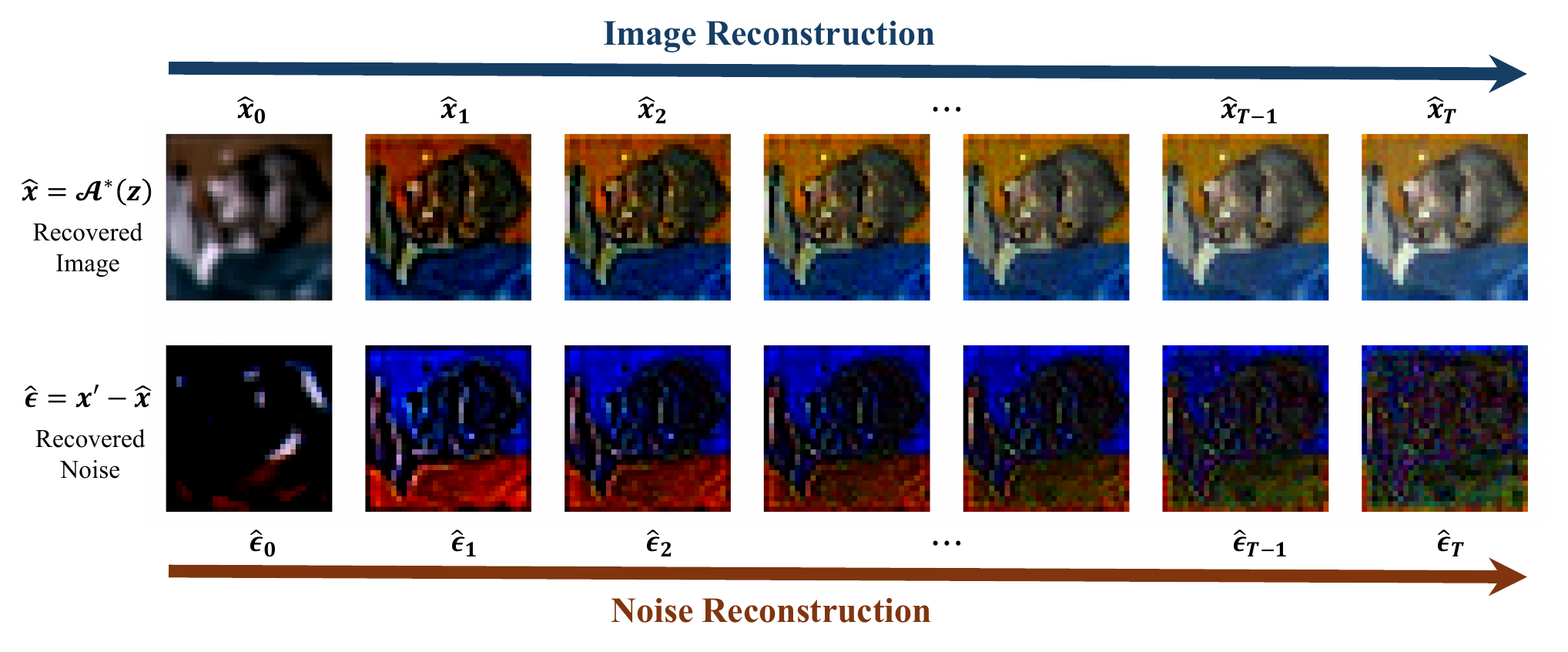}
    \caption{Reconstruction process.}
    \label{fig:reconstruction_process}
\end{figure}

\newpage
Here are instances of the reconstruction process on ImageNet:
\begin{figure}[h!]
    \centering
    \includegraphics[width=0.9\linewidth]{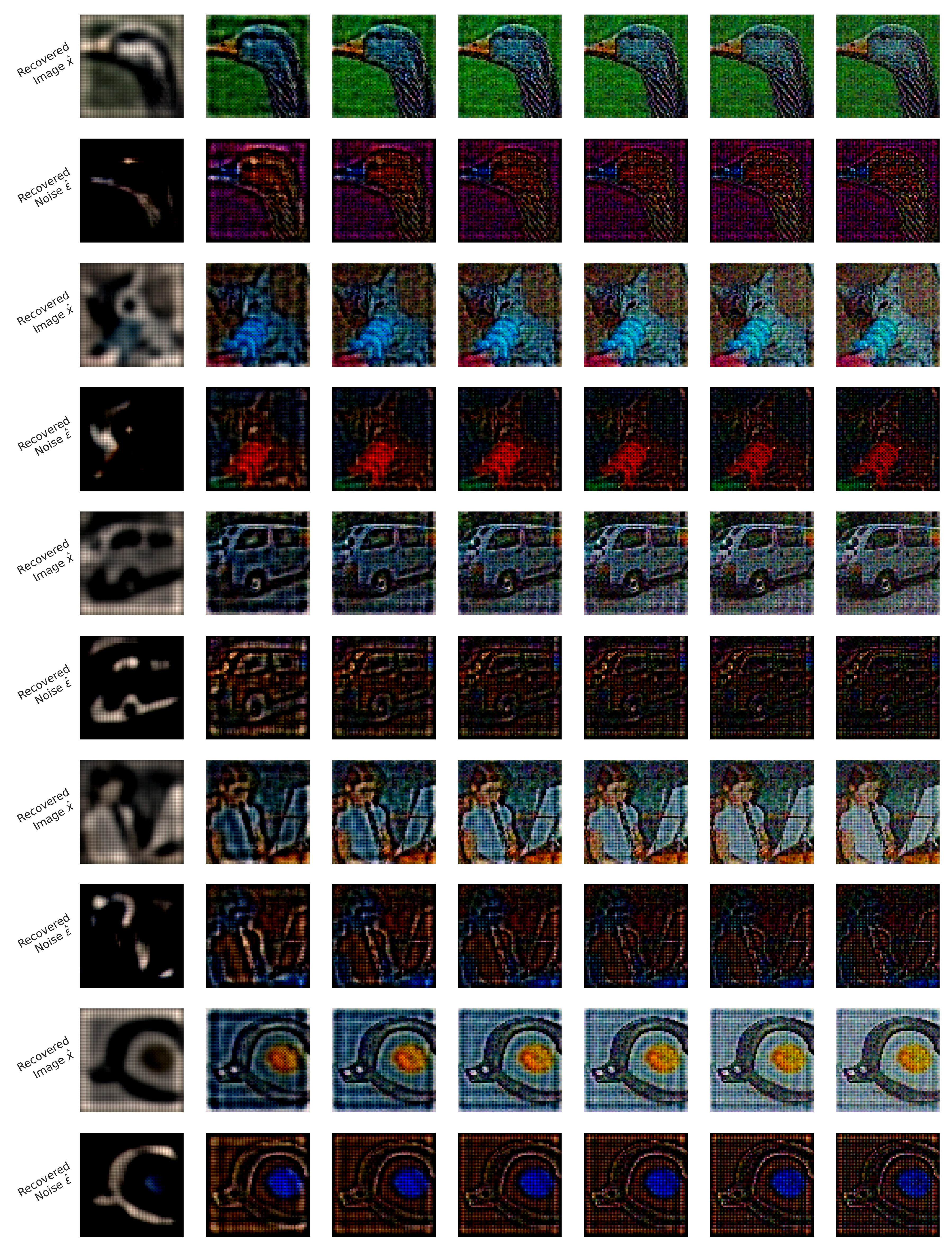}
    \caption{Reconstruction process (ImageNet).}
    \label{fig:reconstruction_process_imagenet}
\end{figure}

\newpage
Here are instances of the reconstruction process on CIFAR10 (Part 1):
\begin{figure}[h!]
    \centering
    \includegraphics[width=0.9\linewidth]{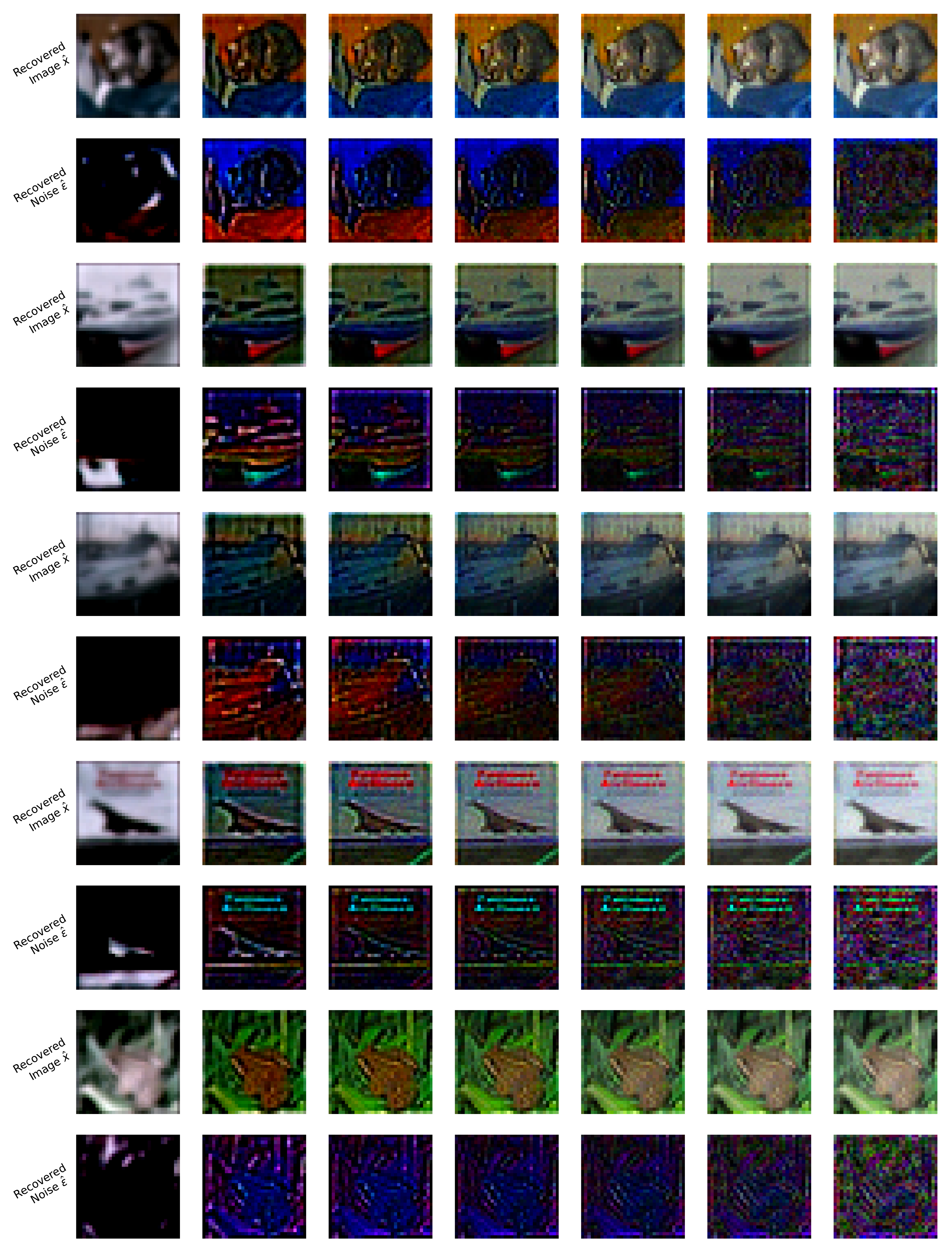}
    \caption{Reconstruction process (CIFAR10, Part 1).}
    \label{fig:reconstruction_process_part1}
\end{figure}

\newpage
Here are instances of the reconstruction process on CIFAR10 (Part 2):
\begin{figure}[h!]
    \centering
    \includegraphics[width=0.9\linewidth]{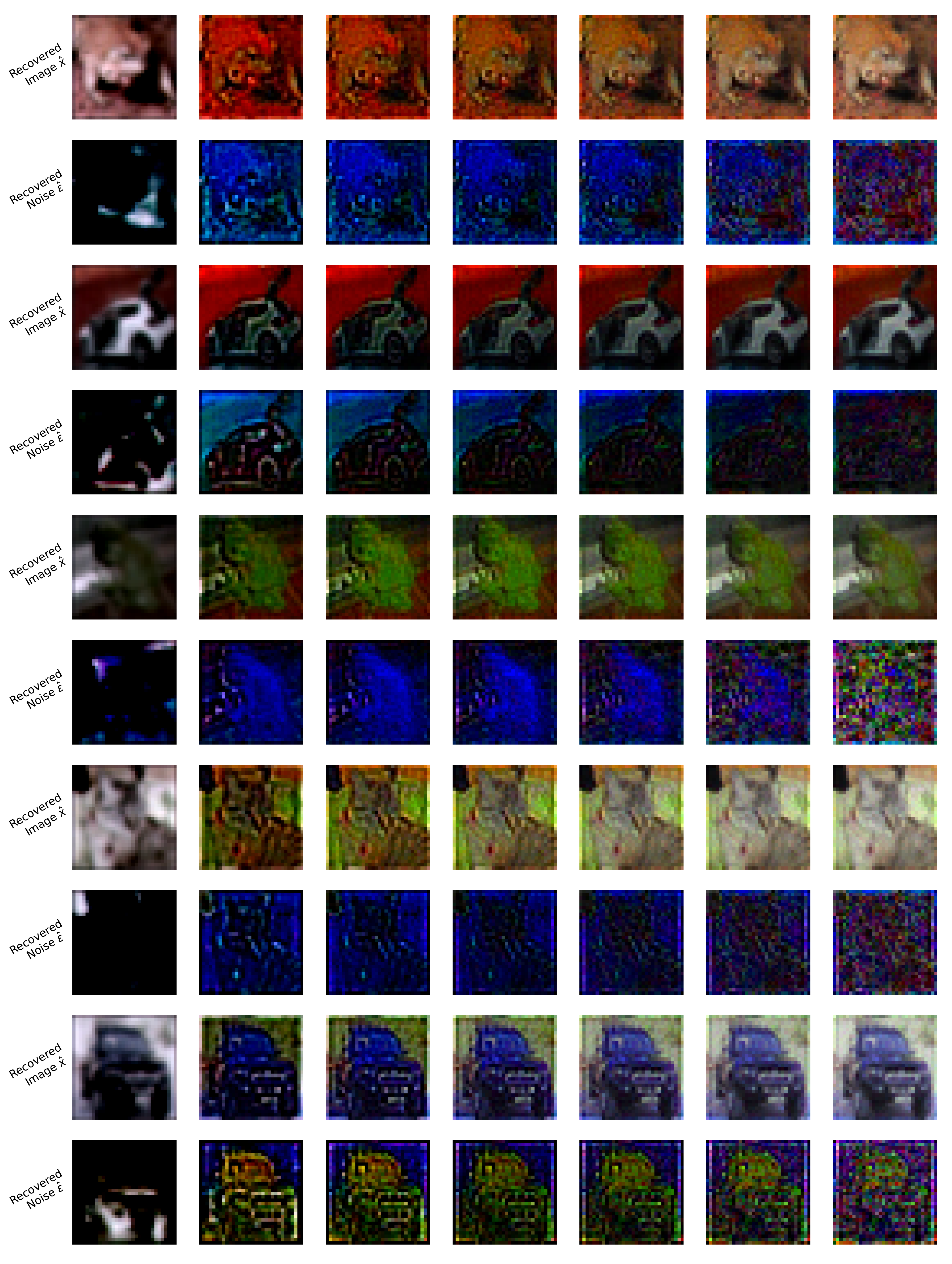}
    \caption{Reconstruction process (CIFAR10, Part 2).}
    \label{fig:reconstruction_process_part2}
\end{figure}

\end{document}